\documentclass[12pt]{article}

\usepackage[utf8]{inputenc}
\usepackage[english]{babel}
\usepackage[margin=1in]{geometry}

\usepackage{blindtext}
\usepackage{appendix}
\usepackage{amssymb}
\usepackage{hyperref}
\usepackage{mathrsfs}
\hypersetup{
    colorlinks=true,
    linkcolor=blue,
    filecolor=magenta,      
    urlcolor=cyan,
}
 
\usepackage{amsthm}
\usepackage{amsmath}
\usepackage{graphicx}
\usepackage{subfig}
\usepackage{float}
\usepackage{booktabs}
\usepackage{multirow}

\usepackage{algorithm, algorithmic}
\newtheorem{theorem}{Theorem}[section]

\newtheorem{lemma}[theorem]{Lemma}
\newtheorem{assumption}{Assumption}[section]
\theoremstyle{remark}
\newtheorem{remark}{Remark}[section]

\theoremstyle{definition}
\newtheorem{definition}{Definition}[section]

\def\m0{\mathbf{0}}

\def \mI {\mathbf{I}}

\def \bB {\boldsymbol{B}}
\def \bS {\boldsymbol{S}}

\def \bY {\boldsymbol{Y}}
\def \bZ {\boldsymbol{Z}}
\def \ba {\boldsymbol{a}}
\def \bb {\boldsymbol{b}}
\def \bs {\boldsymbol{s}}
\def \bx {\boldsymbol{x}}
\def \by {\boldsymbol{y}}
\def \bu {\boldsymbol{u}}
\def \bz {\boldsymbol{z}}

\def \mrd {\mathrm{d}}
\def \mrN {\mathrm{NN}}

\def \mcD {\mathcal{D}}
\def \mcL {\mathcal{L}}
\def \mcN {\mathcal{N}}
\def \mcT {\mathcal{T}}
\def \mcX {\mathcal{X}}
\def \mcY {\mathcal{Y}}
\def \mcZ {\mathcal{Z}}

\def\Ebb{\mathbb{E}}
\def \Rbb{\mathbb{R}}

\def\wh{\widehat}
\def\wt{\widetilde}
\def\ov{\overline}
\def\ck{\check}

\title{\textbf{
Model Free Prediction with Uncertainty  Assessment
}
}
\author{
Yuling Jiao
\thanks{School of Mathematics and Statistics, Wuhan University, Wuhan, China.
Email: yulingjiaomath@whu.edu.cn}
\and
Lican Kang
\thanks{School of Mathematics and Statistics, Wuhan University, Wuhan, China.
Email: kanglican@whu.edu.cn}
\and
Jin Liu 
\thanks{ 
School of Data Science, The Chinese University of Hong Kong, Shenzhen, China.
Email: liujinlab@cuhk.edu.cn}
\and
Heng Peng
\thanks{Department of Mathematics, Hong Kong Baptist University, Hong Kong.
Email: hpeng@hkbu.edu.hk}
\and
Heng Zuo
\thanks{School of Mathematics and Statistics, Wuhan University, Wuhan, China.
Email: zuoheng@whu.edu.cn}
}

\date{}

\begin{document}

\maketitle
\begin{abstract}
Deep nonparametric regression, characterized by the utilization of deep neural networks to learn target functions, has emerged as a focus of research attention in recent years. Despite considerable progress in understanding convergence rates, the absence of asymptotic properties hinders rigorous statistical inference. 
To address this gap, we propose a novel framework that transforms the deep estimation paradigm into a platform conducive to conditional mean estimation, leveraging the conditional diffusion model. 
Theoretically, we develop an
end-to-end convergence rate for the conditional diffusion model
and establish the asymptotic normality of the generated samples. Consequently, we are equipped to construct confidence regions, facilitating robust statistical inference. 
Furthermore, through  numerical experiments, we empirically validate the efficacy of our proposed methodology.

\vspace{0.5cm} \noindent{\bf KEY WORDS}:
Statistical inference,
Conditional diffusion model, 
End-to-End error analysis, 
Deep nonparametric regression.
\end{abstract}

\section{Introduction}
Deep learning \cite{lecun2015deep,goodfellow2016deep},
grounded in the utilization of deep neural networks (DNNs), has evolved into a cornerstone within both industrial applications and academic research. Its advent has sparked widespread interest among scholars who have increasingly turned to DNNs as a powerful tool for estimating underlying regression functions. This trend has led to a burgeoning field known as deep nonparametric regression in statistics, where DNNs are leveraged to model complex relationships between covariates and response variables. In this context, a number of recent studies, including those by \cite{bauer2019deep,schmidt2020nonparametric,kohler2021rate,farrell2021deep,jiao2023deep,bhattacharya2023deep}, have contributed significantly to the advancement of deep 
nonparametric regression methodologies and the elucidation of their theoretical underpinnings. 
Now, let us revisit the concept of deep nonparametric regression. Consider a scenario where we have a pair $(X,Y) \in \mathcal{X} \times \mathcal{Y} \subset 
\mathbb{R}^{d_{\mcX}} \times \mathbb{R}^{d_{\mcY}}$, representing a $d_{\mcX}$-dimensional covariate and its  corresponding $d_{\mcY}$-dimensional response variable. This pair follows an unknown joint distribution $P$. We denote the conditional mean as 
\begin{align}\label{condexp}
f_{0}(\bx):=\mathbb{E}(Y|X=\bx),~ \bx \in \mathcal{X},
\end{align}
where $f_0:\mathbb{R}^{d_{\mcX}}\rightarrow \mathbb{R}^{d_{\mcY}}$ also denotes the underlying regression function.
With access to independently and identically distributed (i.i.d.) data denoted by $\mathbb{S}:=
\{(X_i,Y_i)\}_{i=1}^n \sim P$, the objective of deep nonparametric estimation is to obtain the estimation of $f_0$ using DNNs,  
based on the sample set $\mathbb{S}$.
This deep estimator is denoted as $\wh{f}_n$, expressed as
\begin{align}\label{dnnest}
\wh{f}_n \in \arg\min_{f \in \mathcal{F}}\sum_{i=1}^n \left\|Y_i-f(X_i)\right\|^2,
\end{align}
where $\mathcal{F}$ represents the DNNs.
With this deep estimation  $\wh{f}_n$ in mind, we are empowered to make predictions $\wh{f}_n(\bx)$ of $f_0(\bx)$ for each $\bx \in \mathcal{X}$.
While extant literature \cite{bauer2019deep,schmidt2020nonparametric,kohler2021rate,farrell2021deep,jiao2023deep,bhattacharya2023deep} has made significant strides in providing theoretical guarantees for $\wh{f}_n$, achieving consistency and delineating convergence rates, a lingering challenge persists: the construction of asymptotic convergence for $\wh{f}_n(\bx)$ remains elusive. This gap presents an obstacle in the pathway to making robust statistical inferences.
In light of this, our endeavor in this paper is to illuminate this obscure terrain and devise methodologies to tackle this pertinent issue, thereby bridging the gap in deep nonparametric regression.

In traditional statistical inference frameworks, the utilization of methods such as $M$-estimators and $Z$-estimators is customary for the establishment of asymptotic distributions. These distributions, once established, facilitate the application of the Bootstrap method to construct asymptotic confidence intervals for the obtained estimator. Comprehensive discussions on these techniques and their theoretical foundations can be found in these works, 
such as \cite{van1997weak, van2000asymptotic,kosorok2008introduction}.
This iterative process serves to culminate the statistical inference procedure by providing reliable estimates of population parameters or functions.  
For instance, in the context of kernel estimation, this estimation can be  explicitly expressed as
\begin{align}\label{kernelest0}
f_{n,h}(\bx):=\sum_{i=1}^n Y_i \ell_i(\bx),~ \bx \in \mathcal{X},
\end{align}
where
$
\ell_i(\bx):=\frac{K\left(\frac{X_i-\bx}{h}\right)}{\sum_{i=1}^n K\left(\frac{X_i-\bx}{h}\right)}.
$
Here, $K(\cdot)$ denotes the kernel function with bandwidth $h > 0$.  
Furthermore, under some conditions met by the kernel function $K(\cdot)$  in \eqref{kernelest0}, we can establish the asymptotic normality of $f_{n,h}(\bx)$.
One can refer to \cite{wasserman2006all} for a  detailed analysis  about kernel estimations. 
We notice  that within \eqref{dnnest}, we procure an estimation encompassing the entirety of the functional $f_0$. Conversely, in \eqref{kernelest0}, we derive a pointwise estimation $f_{n,h}(\bx)$ of $f_0(\bx)$, characterized by its explicit expression.  This pivotal property endows us with the capability to engage in various statistical inference procedures. However, it is essential to underscore that such a feat remains elusive for the deep estimation  $\wh{f}_{n}(\bx)$.
Herein lies a fundamental challenge: unlike kernel settings, obtaining the asymptotic distribution of $\wh{f}_n(\bx)$ for every $\bx \in \mathcal{X}$ becomes an elusive endeavor since the techniques of $M$-estimators and $Z$-estimators can not be applicable. Consequently, the attainment of pointwise weak convergence properties akin to those observed in kernel estimation becomes unattainable. 
As a corollary, the application of traditional statistical inference methodologies to deep estimators becomes infeasible.

In light of above predicament,  we propose an alternative approach that circumvents the constraints of conventional statistical inference methods. Rather than fixating on the construction of weak convergence properties for $\wh{f}_n(\bx)$ in \eqref{dnnest}, we advocate for the adoption of conditional generative learning techniques. 
To elaborate, let us denote the random variable 
$Y_{\bx}$ distributed from  the conditional distribution of $Y$ given $X=\bx$, represented as $P_{Y|X=\bx}$.  Consequently, we have $f_0(\bx)=\mathbb{E}(Y_{\bx})$ derived from  the expression in  \eqref{condexp}.
Utilizing the fundamental principles of the law of large numbers and the central limit theorem, we can leverage the sample mean derived from the conditional distribution $P_{Y|X=\bx}$ to effectively estimate the conditional expectation function  $f_0(\bx)$ and delineate its asymptotic convergence properties. 
This methodology entails the estimation of the unknown conditional distribution $P_{Y|X=\bx}$, followed by data acquisition from the estimated conditional distribution.  
By harnessing the power of conditional generative models,  we can  learn $P_{Y|X=\bx}$ from the available data, culminating in the derivation of an estimated distribution denoted as $\wh{P}_{Y|X=\bx}$. Following this estimation process, we proceed to generate data  drawn from the estimated conditional distribution $\wh{P}_{Y|X=\bx}$.
To achieve this purpose,  we design a 
conditional diffusion model.
Under some conditions  delineated within our theoretical framework, we rigorously establish an upper bound for the total variation ($\mathrm{TV}$) 
distance  between the estimated distribution $\wh{P}_{Y|X=\bx}$ and the true conditional distribution $P_{Y|X=\bx}$. 
Furthermore, drawing upon the insights gleaned from our theoretical investigations, we undertake the construction of asymptotic normality within the distribution of the generated data samples. 
Specifically, assuming  that
$\{\wh{Y}_{\bx,i}\}_{i=1}^M$ are i.i.d. drawn from  
$\wh{P}_{Y|X=\bx}$, we  denote
$
\overline{\wh{Y}}_{\bx}:=\frac{1}{M}\sum_{i=1}^M \wh{Y}_{\bx,i}
$
and
$\wh{S}^2_{\bx}:=\frac{1}{M-1}\sum_{i=1}^M \left(\wh{Y}_{\bx,i}-\overline{\wh{Y}}_{\bx}\right)\cdot\left(\wh{Y}_{\bx,i}-\overline{\wh{Y}}_{\bx}\right)^{\top}$.
Therefore, we can deduce  that
$
\sqrt{M}\wh{S}_{\bx}^{-1}\left(\overline{\wh{Y}}_{\bx}-f_0(\bx)\right)
$ 
weakly converges to the standard Gaussian distribution under some certain conditions.
This pivotal result forms a robust foundation for the application of our methodology in statistical inference tasks within the intricate domain of deep nonparametric regression.
\subsection{Main Contributions}
Our main contributions can be summarized as follows:
\begin{itemize}
\item 
This work 
explores statistical inference for deep nonparametric regression. Central to this pursuit is the development of a conditional diffusion model, which serves as a cornerstone in estimating the target conditional distribution. 
We rigorously establish the asymptotic convergence results for the deep estimation, allowing the derivation of a validated confidence interval.
In addition, we have designed a series of simulation experiments and real-data analyzes to thoroughly evaluate the efficacy and numerical efficiency of our proposed method.
\item 
We develop a novel conditional diffusion model, representing a methodological advancement within the domain of diffusion models. Theoretically, we meticulously construct comprehensive end-to-end analytical frameworks tailored to our proposed conditional diffusion models. Specifically, we first conduct a rigorous theoretical analysis focusing on deep score estimation. Then, we rigorously establish an upper bound for the 
$\mathrm{TV}$ distance
between the law of the generated data and the underlying conditional distribution.

\end{itemize}

\subsection{Related Work}
In this section, we discuss related work, including deep nonparametric regression, statistical inference, and diffusion models.

\noindent
\textbf{Deep nonparametric regression:}
The integration of DNNs into regression analysis has ignited significant interest among researchers aiming to uncover intricate relationships between response and predictor variables. Seminal works by \cite{bauer2019deep, schmidt2020nonparametric, kohler2021rate,farrell2021deep, jiao2023deep,bhattacharya2023deep} have spearheaded this exploration, focusing on leveraging DNNs to learn the underlying regression  from data, thus diverging from conventional nonparametric approaches.
Theoretical foundations supporting the efficacy of DNN-based regression models are rooted in empirical process theory \cite{van1997weak,van2000asymptotic,gyorfi2002distribution,kosorok2008introduction} 
and deep approximation theory
\cite{lu2021deep,petersen2018optimal,yarotsky2017error}. These frameworks provide invaluable insights into the convergence rates of deep nonparametric regression, elucidating their minimax optimal convergence rates under some conditions, as established in \cite{bauer2019deep, schmidt2020nonparametric, kohler2021rate,farrell2021deep, jiao2023deep,bhattacharya2023deep}. 
In addition, deep nonparametric regression can address the curse of dimensionality, an ubiquitous challenge in high-dimensional data settings. This capability is particularly notable when the underlying regression function exhibits a composite structure or when the data conform to low-dimensional structures, such as manifolds.
Comparative analyses between deep nonparametric regression and traditional counterparts underscore the superiority of the former, especially in scenarios characterized by high-dimensional data.  While traditional nonparametric regression methods may struggle with the dimensionality of the data, deep nonparametric regression excels in deciphering intricate patterns and capturing nuanced relationships, offering a more robust and comprehensive framework for regression analysis.
For further details on deep nonparametric regression, we refer the readers to 
\cite{bauer2019deep, schmidt2020nonparametric, kohler2021rate,farrell2021deep, jiao2023deep,bhattacharya2023deep} and the references therein.

\noindent
\textbf{Statistical inference:}
Statistical inference is paramount in statistics, providing a structured methodology for extracting meaningful insights from observed data. In the context of nonparametric regression, the initiation of statistical inference necessitates the establishment of asymptotic convergence for estimators. 
Asymptotic convergence  \cite{van1997weak, van2000asymptotic,kosorok2008introduction} serves as a fundamental characteristic of  estimators, elucidating their behavior as sample sizes expand indefinitely.   Such convergence lays the theoretical groundwork for numerous statistical methodologies, facilitating the construction of asymptotic confidence intervals. Confidence intervals play a pivotal role in statistical inference, enabling the quantification of uncertainty surrounding estimates of population parameters. By constructing intervals centered around  estimates, confidence intervals provide a range of plausible values for the true parameter, accompanied by a predetermined level of confidence. This allows us to assess the precision of the estimates and evaluate the reliability of the findings, thereby enhancing the robustness of statistical estimators.
To elucidate, we introduce some pertinent methods, such as  
nearest neighbor estimator \cite{biau2015lectures}, 
kernel estimation \cite{wasserman2006all},
distributional random forests \cite{cevid2022distributional},
and conformal prediction
\cite{vovk2005algorithmic,lei2018distribution,fontana2023conformal},
among others.

The nearest neighbor estimator, a prominent method in nonparametric estimation, has garnered widespread usage, as detailed in 
\cite{biau2015lectures}.
Given $\bx \in \mathbb{R}^{d_{\mcX}}$,  
the dataset   
$$
\left\{\left(X_1, Y_1\right), \ldots,\left(X_n, Y_n\right)\right\}
$$ 
are rearranged based on the ascending order of their distances from $\bx$. 
This reordering yields a new sequence of data pairs:
$$
\left\{\left(X_{(1)}(\bx), Y_{(1)}(\bx)\right), \ldots,\left(X_{(n)}(\bx), Y_{(n)}(\bx)\right)\right\}.
$$
Here, $X_{(i)}(\bx)$, $i \in \{1,\ldots,n\}$, is termed the $i$th nearest neighbor of $\bx$.
Consequently, the nearest neighbor estimation is defined as
$$
f_n(\bx)=\frac{1}{n} \sum_{i=1}^{n} v_i Y_{(i)}(\bx),
$$
where $(v_1,\ldots,v_n)$ is a weight vector summing to one.
Specifically, if we set $v_i=\frac{1}{k} I(1\leq i \leq k)$ for $k \in \{1,\ldots,n\}$, then we obtain the 
$k$-nearest neighbor estimation, formulated as 
$$
f_n(\bx)=\frac{1}{k} \sum_{i=1}^{k} Y_{(i)}(\bx).
$$
Under certain conditions, the asymptotic convergence of 
$k$-nearest neighbor estimation can be established \cite{biau2015lectures}.
To enhance the implementation of the nearest neighbor estimator, \cite{steele2009exact,biau2010rate} employ the bagging technique to adaptively select weights in a distributional manner, resulting in the distributional nearest neighbor estimator.
The inference of distributional nearest neighbor estimators is further explored in \cite{demirkaya2024optimal}, where a two-scale distributional nearest neighbor estimator is proposed and shown to be asymptotically normal.

Kernel  estimation can be  explicitly expressed as \eqref{kernelest0}.
Under some conditions, the asymptotic normality of the kernel estimation
can be obtained, as discussed in \cite{wasserman2006all}. 
Additionally, kernel estimation can be viewed as a variation of nearest neighbor estimation when the weights of the nearest neighbor estimator are represented in a kernel form.
Furthermore, conformal prediction is a robust technique that ensures valid predictive inference by leveraging conformity scores derived from a fitted predictive model and calibration data. This method systematically evaluates the proximity of new data points to existing data and utilizes these conformity scores to construct reliable prediction intervals.
Fundamentally, by ranking the conformity scores and selecting those corresponding to a desired confidence level, conformal prediction provides prediction intervals with guaranteed coverage probabilities, offering a principled approach to uncertainty quantification in predictive modeling.
For further details, refer to \cite{vovk2005algorithmic,lei2018distribution,fontana2023conformal}.
Recently, \cite{liu2024novel}   introduced a perturbation-assisted inference framework for uncertainty quantification, which uses synthetic data generated by generative models.

Distributional random forests, as pioneered by \cite{cevid2022distributional}, epitomize a noteworthy advancement within the domain of Random Forest algorithms, originally conceptualized by \cite{breiman2001random}. 
This novel framework is tailored for the estimation of multivariate conditional distributions, thus furnishing a versatile tool conducive to predicting diverse targets, including conditional average treatment effects, conditional quantiles, and conditional correlations.
Furthermore, the consistency and convergence rate of distributional random forests are also expounded upon in \cite{cevid2022distributional}. These theoretical results provide a pivotal understanding of the methodology's performance characteristics, laying a solid foundation for its rigorous evaluation and application in empirical studies.
Recent developments in this burgeoning field have been elucidated in \cite{naf2023confidence}, with a specific emphasis on the asymptotic convergence and inference of distributional random forests.

\noindent
\textbf{Diffusion models:}
Diffusion models, a pivotal category within the domain of deep generative learning methods, have garnered considerable attention in   research efforts \cite{sohl2015deep,song2019generative,ho2020denoising,song2020improved,song2020score,nichol2021improved,benton2024denoising}. 
These models are composed of two fundamental components: the forward and backward processes, which are designed to introduce and remove noises, respectively.
In the backward process, a critical component involves training a score neural network. This network is systematically trained to effectively facilitate sampling, allowing for the generation of samples that closely align with the target distribution.
In applications, diffusion models have emerged as pivotal catalysts for advancements across diverse domains, notably including  text-to-image generation, natural language processing, and image and audio synthesis
\cite{dhariwal2021diffusion,ho2022cascaded,rombach2022high,saharia2022photorealistic,zhang2023adding,han2022card,austin2021structured,li2022diffusion}.   
Moreover, the error analysis of diffusion models can be categorized into two primary paradigms: those wherein the error in score estimation is acknowledged and those wherein it remains unidentified. The former line  has been extensively pursued in studies \cite{chen2023improved,conforti2023score,lee2022convergence,lee2023convergence,benton2023linear,li2023towards,gao2023wasserstein}. Conversely, the latter one has been elucidated by \cite{oko2023diffusion,chen2023score,jiao2024latent},
whose endeavors are underpinned by an end-to-end framework. Here, the integration of score estimation theory plays a pivotal role in their theoretical developments. 
In this paper, we also establish the 
end-to-end error analysis of our proposed method.
Conditional diffusion models were first introduced in \cite{song2020score}, which
share similarities with diffusion models, but with a distinct emphasis on modeling conditional distributions. This distinction arises from their explicit incorporation of external factors into the modeling framework.  
For instance, within text-to-image diffusion models, visually striking images are generated by inputting a text prompt.
As a result, the score function within these models also adopts a conditional form.
For examples of conditional diffusion models, refer to \cite{han2022card,shen2023non,zhou2023testing},
 and for  theoretical exploration of these models, see \cite{fu2024unveil}.
For an in-depth exploration of diffusion models, interested readers are directed to the specialized review by  \cite{yang2023diffusion,chen2024overview}, which offers a comprehensive analysis elucidating the intricacies of diffusion model principles and their practical applications.
Furthermore, there exist alternative conditional generative learning methods, such as conditional generative adversarial networks \cite{mirza2014conditional}, conditional variational autoencoders \cite{sohn2015learning}, generative conditional distribution samplers \cite{zhou2023deep}, and conditional F\"ollmer flows \cite{chang2024deep}, among others. 
Recently,  the end-to-end error analyses for 
ODE-based  generative models have been conducted by \cite{chang2024deep,gao2024convergence,jiao2024convergence}. 
Nonetheless, it is imperative to acknowledge that these methods have yet to attain asymptotic convergence and establish statistical inference capabilities.

\subsection{Notations and Paper Organization}
We introduce the notations used in this paper. 
Let $[N]:=\{0,1,\cdots,N-1\}$ represent the set of integers ranging from 0 to $N-1$. 
Let $\mathbb{N}^+$ denote the set of positive integers.
For matrices $A, B \in \mathbb{R}^{d\times d}$, we assert $A\preccurlyeq B$ when the matrix $B - A$ is positive semi-definite. 
We denote by $\mI_d$
the identity matrix in $\mathbb{R}^{d\times d}$.
The $\ell^2$-norm of a vector $\by=\{y_1,\ldots,y_d\}^{\top}\in\mathbb{R}^d$ is defined by $\Vert\by\Vert:=\sqrt{\sum_{i=1}^{d}y_i^2}$, and the outer product of the vector $\by$ is defined by $\by^{\otimes2}:=\by\by^{\top}$. 
We denote $\Vert\by\Vert_0$ as the number of 
non-zero elements in $\by$.
Simultaneously, the operator norm of a matrix $A$ is articulated as $\Vert{A}\Vert:=\sup_{\Vert\by\Vert\leq{1}}\Vert A\by\Vert$. 
The function space $C^2(\mathbb{R}^d)$ encompasses functions that are twice continuously differentiable from $\mathbb{R}^{d}$ to $\mathbb{R}$. For any $f\in C^2(\mathbb{R}^{d})$, the symbols $\nabla{f}$, $\nabla^2{f}$, and ${\rm{\Delta}}f$ signify its gradient, Hessian matrix, and Laplacian, respectively.
The $L^{\infty}(K)$-norm, denoted as  
$\Vert{f}\Vert_{L^{\infty}(K)}:=\sup_{\by\in K}|f(\by)|$,
captures the supermum of the absolute values of a function over a  set $K \subset \mathbb{R}^d$. For a vector function $\boldsymbol{v}:\mathbb{R}^{d}\rightarrow\mathbb{R}^{d}$, the $L^{\infty}(K)$-norm is defined as $\Vert \boldsymbol{v} \Vert_{L^{\infty}(K)} := \sup_{\by\in K}\Vert \boldsymbol{v}(\by)\Vert$. The asymptotic notation $f(\by) = \mathcal{O}\left(g(\by)\right)$ is employed to signify that $f(\by)\leq Cg(\by)$ for some constant $C > 0$. Additionally, the notation $\widetilde{\mathcal{O}}(\cdot)$ is utilized to discount logarithmic factors in the asymptotic analysis.

The remainder of this paper is organized as follows.
In Section \ref{sec:prel}, we provide the necessary preliminaries, including an introduction to the OU process, definitions of $\mathrm{TV}$ distance, DNNs, and covering numbers. 
Section \ref{sec:method} presents a detailed description of the methodology. 
The end-to-end theoretical analysis for the conditional diffusion models and asymptotic normality are developed in Section \ref{sec:ta}.
Section \ref{sec:prs} contains the detailed proof sketches.
Section \ref{sec:na} shows the numerical experiments conducted. 
The conclusion is presented in Section \ref{sec:con}. 
In the Appendix, we provide detailed proofs for all lemmas and theorems presented in this paper.

\section{Preliminaries}\label{sec:prel}
In this section, we provide an overview of the preliminary concepts, including the OU process, definitions of the $\mathrm{TV}$ distance, DNNs, and covering numbers.

\noindent
\textbf{OU Process}: We consider  the Ornstein-Uhlenbeck (OU) process as follows
\begin{equation}\label{sde: OU}
    \mrd\ov{\bY}_t = -\ov{\bY}_t\mrd t + \sqrt{2}\mrd\bB_t,~ \ov{\bY}_0\sim p_{0}(\by), ~t\geq 0,
\end{equation}
where $\{\bB_t\}_{t\geq 0}$ is a Brownian motion in $\Rbb^{d_{\mcY}}$. 
Then, $\ov{\bY}_t$ has an explicit solution
$$
\ov{\bY}_t = e^{-t}\ov{\bY}_0 + e^{-t}\int_{0}^{t}\sqrt{2}e^{s}\mrd\bB_s,
$$
and
$$
\ov{\bY}_t|\ov{\bY}_0\sim\mcN(e^{-t}\ov{\bY}_0, (1-e^{-2t})\mI_{d_{\mcY}}).
$$
Therefore, when $t\rightarrow\infty$, we know $\ov{\bY}_t\rightarrow\mcN(\m0,\mI_{d_{\mcY}})$.

\begin{definition}[$\mathrm{TV}$ Distance]
\label{def:tv}
Given two probability distributions $\mu$ and $\nu$ on a sample space $\Omega\subseteq\Rbb^{d_{\mcY}}$, the $\mathrm{TV}$  distance between $\mu$ and $\nu$ is defined as:
$$
\mathrm{TV}(\mu,\nu)=\sup_{A \subseteq \Omega}\left|\mu(A)-\nu(A)\right|.
$$
For functions $f:\Omega\rightarrow \mathbb{R}$, define $\|f\|_{\Omega,\infty}:=\sup_{\by \in \Omega} |f(\by)|$.
Then, we can also write the TV distance as:
$$
\mathrm{TV}(\mu,\nu)=\sup_{\|f\|_{\Omega,\infty}\leq 1}\left|\int_{\Omega} f(\by)\mu(\mrd\by)-\int_{\Omega} f(\by)\nu(\mrd\by)\right|.
$$

\end{definition}

\begin{definition}[ReLU DNNs]\label{relufnns}
ReLU deep neural networks, denoted as NN$(L,M,J,K,\kappa)$, with depth $L$, width $M$,  
sparsity level $J$,
boundness $K$,  
and weight $\kappa$,
can be defined as
    \begin{equation*}
\begin{aligned}
            {\rm{NN}}(L,M,J,K,\kappa) = \Big\{&\bs(t,\bz): (\boldsymbol{W}_L{\rm{ReLU}}(\cdot) + \bb_L)\circ\cdots\circ(\boldsymbol{W}_1{\rm{ReLU}}(\cdot) + \bb_1)(
            [t, \bz^{\top}]^{\top}):\\
&
 \boldsymbol{W}_i \in \mathbb{R}^{d_{i+1}\times d_i},  \bb_i \in \mathbb{R}^{d_{i+1}}, i=0,1,\ldots,L-1,\\
& M := \max\{d_0,\ldots,d_L\},~\mathop{\sup}_{t,\bz }\Vert{\bs(t, \bz)}\Vert\leq{K},\\
&\mathop{\max}_{1\leq{i}\leq{L}}\{\Vert{\bb}_i\Vert_{\infty}, \Vert{\boldsymbol{W}_i}\Vert_{\infty}\}\leq{\kappa},\\
&\sum_{i=1}^L\left(\Vert{\boldsymbol{W}_i}\Vert_0 + \Vert{\bb_i}\Vert_0\right)\leq J\Big\}.
\end{aligned}
\end{equation*}
\end{definition}

\begin{definition}[Covering number]
    Let $\rho$ be a pseudo-metric on $\mathcal{U}$ and $S\subseteq\mathcal{U}$. For any $\delta > 0$, a set $A\subseteq\mathcal{U}$ is called a $\delta$-covering number of $S$ if for any $\ba\in S$ there exists $\bb\in A$ such that $\rho(\ba,\bb)\leq\delta$. The $\delta$-covering number of $S$, denoted by $\mathcal{N}(\delta,S,\rho)$, is the minimum cardinality of any $\delta$-covering of $S$.
\end{definition}

\section{Problem Setting}\label{sec:method}
In this section, we explicate the procedural intricacies of our methodology, which centrally revolves around the utilization of a conditional diffusion model to facilitate rigorous statistical inference. Initially, we undertake a comprehensive exposition on the conditional diffusion model, encompassing forward and reverse processes, score matching, as well as prior distribution replacement and Euler Maruyama (EM) discretization. 
Subsequently, leveraging the established conditional diffusion model, we can obtain data that closely approximate the target distribution. Through this rigorous approach, we achieve the foundational objective of statistical inference,  facilitating robust  estimations within deep nonparametric regression.
\\\\
\noindent
\textbf{Forward Process}: We rewrite the SDE \eqref{sde: OU} with conditional prior distribution as follows:
\begin{equation}\label{sde: cond_OU}
    \mrd\ov{\bY}_t = -\ov{\bY}_t\mrd t + \sqrt{2}\mrd\bB_t, ~ \ov{\bY}_0\sim p_0(\by|\bx).
\end{equation}
It determines a diffusion process,  starting with a conditional distribution $p_{0}(\by|\bx)$ for some observation $\bx\in\mcX$ at time $t=0$,  that approaches $\mcN(\m0,\mI_{d_\mcY})$ as $t\rightarrow\infty$. 
In this context, we set $p_{0}(\by|\bx)$ as the target conditional distribution
$P_{Y|X=\bx}$.
We denote by $p_{t}(\by|\bx)$  the conditional probability density function of $\ov{\bY}_t$.
The process \eqref{sde: cond_OU} describes how data are transformed into noises. Conditioning on $\ov{\bY}_0$ and $\bx$, the transition probability distribution from $\ov{\bY}_0$ to $\ov{\bY}_t$ is given by 
$\ov{\bY}_t|(\ov{\bY}_0,\bx)\sim\mcN(e^{-t}\ov{\bY}_0, (1-e^{-2t})\mI_{d_\mcY})$ for each $t\geq 0$.
\\\\
\noindent\textbf{Reverse Process}: 
According to
\cite{anderson1982reverse,haussmann1986time}, the SDE \eqref{sde: cond_OU} can be reversed if we know the score of the distribution at each time $t$, denoted as $\nabla_{\by}\log p_{t}(\by|\bx)$. The reverse SDE reads
\begin{equation}\label{sde: reverse_cond_OU}
    \mrd\bY_t^R = \left[-\bY_t^R - 2\nabla_{\by}\log p_{t}(\bY_t^R|\bx)\right]\mrd t + \sqrt{2}\mrd\widetilde{\bB}_t,
\end{equation}
where 
$\{\widetilde{\bB}\}_{t\geq 0}$ is a standard Brownian motion when time flows backwards from $t=\infty$ to $t=0$, and $\mrd t$ is an infinitesimal negative timestep. The reverse SDE \eqref{sde: reverse_cond_OU} describes how to generate samples from noises. 

For convenience, we reformulate the reverse SDE \eqref{sde: reverse_cond_OU} in a forward version by switching time direction $t\rightarrow T - t$:
\begin{equation}\label{sde: forward_cond_OU}
    \mrd\bY_t = \left[\bY_t + 2\nabla_{\by}\log p_{T-t}(\bY_t|\bx)\right]\mrd t + \sqrt{2}\mrd \bB_t, ~\bY_0\sim p_{T}(\by|\bx),
\end{equation}
where $\{\bB_t\}_{0\leq t\leq T}$ is the usual (forward) Brownian motion.
\\\\
\noindent\textbf{Score Matching:}
The process \eqref{sde: forward_cond_OU} transforms noises into samples from 
$p_{0}(\by|\bx)$. 
However, direct simulation of this process \eqref{sde: forward_cond_OU} poses challenges due to the unavailability of the conditional score function $\nabla_{\by}\log p_{t}(\by|\bx)$.
Consequently, the drift term $\bb(t,\by,\bx) := \by + 2\nabla_{\by}\log{p}_{t}(\by|\bx)$ remains inaccessible.
To address this challenge,  we implement score matching methods \cite{hyvarinen2005estimation,vincent2011connection} to derive an estimator for the drift term.
Specifically, we  introduce the following population-level loss function:
$$
\mcL(\bs) := \frac{1}{T - T_0}\int_{T_0}^{T}\mathbb{E}_{\ov\bY_{t}, \bx}\Vert \bs(t,\ov{\bY}_t,\bx) - \bb(t,\ov{\bY}_t,\bx) \Vert^2\mrd t,
$$
where $0 < T_0 < T < \infty$. We notice that the selection of $T_0$ is motivated by the singularity of $\partial_t\bb(t,\by,\bx)$ at $t = 0$. 
With slight notational abuse, 
using denoising score matching allows us to represent this objective as:
$$
\begin{aligned}
\mcL(\bs) 
&= \frac{1}{T - T_0}\int^{T}_{T_0}\Ebb_{\ov\bY_{t},\bx}\left\Vert \bs(t,\ov{\bY}_t,\bx) - \ov{\bY}_t - 2\nabla_{\by}\log p_{t}(\ov{\bY}_t|\ov{\bY}_0,\bx) \right\Vert^2\mrd t\\
&=\frac{1}{T - T_0}\int^{T}_{T_0}\Ebb_{(\ov{\bY}_0,\bx)}\Ebb_{\ov{\bY}_t|(\ov{\bY}_0,\bx)}\left\Vert\bs(t,\ov{\bY}_t,\bx) -\ov{\bY}_t + \frac{2\left(\ov{\bY}_t - e^{-t}\ov{\bY}_0\right)}{1-e^{-2t}} \right\Vert^2\mrd t\\
&=\frac{1}{T - T_0}\int_{T_0}^{T}\Ebb_{(\ov{\bY}_0,\bx)}\Ebb_{\bZ}\left\Vert \bs(t,e^{-t}\ov{\bY}_0 + \sqrt{1-e^{-2t}}\bZ,\bx) -e^{-t}\ov{\bY}_0 + \frac{(1 + e^{-2t})\bZ}{\sqrt{1-e^{-2t}}}\right\Vert^2\mrd t.
\end{aligned}
$$
In practice, we only have access to observable data.
Specifically, given $n$ i.i.d. samples 
$\left\{(\ov{\bY}_{0,i}, \bx_{i})\right\}_{i=1}^{n}$ from $p_{0}(\by|\bx)p(\bx)$, $m$ i.i.d. samples $\left\{(t_j, \bZ_j)\right\}_{j=1}^{m}$ from $U[T_0, T]$ and $\mcN(0, \mI_{d_\mcY})$, we can utilize the empirical risk minimizer (ERM) as the  estimator of  the conditional score function. This ERM, denoted as $\wh{\bs}$, is determined by
\begin{equation}\label{eq: edrift}
\wh{\bs}\in\mathop{\arg\min}_{\bs\in\mrN}\wh{\mcL}(\bs),
\end{equation}
where $\mrN$ refers to ReLU DNNs defined in Definition \ref{relufnns}, and 
$\wh\mcL(\bs)$ denotes the empirical risk loss function, defined as
$$
\wh\mcL(\bs):=\frac{1}{mn}\sum_{j=1}^{m}\sum_{i=1}^{n}\left\Vert \bs(t_j,e^{-t_j}\ov{\bY}_{0,i} + \sqrt{1-e^{-2t_j}}\bZ_j, \bx_i) -e^{-t_j}\ov{\bY}_{0,i} + \frac{(1 + e^{-2t_j})\bZ_j}{\sqrt{1-e^{-2t_j}}}\right\Vert^2.
$$
\\\\
\noindent
\textbf{Prior Distribution Replacement and EM Discretization:}
Given the estimated conditional score function $\wh{\bs}$, we can  establish an SDE  initializing from the prior distribution $p_{T}(\by|\bx)$, defined as 
\begin{equation}\label{sde: sampling}
    \mrd\wh{\bY}_t = \wh{\bs}(T-t,\wh{\bY}_t,\bx)\mrd t + \sqrt{2}\mrd \bB_t, \quad {\wh{\bY}}_0 = \bY_0\sim p_{T}(\by|\bx), \quad 0\leq t\leq T-T_0.
\end{equation}
Subsequently, we can employ a discrete-approximation method for the sampling dynamics \eqref{sde: sampling}. Let
$$
0 = t_0 < t_1 < \cdots < t_N = T-T_0, \quad N\in\mathbb{N}^{+},
$$
be the discretization points on $[0, T-T_0]$. We consider the explicit EM scheme:
\begin{equation}\label{sde: EM_scheme}
    \mrd\wt{\bY}_t = \wh{\bs}(T-t_k, \wt{\bY}_{t_k}, \bx)\mrd t + \sqrt{2} \mrd \bB_t, \quad \wt{\bY}_0 = \bY_0\sim p_{T}(\by|\bx), \quad t\in[t_k, t_{k+1}],
\end{equation}
for $k=0,1,\cdots,N-1$. Unfortunately, the conditional distribution $p_{T}(\by|\bx)$ is unknown, rendering it impossible to utilize dynamics \eqref{sde: EM_scheme} for generating new samples.
Thus, we substitute $p_{T}(\by|\bx)$ with $\mcN(\m0,\mI_{d_\mcY})$ since  $p_{T}(\by|\bx)$ converges to  $\mcN(\m0,\mI_{d_\mcY})$ as $T$  tends to infinity. This adaptation presents a practically viable strategy.
Therefore, the sampling dynamics  take the following form:
\begin{equation}\label{sde: distri_replace}
    \mrd\ck{\bY}_t = \wh{\bs}(T-t_k, \ck{\bY}_{t_k}, \bx)\mrd t + \sqrt{2}\mrd \bB_t,\quad \ck{\bY}_0 \sim\mcN(\m0,\mI_{d_\mcY}),\quad t\in[t_k, t_{k + 1}].
\end{equation}
Consequently, we can employ this dynamics \eqref{sde: distri_replace} to generate new samples that are approximately distributed according to the target conditional distribution.
\\\\
\noindent
\textbf{Statistical Inference:}
With the conditional diffusion  model \eqref{sde: distri_replace}, we can obtain a conditional distribution denoted as  $\wh{P}_{Y|X=\bx}$, which approximates 
the target conditional distribution $P_{Y|X=\bx}$, as illustrated in Theorem \ref{th: end_to_end_convergence}.
Subsequently, running \eqref{sde: distri_replace} $M$ times  produces the generated data set  
$\{\wh{Y}_{\bx,i}\}_{i=1}^M$ which are i.i.d. drawn from   $\wh{P}_{Y|X=\bx}$.
Let
$
\overline{\wh{Y}}_{\bx}:=\frac{1}{M}\sum_{i=1}^M \wh{Y}_{\bx,i}
$
and
$\wh{S}^2_{\bx}:=\frac{1}{M-1}\sum_{i=1}^M \left(\wh{Y}_{\bx,i}-\overline{\wh{Y}}_{\bx}\right)\cdot
\left(\wh{Y}_{\bx,i}-\overline{\wh{Y}}_{\bx}\right)^{\top}
$.
Thus, we can conclude  that
$
\sqrt{M}\wh{S}_{\bx}^{-1}\left(\overline{\wh{Y}}_{\bx}-f_0(\bx)\right)
$ 
weakly converges to the standard Gaussian distribution under some conditions, as demonstrated in Theorem \ref{th: weakcon}.
Then, given $\alpha \in (0,1)$, we can construct the asymptotic $1-\alpha$ confidence interval for $f_0(\bx)$, denoted as
\begin{align}\label{asycond}
\left[\overline{\wh{Y}}_{\bx}-Z_{\alpha/2}\wh{S}_{\bx}/\sqrt{M},\overline{\wh{Y}}_{\bx}+Z_{\alpha/2}\wh{S}_{\bx}/\sqrt{M}\right].
\end{align}
Here, 
$Z_{\alpha/2}$ denotes the upper  quantile of a standard Gaussian distribution corresponding to the $\alpha/2$ significance level.

Overall, the structural framework of our proposed method is encapsulated within the following Algorithm \ref{sampling_algorithm}.
\begin{algorithm}[H]
\caption{
Statistical Inference via Conditional Diffusion Models
}
\label{sampling_algorithm}
\begin{algorithmic}
\STATE 
1. {\bf Input:}
$T_0$, $T$, $M$, $\alpha \in (0,1)$, 
$\left\{(\ov{\bY}_{0,i}, \bx_{i})\right\}_{i=1}^{n}\sim p_{0}(\by|\bx)p(\bx)$,  $\left\{(t_j, \bZ_j)\right\}_{j=1}^{m}$ $\sim$ $U[T_0, T]$ and $\mcN(0, \mI_{d_\mcY})$.
\STATE
2. {\bf Score estimation:} 
Obtain  $\widehat{\boldsymbol{s}}$ by  \eqref{eq: edrift}.
\STATE 
3. {\bf Data Generation:} 
Running the sampling dynamic \eqref{sde: distri_replace} $M$ times  generates the  data set  
$\{\wh{Y}_{\bx,i}\}_{i=1}^M$.
\STATE
4. Compute
$
\overline{\wh{Y}}_{\bx}=\frac{1}{M}\sum_{i=1}^M \wh{Y}_{\bx,i}
$
and
$\wh{S}^2_{\bx}=\frac{1}{M-1}\sum_{i=1}^M \left(\wh{Y}_{\bx,i}-\overline{\wh{Y}}_{\bx}\right)\cdot
\left(\wh{Y}_{\bx,i}-\overline{\wh{Y}}_{\bx}\right)^{\top}
$.

5. {\bf Output:} The asymptotic $1-\alpha$ confidence interval \eqref{asycond}.       
    \end{algorithmic}
\end{algorithm}

\section{Theoretical Analysis}\label{sec:ta}
In this section, we present our theoretical findings. Initially, we establish the TV distance error bound between the generated distribution and the target distribution, providing an end-to-end assessment of the conditional diffusion model. Subsequently, we develop the framework for asymptotic normality. To facilitate this analysis, we introduce several essential assumptions.

\begin{assumption}[Bounded Support]\label{ass: bounded_support}
    For any $\bx\in[0, B_{\mcX}]^{d_{\mcX}}$, the conditional density function $p_0(\by|\bx)$ is supported on $[0,1]^{d_\mcY}$.
\end{assumption}

\begin{assumption}[Lipschitz Continuity]\label{ass: Lip_target}
    The conditional density function $p_{0}(\by|\bx)$ is L-Lipschitz continuous w.r.t. $\by$ on $\mathbb{R}^{d_\mcY}\times[0,B_{\mcX}]^{d_{\mcX}}$.
\end{assumption}

\begin{assumption}[Lipschitz Score]\label{ass: smoothness_y}
    The conditional score function $\nabla_{\by}\log p_t(\by|\bx)$ is $\beta$-Lipschitz continuous w.r.t. $\bx$ on $[T_0, T]\times\Rbb^{d_\mcY}\times[0,B_{\mcX}]^{d_{\mcX}}$.
\end{assumption}
\begin{remark}
These foundational assumptions play a crucial role in advancing the theoretical framework of diffusion models.   Specifically, the introduction of the bounded support assumption
has been documented in the works such as \cite{lee2023convergence,li2023towards,oko2023diffusion}.
Meanwhile, in nonparametric regression, the imposition of a bounded assumption concerning the response variable is widely observed in the literature \cite{bauer2019deep,gyorfi2002distribution,kohler2021rate,farrell2021deep}. This assumption carries significant technical implications, with the potential for further generalization to encompass unbounded scenarios through the incorporation of an exponential tail.
The Lipschitz continuity 
imposed on the target distribution serves as a technical condition, ensuring mathematical rigor in our analysis.
Specifically, this assumption 
implies that  $p_0(\by|\bx)$ has no jump point on the boundary. 
It encompasses many common distributions, broadening the applicability of our theoretical framework.
 For simplicity, we assume $d_{\mcX}=d_{\mcY}=1$. One example of a distribution that satisfies the Lipschitz continuity is the conditional beta distribution, defined as $p_0(y|x):=\frac{y^{a(x)-1}(1-y)^{b(x) - 1}}{B(a(x), b(x))}\mI_{\{y\in[0,1]\}}$,
    where $x\in[0, 1]$, 
    $a(x), b(x) > 2$, and $B(\cdot, \cdot)$ denotes the beta function.
Additionally,  the modified Gaussian mixture distribution also satisfies this assumption. By truncating the distribution, we limit the range of the Gaussian components, ensuring they remain within a bounded interval. Scaling and normalization further adjust the distribution to meet the Lipschitz continuity condition, making it a suitable candidate for our analysis.
Moreover,  the Lipschitz property of score 
is also assumed in the study of diffusion models 
\cite{chen2023improved,lee2022convergence,lee2023convergence,gao2023wasserstein,chen2023score}.
    
\end{remark}

\begin{remark}
If $\bx$ takes values from a finite set, i.e. $\bx\in\{\bx_1,\bx_2,\cdots, \bx_{n_r}\}$ (a common scenario in tasks such as image generation with $n_r$ classifications), then $\bb(t,\by,\bx)$ can be seen as $\bb(t,\by,\bx)=\sum_{i=1}^{n_r}\bb_i(t,\by)\delta(\bx-\bx_i)$, 
where $\delta(\bx)$ refers to Dirac delta function. 
Partitioning the overall dataset into $n_r$ subsets, each composed of data with a single label, allows for the individual training of $\bb_i(t,\by)$. This scenario turns into unconditional generation, temporarily setting aside theoretical analysis in this paper.
\end{remark}

With these assumptions, we derive our first primary result, elucidating the convergence rate of the  TV distance between the generated and target distributions in an end-to-end manner, as shown in the following theorem.
\begin{theorem}[End-to-End Convergence Rate]\label{th: end_to_end_convergence}
Suppose that Assumptions \ref{ass: bounded_support}-\ref{ass: smoothness_y} hold and the drift estimator $\wh{\bs}$  defined in \eqref{eq: edrift} is constructed as introduced in Theorem \ref{th: drift_estimation}. Let $m > n^{\frac{d_{\mcX} + d_{\mcY} + 5}{ d_{\mcX} + d_{\mcY} + 3}}$. By choosing $\max_{k=0,\cdots,N-1}|t_{k+1} - t_k| = \mathcal{O}\Big(n^{-\frac{6}{(d_{\mcY} + 3)(d_{\mcX} + d_{\mcY} + 3)}}\Big)$, $T_0 = \mathcal{O}\left(n^{-\frac{1}{(d_\mcY + 3)(d_\mcX + d_{\mcY} + 3)}}\right)$,
and $T = \mathcal{O}(\log n)$,  we have 
\begin{equation*}
    \Ebb_{\mcD,\mcT,\mcZ}\left[\mathrm{TV}(\ck{p}_{T_0}, p_0)\right]= \widetilde{\mathcal{O}}\Big(n^{-\frac{1}{2(d_\mcY + 3)(d_\mcX + d_{\mcY} + 3)}}\Big).
\end{equation*}
\end{theorem}
\begin{remark}
In Theorem \ref{th: end_to_end_convergence}, we establish a rigorous analysis of the TV distance error bound associated with the learning distribution. This analysis is accomplished through the integration of score estimation, early stopping techniques, and numerical analysis of SDE. This result represents an end-to-end framework that aligns seamlessly with the established theoretical paradigm pertaining to diffusion models \cite{oko2023diffusion,chen2023score,jiao2024latent}. Furthermore, our result exhibits a superior level of comprehensiveness compared to prior works that focused solely on accurate score estimation error \cite{chen2023improved,conforti2023score,lee2022convergence,lee2023convergence,benton2023linear,li2023towards,gao2023wasserstein}.
Recently,  \cite{fu2024unveil}  provided an end-to-end convergence rate for conditional diffusion models, which is deemed optimal due to the introduction of a more stringent assumption (Assumption 3.3) compared to our own. Specifically, Assumption 3.3 posits that the target conditional density can be expressed as $p_{0}(\by|\bx)=\exp(-C_1\|\by\|^2/2)\cdot f(\by,\bx)$, where  
$f(\by, \bx) \geq C_2 $  and  $f$ is H\"older continuous, with $ C_1, C_2 > 0$.  Essentially, this assumption facilitates a factor of 
$\frac{T}{\log \frac{1}{T_0}}$ in the error bound for their score estimation, as detailed in Theorems 4.1-4.2 of \cite{fu2024unveil}.  Contrastingly, 
in  Theorem \ref{th: drift_estimation}, the error bound for drift estimation incorporates a factor of $
\frac{T\xi^{d_\mcY}}{T_0^5}  
$.  If we adopt the same conditions as Assumption 3.3 of \cite{fu2024unveil}, we can reduce this factor in Theorem 
\ref{th: drift_estimation}, leading to an optimal convergence rate. We consider this a promising direction for future work.
\end{remark}

Next, we proceed to present our second main result. Herein, we establish the property of asymptotic normality, which paves the way for statistical inference in deep nonparametric regression.
\begin{theorem}[Asymptotic Normality]\label{th: weakcon}
Suppose that the conditions of Theorem \ref{th: end_to_end_convergence} hold and set $n\geq M^{(d_\mcY + 3)(d_\mcX + d_{\mcY} + 4)}$.
Then, 
$
\sqrt{M}\wh{S}_{\bx}^{-1}\left(\overline{\wh{Y}}_{\bx}-f_0(\bx)\right)
$
weakly converges to standard normal distribution $\mcN(0, \mI_{d_\mcY})$,
where
$
\overline{\wh{Y}}_{\bx}=\frac{1}{M}\sum_{i=1}^M \wh{Y}_{\bx,i}
$
and
$\wh{S}^2_{\bx}=\frac{1}{M-1}\sum_{i=1}^M \left(\wh{Y}_{\bx,i}-\overline{\wh{Y}}_{\bx}\right)\cdot
\left(\wh{Y}_{\bx,i}-\overline{\wh{Y}}_{\bx}\right)^{\top}
$.
\end{theorem}
\begin{remark}
In Theorem \ref{th: weakcon}, we establish the asymptotic normality, providing valuable insights into statistical inference within deep nonparametric regression. This achievement constitutes a significant advancement, effectively bridging a notable gap within the existing theoretical landscape concerning deep nonparametric regression.  Consequently, it contributes to a deeper understanding of the statistical properties inherent in deep nonparametric regression.
\end{remark}

\section{Proof Sketch}\label{sec:prs}
In this section,  we provide a detailed sketch of the proof for   Theorem 
\ref{th: end_to_end_convergence}.
For any fixed $\bx\in [0,B_\mcX]^{d_\mcX}$, we denote the distributions of $\bY_t$, $\wh{\bY}_t$, $\wt{\bY}_t$, $\ck{\bY}_t$ as $p_{T-t}$, $\wh{p}_{T-t}$, $\wt{p}_{T-t}$, $\ck{p}_{T-t}$ respectively.  
Using the triangle inequality, we have
\begin{align}\label{eq1}
\Ebb_{\mcD,\mcT,\mcZ}[\mathrm{TV}(\ck{p}_{T_0}, p_0)]\leq \Ebb_{\mcD,\mcT,\mcZ}[\mathrm{TV}(\ck{p}_{T_0},\wt{p}_{T_0})] + \Ebb_{\mcD,\mcT,\mcZ}[\mathrm{TV}(\wt{p}_{T_0},p_{T_0})] + \mathrm{TV}(p_{T_0}, p_0).
\end{align}
To bound $\Ebb_{\mcD,\mcT,\mcZ}[\mathrm{TV}(\ck{p}_{T_0}, p_0)]$, we subsequently focus on bounding each of the three terms on the right-hand side of \eqref{eq1} separately in the following Sections \ref{sec:b1}-\ref{sec:b3}.
As a preliminary step,  we leverage Assumption \ref{ass: bounded_support} to establish the Lipschitz continuity of   the drift term $\bb(t,\by,\bx)$ w.r.t. $\by$.
 This foundational analysis provides a crucial basis for our subsequent investigations into the various components that contribute to the desired TV error bound.
\begin{lemma}\label{lem: smooth_lipschitz}
    Suppose that Assumption \ref{ass: bounded_support} holds, and let $T_0 < \frac{\log 2}{2}$. Then, the drift term $\bb(t,\by,\bx)$ is $\xi$-Lipschitz continuous w.r.t. $\by$ on $[T_0, T]\times\Rbb^{d_\mcY}\times[0,B_{\mathcal{X}}]^{d_{\mathcal{X}}}$, where $\xi\leq\frac{2d_\mcY}{T_0^2}$.
\end{lemma}

\subsection{Bound $\Ebb_{\mcD,\mcT,\mcZ}[\mathrm{TV}(\ck{p}_{T_0},\wt{p}_{T_0})]$}\label{sec:b1}
To bound $\Ebb_{\mcD,\mcT,\mcZ}[\mathrm{TV}(\ck{p}_{T_0},\wt{p}_{T_0})]$,
we first introduce a lemma which describes the KL divergence between two multivariate normal distributions.
\begin{lemma}\label{lem: KL_2_normal}
 Let $p(\by)=\mcN(\boldsymbol{\mu}_1,\boldsymbol{\Sigma}_1)$, $q(\by)=\mcN(\boldsymbol{\mu}_2,\boldsymbol{\Sigma}_2)$, then the $\mathrm{KL}$ divergence between $p$ and $q$ satisfies
$$
    \mathrm{KL}(p|q) = \frac{1}{2}\left[
    (\boldsymbol{\mu}_1-\boldsymbol{\mu}_2)^{\top}\boldsymbol{\Sigma}_2^{-1}(\boldsymbol{\mu}_1 - \boldsymbol{\mu}_2) - \log\left|\boldsymbol{\Sigma}_2^{-1}\boldsymbol{\Sigma}_1\right| + \mathrm{Tr}(\boldsymbol{\Sigma}_2^{-1}\boldsymbol{\Sigma}_1) - d_\mcY
    \right].
    $$
    In particular, when $q(\by)=\mcN(\m0,\mI_{d_\mcY})$, the result  simplifies to
    $$
    \mathrm{KL}(p|q) = \frac{1}{2}\left[\Vert\boldsymbol{\mu}_1\Vert^2-\log|\boldsymbol{\Sigma}_1| + \mathrm{Tr}(\boldsymbol{\Sigma}_1) - d_\mcY\right].
    $$
\end{lemma}

According to Lemma \ref{lem: KL_2_normal} and Pinsker's inequality  $\mathrm{TV}^2(p,q)\leq\frac{1}{2}\mathrm{KL}(p|q)$, we can bound $\Ebb_{\mcD,\mcT,\mcZ}[\mathrm{TV}(\ck{p}_{T_0},\wt{p}_{T_0})]$. 
\begin{theorem}\label{th:TV_initial_distribution}
    Under Assumption \ref{ass: bounded_support}, we have
    $$
    \Ebb_{\mcD,\mcT,\mcZ}[\mathrm{TV}(\ck{p}_{T_0},\wt{p}_{T_0})]\lesssim e^{-T}.
    $$
\end{theorem}

\subsection{Bound $\mathrm{TV}(p_{T_0},p_0)$}\label{sec:b2}
Indeed, from the definition of $p_{T_0}$, it naturally ensues that $p_{T_0}$ converges to $p_{0}$ as $T_0$ approaches zero. Consequently, the TV distance $\mathrm{TV}(p_{T_0}, p_0)$ can be effectively controlled through the parameter $T_0$.
\begin{theorem}\label{th: TV_early_stopping}
    Under Assumption \ref{ass: bounded_support} and Assumption \ref{ass: Lip_target},  for any $\bx \in[0,B_\mcX]^{d_\mcX}$, we have 
    $$
    \mathrm{TV}(p_{T_0}, p_0)=\mathcal{O}\left(\sqrt{T_0}\log^{(d_\mcY + 1)/2}\frac{1}{T_0}\right).
    $$ 
\end{theorem}

\subsection{Bound $\Ebb_{\mcD,\mcT,\mcZ}[\mathrm{TV}(\widetilde{p}_{T_0}, p_{T_0})]$}\label{sec:b3}
Using Girsanov theorem and 
Pinsker's inequality,
$\Ebb_{\mcD,\mcT,\mcZ}[\mathrm{TV}(\widetilde{p}_{T_0}, p_{T_0})]$ is primarily determined by the drift estimation. Through error decomposition, this can be further bounded by controlling statistical and approximation errors.
To elucidate, we begin by exploring the relationship between the loss function
$\mathcal{L}(\bs)$ and the $L^2$ approximation error $\mathbb{E}_{\ov{\bY}_t,\bx}\Vert \bs(t,\ov{\bY}_t,\bx) - \bb(t,\ov{\bY}_t,\bx)\Vert^2$.
\begin{lemma}\label{lem: connection}
For any $\bs:[T_0, T]\times\mathbb{R}^{d_\mcY}\times[0,B_{\mathcal{X}}]^{d_{\mathcal{X}}}\rightarrow\mathbb{R}^{d_\mcY}$, we have
\begin{equation*}
    \mathcal{L}(\bs) - \mathcal{L}(\bb) = \frac{1}{T - T_0}\int_{T_0}^{T}\mathbb{E}_{\ov\bY_t,\bx}\Vert \bs(t,\overline{\bY}_t,\bx) - \bb(t,\overline{\bY}_t,\bx)\Vert^2\mrd t.
\end{equation*}
\end{lemma}

\noindent\textbf{Error Decomposition.} Let $\mathcal{D}:=\{(\overline{\bY}_{0,i},\bx_{i})\}_{i=1}^{n}$, $\mathcal{T}:=\{t_j\}_{j=1}^{m}$, 
and
$\mathcal{Z}=\{\bZ_j\}_{j=1}^{m}$. We denote
$$
\overline{\mathcal{L}}_{\mathcal{D}}(\bs):=\frac{1}{n}\sum_{i=1}^{n}\ell_{\bs}(\overline{\bY}_{0,i},\bx_{i}), \quad \wh{\mathcal{L}}_{\mathcal{D},\mathcal{T},\mathcal{Z}}(\bs):=\frac{1}{n}\sum_{i=1}^{n}\wh{\ell}_{\bs}(\overline{\bY}_{0,i},\bx_{i}),
$$
where
$$
\ell_{\bs}(\overline{\bY}_{0,i},\bx_{i}) = \frac{1}{T - T_0}\int_{T_0}^{T}\mathbb{E}_{\bZ}\left\Vert \bs(t,e^{-t}\overline{\bY}_{0,i} + \sqrt{1-e^{-2t}}\bZ,\bx_i) -e^{-t}\overline{\bY}_{0,i} + \frac{(1 + e^{-2t})\bZ}{\sqrt{1-e^{-2t}}}\right\Vert^2\mrd t,
$$
and
$$
\wh{\ell}_{\bs}(\overline{\bY}_{0,i},\bx_{i}) = \frac{1}{m}\sum_{j=1}^{m}\left\Vert \bs(t_j,e^{-t_j}\overline{\bY}_{0,i} + \sqrt{1-e^{-2t_j}}\bZ_j, \bx_i) -e^{-t_j}\overline{\bY}_{0,i} + \frac{(1 + e^{-2t_j})\bZ_j}{\sqrt{1-e^{-2t_j}}}\right\Vert^2.
$$
Then, for any $\bs\in\mathrm{NN}$, it yields that
$$
\begin{aligned}
\mathcal{L}(\wh{\bs}) - \mathcal{L}(\bb) &= \mathcal{L}(\wh{\bs}) - 2\overline{\mathcal{L}}_{\mathcal{D}}(\wh{\bs}) + \mathcal{L}(\bb) + 2\left(\overline{\mathcal{L}}_{\mathcal{D}}(\wh{\bs}) - \wh{\mathcal{L}}_{\mathcal{D},\mathcal{T},\mathcal{Z}}(\wh{\bs})\right) + 2\left(\mathcal{L}_{\mathcal{D},\mathcal{T},\mathcal{Z}}(\wh{\bs}) - \mathcal{L}(\bb)\right)  \\
&\leq\mathcal{L}(\wh{\bs}) - 2\overline{\mathcal{L}}_{\mathcal{D}}(\wh{\bs}) + \mathcal{L}(\bb) + 2\left(\overline{\mathcal{L}}_{\mathcal{D}}(\wh{\bs}) - \wh{\mathcal{L}}_{\mathcal{D},\mathcal{T},\mathcal{Z}}(\wh{\bs})\right) + 2\left(\mathcal{L}_{\mathcal{D},\mathcal{T},\mathcal{Z}}(\bs) - \mathcal{L}(\bb)\right).
\end{aligned}
$$
Taking expectations, followed by taking the infimum over $\bs\in\mathrm{NN}$ on both sides of the above inequality, it holds that
$$
\begin{aligned}
&~\mathbb{E}_{\mathcal{D},\mathcal{T},\mathcal{Z}}\left(\frac{1}{T - T_0}\int_{T_0}^{T}\mathbb{E}_{\ov\bY_t,\bx}\Vert\wh{\bs}(t,\overline{\bY}_t,\bx) - \bb(t,\overline{\bY}_t,\bx) \Vert^2\mrd t\right)\\
=&~\mathbb{E}_{\mathcal{D},\mathcal{T},\mathcal{Z}}\mathcal{L}(\wh{\bs}) - \mathcal{L}(\bb)\\
\leq &~ \mathbb{E}_{\mathcal{D},\mathcal{T},\mathcal{Z}}\left(\mathcal{L}(\wh{\bs}) - 2\overline{\mathcal{L}}_{\mathcal{D}}(\wh{\bs}) + \mathcal{L}(\bb)\right) + 2\mathbb{E}_{\mathcal{D},\mathcal{T},\mathcal{Z}}\left(\overline{\mathcal{L}}_{\mathcal{D}}(\wh{\bs}) - \wh{\mathcal{L}}_{\mathcal{D},\mathcal{T},\mathcal{Z}}(\wh{\bs})\right) + 2\mathop{\inf}_{\bs\in\mathrm{NN}}\left(\mathcal{L}(\bs) - \mathcal{L}(\bb)\right).
\end{aligned}
$$
In the above inequality, the terms $\mathbb{E}_{\mathcal{D},\mathcal{T},\mathcal{Z}}\left(\mathcal{L}(\wh{\bs}) - 2\overline{\mathcal{L}}_{\mathcal{D}}(\wh{\bs}) + \mathcal{L}(\bb)\right) + 2\mathbb{E}_{\mathcal{D},\mathcal{T},\mathcal{Z}}\left(\overline{\mathcal{L}}_{\mathcal{D}}(\wh{\bs}) - \wh{\mathcal{L}}_{\mathcal{D},\mathcal{T},\mathcal{Z}}(\wh{\bs})\right)$ and $\mathop{\inf}_{\bs\in\mathrm{NN}}\left(\mathcal{L}(\bs) - \mathcal{L}(\bb)\right)$ denote the statistical error and approximation error, respectively. Next, we bound these two errors in the following Lemmas 
\ref{lem: approx_err}-\ref{lem: stat_error}.
\begin{lemma}[Approximation Error]\label{lem: approx_err}
Suppose that Assumption \ref{ass: bounded_support} holds. Given an approximation error $\epsilon > 0$, if $0 < T_0 < \frac{\log 2}{2}$, we can choose a neural network $\bs$ with the following structure:
$$
L = \mathcal{O}\left(\log{\frac{1}{\epsilon}} + d_\mcX + d_{\mathcal{Y}}\right),
M = \mathcal{O}\Bigg(\frac{d_\mcY^{\frac{5}{2}}(T-T_0)}{T_0^3}\left(\log{\frac{d_\mcY}{\epsilon T_0}}\right)^{\frac{d_\mcY+1}{2}}\xi^{d_\mcY}(\beta B_{\mcX})^{d_{\mcX}}\epsilon^{-(d_\mcX + d_{\mcY} + 1)}\Bigg),
$$
$$
J = \mathcal{O}\Bigg(\frac{d_\mcY^{\frac{5}{2}}(T-T_0)}{T_0^3}\left(\log{\frac{d_\mcY}{\epsilon T_0}}\right)^{\frac{d_\mcY+1}{2}}\xi^{d_\mcY}(\beta B_{\mathcal{X}})^{d_{\mcX}}\epsilon^{-(d_\mcX+d_{\mathcal{Y}}+1)}\left(\log{\frac{1}{\epsilon}} + d_\mcX + d_{\mathcal{Y}}\right)\Bigg), 
$$
$$
K = \mathcal{O}\Bigg(\frac{\sqrt{d_\mcY\log{\frac{d_\mcY}{\epsilon T_0}}}}{T_0}\Bigg),
\kappa=\mathcal{O}\Bigg(
\xi\sqrt{\log{\frac{d_\mcY}{\epsilon T_0}}}
\vee
\frac{(T-T_0)d_{\mcY}^{3/2}\sqrt{\log\frac{d_\mcY}{\epsilon T_0}}}{T_0^3}\vee\beta B_{\mathcal{X}}
\Bigg),
$$
$$
\gamma_1 = \mathcal{O}\left(\frac{20 d_\mcY^2}{T_0^2}\right), \gamma_2 = \mathcal{O}\Big(\frac{10 d_\mcY^{3/2}\sqrt{\log\frac{d_\mcY}{\epsilon T_0}}}{T_0^3}\Big),
$$
such that for any $t\in[T_0, T]$ and $\bx\in[0,B_\mcX]^{d_\mcX}$, we have
$$
\mathbb{E}_{\ov\bY_t}\Vert \bs(t,\overline{\bY}_t,\bx) - \bb(t,\overline{\bY}_t,\bx)\Vert^2\leq (1 + d_\mcY)\epsilon^2,
$$
where 
$\xi$ is defined in Lemma \ref{lem: smooth_lipschitz}, 
and $\gamma_1$ and $\gamma_2$ are the Lipschitz constants for $\bs$ with respect to $\by$ and $t$, respectively.
\end{lemma}

\begin{lemma}[Statistical Error]\label{lem: stat_error}
Given an approximation error $\epsilon > 0$, if $0 < T_0 < \frac{\log 2}{2}$, then the estimator $\wh{\bs}$ defined in \eqref{eq: edrift}, with the neural network structure introduced in Lemma \ref{lem: approx_err}, satisfies
\begin{equation*}
\mathbb{E}_{\mathcal{D},\mathcal{T},\mathcal{Z}}\left(\mathcal{L}(\wh{\bs}) - 2\overline{\mathcal{L}}_{\mathcal{D}}(\wh{\bs}) + \mathcal{L}(\bb)\right)
= \widetilde{\mathcal{O}}\left(\frac{1}{n}\cdot\frac{(T-T_0)\xi^{d_\mcY}(\beta B_{\mathcal{X}})^{d_{\mathcal{X}}}\epsilon^{-(d_\mcX + d_{\mathcal{Y}} + 1)}}{T_0^5}\right)
\end{equation*}
and
\begin{equation*}
\mathbb{E}_{\mathcal{D},\mathcal{T},\mathcal{Z}}\left(\overline{\mathcal{L}}_{\mathcal{D}}(\wh{\bs}) - \wh{\mathcal{L}}_{\mathcal{D},\mathcal{T},\mathcal{Z}}(\wh{\bs})\right) = \widetilde{\mathcal{O}}\left(\frac{1}{\sqrt{m}}\cdot\frac{(T-T_0)^{\frac{1}{2}}\xi^{\frac{d_\mcY}{2}}(\beta B_{\mathcal{X}})^{\frac{d_{\mathcal{X}}}{2}}\epsilon^{-\frac{d_\mcX + d_{\mathcal{Y}} + 1}{2}}}{T_0^{\frac{7}{2}}}\right).
\end{equation*}
\end{lemma}

Combining the error analysis for approximation and statistical errors presented in Lemmas \ref{lem: approx_err}-\ref{lem: stat_error}, we can derive the error bound for the drift estimation, as demonstrated in the following theorem.

\begin{theorem}[Error Bound for Drift Estimation]\label{th: drift_estimation} Suppose Assumption \ref{ass: bounded_support} and Assumption \ref{ass: smoothness_y} hold, let $0 < T_0 < \frac{\log 2}{2}$,
$T < \infty$. By choosing $\epsilon=n^{-\frac{1}{d_\mcX + d_{\mathcal{Y}} + 3}}$ in Lemmas \ref{lem: approx_err}-\ref{lem: stat_error}, the drift estimator $\wh{\bs}$ satisfies

\begin{equation*}
\begin{aligned}
&~\mathbb{E}_{\mathcal{D},\mathcal{T},\mathcal{Z}}\left(\frac{1}{T - T_0}\int_{T_0}^{T}\mathbb{E}_{\ov\bY_t,\bx}\Vert \wh\bs(t,\overline{\bY}_t,\bx) - \bb(t,\overline{\bY}_t,\bx)\Vert^2 \mrd t\right)\\
=&~\widetilde{\mathcal{O}}\left(
\frac{(T-T_0)\xi^{d_\mcY}(\beta B_{\mathcal{X}})^{d_{\mathcal{X}}}}{T_0^5}  
\left(n^{-\frac{2}{d_\mcX + d_{\mathcal{Y}} + 3}} + n^{\frac{d_\mcX + d_{\mathcal{Y}} + 1}{2(d_\mcX + d_{\mathcal{Y}} + 3)}}m^{-\frac{1}{2}}\right)
\right).
\end{aligned}
\end{equation*}
\end{theorem}
With Theorem \ref{th: drift_estimation}, we can give the upper bound of $\Ebb_{\mcD,\mcT,\mcZ}[\mathrm{TV}(\wt{p}_{T_0}, p_{T_0})]$, as shown in the following theorem.
\begin{theorem}\label{th: TV_sampling_error}
Suppose that assumptions of Theorem \ref{th: drift_estimation} hold, then we have
\begin{equation*}
\begin{aligned}
\Ebb_{\mcD,\mcT,\mcZ}[\mathrm{TV}(\wt{p}_{T_0}, p_{T_0})]&=\widetilde{\mathcal{O}}\Bigg(
    \frac{(T-T_0)\xi^{\frac{d_\mcY}{2}}(\beta B_{\mathcal{X}})^{\frac{d_\mcX}{2}}}{T_0^{\frac{5}{2}}}  
    \left(n^{-\frac{1}{d_\mcX + d_{\mathcal{Y}} + 3}} + n^{\frac{d_\mcX + d_{\mathcal{Y}} + 1}{4(d_\mcX + d_{\mathcal{Y}} + 3)}}m^{-\frac{1}{4}} \right) \\
    &~~~~~~~ + \left(d_\mcY\gamma_1 + \gamma_2\right)\sqrt{\sum_{k=0}^{N-1}(t_{k+1}-t_k)^3} + {d_\mcY}^{\frac{1}{2}
    }\sqrt{\sum_{k=0}^{N-1}(t_{k+1}-t_k)^2}\Bigg),
    \end{aligned}
\end{equation*}
\end{theorem}

Based on the analysis provided above, the synthesis of Theorems \ref{th:TV_initial_distribution}, \ref{th: TV_early_stopping}, and \ref{th: TV_sampling_error} yields the comprehensive end-to-end convergence rate as established in Theorem \ref{th: end_to_end_convergence}.
 
\section{Numerical Experiments}\label{sec:na}
In this section, we undertake a rigorous examination through comprehensive simulation studies in Section \ref{sec:sim} and real data analyses in Section \ref{sec:rda} to assess the efficacy and practical applicability of our proposed method. This aims to provide a detailed evaluation of our method's performance, robustness, and utility in both experimental settings and real-world scenarios.
\subsection{Simulation Studies}\label{sec:sim}
In our simulation studies, we consider the following three regression models:
\begin{itemize}
\item [(I.)] 
Assume that $Y=f_0(X)+\epsilon$, where $f_0(X)=\left(x_1-1\right)^2+\left(x_2+1\right)^3-3 x_3$ with $X=\left(x_1, x_2, x_3\right)^{\top}\sim U\left([0,1]^{3}\right)$ and $\epsilon \sim \mcN\left(0, 1\right)$.
\item [(II.)]
Assume that $Y=f_0(X)+\epsilon$, where 
$$
f_0(X)=\left(x_1-2+ x_2^2\right)^2+\left(3-x_2\right)^2+ \sqrt{x_3+1}\left(x_3-1\right)^2
$$
with  $X=\left(x_1, x_2, x_3\right)^{\top}\sim U\left([0,1]^3\right)$  
and $\epsilon\sim U([-1/2, 1/2])$.
\item [(III.)]
Assume that $Y=x_1^2+\frac{1}{2}\exp \left(\left(x_2+\frac{x_3}{3}\right)\right)+$ $x_4-x_5+\frac{1}{8}\left(1+x_2^2 + x_5^2\right) \times \epsilon$, 
where
$X=\left(x_1, x_2,x_3,x_4,x_5\right)^{\top}\sim  \mcN(\m0,\mI_{5})$,
and $\epsilon \sim \mcN\left(0, 1\right)$.
\end{itemize}
In the aforementioned Models 
(I)-(III), both bounded and unbounded cases are considered. These models are derived from the works of \cite{demirkaya2024optimal,zhou2023deep}, with slight modifications to their original formulations.
We implement Algorithm \ref{sampling_algorithm} with $\wt{M}$ replications,
and assess its performance using the following metrics: mean squared error (MSE), variance, bias, and coverage probability (CP):
\begin{align*}
&
\mbox{MSE}:=\frac{1}{\wt{M}}\sum_{j=1}^{\wt{M}}\left(\overline{\wh{Y}}_{x}^{(j)}-f_0(x)\right)^2,\\
&
\mbox{Variance}:=\frac{1}{\wt{M}}\sum_{j=1}^{\wt{M}}
\left(\overline{\wh{Y}}_{x}^{(j)}-\frac{1}{\wt{M}} \sum_{j=1}^{\wt{M}}\overline{\wh{Y}}_{x}^{(j)} \right)^2,\\
&
\mbox{Bias}^2:=\left(\frac{1}{\wt{M}} \sum_{j=1}^{\wt{M}}\overline{\wh{Y}}_{x}^{(j)} -f_0(x)\right)^2,\\
&
\mbox{CP}:=\frac{1}{\wt{M}}\sum_{j=1}^{\wt{M}} I\left(\widehat{T}_x^{(j)} \in [-Z_{\alpha/2},Z_{\alpha/2}]\right).
\end{align*}
Here, we define:
\begin{align*}
&\overline{\wh{Y}}_{\bx}^{(j)}=\frac{1}{M}\sum_{i=1}^M \wh{Y}_{\bx,i}^{(j)},\\
&\wh{S}^2_{j,\bx}=\frac{1}{M-1}\sum_{i=1}^M \left(\wh{Y}_{\bx,i}^{(j)}-\overline{\wh{Y}}_{\bx}^{(j)}\right)^2,\\
&
\widehat{T}_x^{(j)}=\sqrt{M}\wh{S}_{j,\bx}^{-1}\left(\overline{\wh{Y}}_{\bx}^{(j)}-f_0(\bx)\right),
\end{align*}
where $j \in \{1,\ldots, \wt{M}\}$ denotes the $j$-th simulation iteration.

The results from the simulation studies are comprehensively summarized in Tables \ref{Table:sim1}-\ref{Table:sim2}. 
Table  \ref{Table:sim1}  details the results for the fixed test point scenario, while Table \ref{Table:sim2} presents the results for the random test points scenario.
In Models (I)-(II), we randomly sample 10000 data points, whereas in Model (III) we randomly sample 20000 data points. From these samples, 90\% are allocated to the training set, as implemented in Algorithm \ref{sampling_algorithm}, with the remaining 10\% forming the test set.  Specifically,  in the fixed test point scenario, the first two test points in Table \ref{Table:sim1} are randomly selected from the test set, with the last point specified as 0.5 to ensure fairness of the evaluation. Conversely, in the random test point scenario depicted in Table \ref{Table:sim2}, we select $\wt{M}$ random test points from the test set. 
For these simulations, we set $M=100$ in Algorithm \ref{sampling_algorithm},  the simulation time $\wt{M}=1000$, and the confidence level $\alpha=5\%$.
In the random test point case, where the test points change in every simulation iteration, we focus primarily on the CP and MSE.
From Table \ref{Table:sim1} of the fixed test point case, we can see that the MSE, variance, and bias values approaching zero indicate the robustness and effectiveness of our method for  nonparametric estimation. Additionally, CP values nearing 0.95 suggest that our method provides effective confidence intervals, thereby validating our theoretical findings.
In Table  \ref{Table:sim2}, which presents the results for the random test points case, we observe that the CP values decrease by approximately two percentage points compared to the fixed test point scenario, yet they still approximate 95\%, reaffirming the robustness of our method across varying test points.
The minimal MSE in Table  \ref{Table:sim2} further substantiates the consistency and precision of our approach across different scenarios. 
These results collectively underscore the efficacy of our method in producing stable and reliable estimates, regardless of the variability in test point selection.

\begin{table}[H]
\caption{The CP, MSE, Variance, $\mathrm{Bias}^2$ at different fixed test points in Models (I)-(III) with $M=100$ and $\wt{M}=1000$.}
\label{Table:sim1}
\centering
\begin{tabular}{cccccc}
    \toprule
    Model& $\bx$ & CP &MSE &Variance &$\mathrm{Bias}^2$ 
    \\
    \midrule 
    \multirow{3}{*}{(I)}&(0.473, 0.557, 0.386)&94.80\%&0.029363&0.028232&0.001131
    \\
    &(0.854, 0.199, 0.807)&94.90\%&0.029545&0.02894&0.000605 \\
    &(0.500, 0.500, 0.500)&95.60\%&0.028246&0.028236&0.000010
    \\
    \midrule 
    \multirow{3}{*}{(II)}&(0.633, 0.014, 0.936)&94.80\%&0.019684&0.019164&0.000520
    \\
    &(0.069, 0.500, 0.765)&95.10\%&0.019218&0.019205&0.000013
    \\
    &(0.500, 0.500, 0.500)&93.70\%&0.021202&0.019228&0.001974
    \\
    \midrule
    \multirow{3}{*}{(III)}&(1.386, 1.333, -0.029, 0.583, -0.012)&95.20\%&0.017822&0.017600&0.000222\\
    &(1.270, 0.723, 0.764, 0.408, -0.162)&94.20\%&0.019546&0.017455&0.002091\\
    &(0.500, 0.500, 0.500, 0.500, 0.500)&93.60\%&0.019481&0.017344&0.002137\\
    \bottomrule
\end{tabular}
\end{table}

\begin{table}[H]
\caption{The CP and MSE at different random test points in Models (I)-(III) with $M=100$ and $\wt{M}=1000$.}
\label{Table:sim2}
\centering
\setlength{\tabcolsep}{15mm}
\begin{tabular}{ccc}
    \toprule
    Model&CP &MSE 
    \\
    \midrule 
    (I)&92.30\%&0.037975
    \\
    (II)&92.50\%&0.024809 \\
    (III)&90.10\%&0.030276\\
\bottomrule
\end{tabular}

\end{table}

\subsection{Real Data Analysis}\label{sec:rda}
In this section, we conduct real data analyzes to validate the practical efficacy of our proposed method. Our endeavor entails a rigorous examination of three  datasets: wine quality, abalone, and superconductivity. Notably, the wine quality dataset has been studied by \cite{chang2024deep}, and the abalone dataset is explored in the works of \cite{zhou2023deep,demirkaya2024optimal}.
The primary objective of our analysis is to construct a prediction interval for the response variable  $Y$ given the covariate vector $X$.  
We assume the relationship between $Y$ and $X$ is modeled as follows:
$$
Y = f_0(X) + \epsilon, ~ \epsilon\sim\mathcal{N}(0, \sigma^2),
$$
where $f_0$ represents   the unknown regression function.

Given the pair $(X, Y)$,  let 
$\{Y_j\}_{j=1}^{M}$ be $M$ independent copies of $Y$ given $X$, that is, $Y_j=f_0(X)+\epsilon_{j}$. Then, we have 
$$
Y - \frac{1}{M}\sum_{j=1}^{M}Y_{j} = \epsilon - \frac{1}{M}\sum_{j=1}^{M}\epsilon_{j}\sim\mathcal{N}\left(0, \left(1 + \frac{1}{M}\right)\sigma^2\right).
$$
The sample variance is given by
$$
S^2 = \frac{1}{M-1}\sum_{j=1}^{M}\left(Y_j - \frac{1}{M}\sum_{j=1}^{M}Y_j\right)^2 = \frac{1}{M-1}\sum_{j=1}^{M}\left(\epsilon_j - \frac{1}{M}\sum_{j=1}^{M}\epsilon_j\right)^2.
$$
Since $\epsilon, \epsilon_1, \epsilon_2, \cdots,\epsilon_M$ are independent, and considering the independence of $\frac{1}{M}\sum_{j=1}^{M}\epsilon_j$ and $S^2$, we obtain
$$
\frac{Y - \frac{1}{M}\sum_{j=1}^{M}Y_j}{S\sqrt{1 + \frac{1}{M}}}\sim t(M-1),
$$
where $t(M-1)$ denotes $t$-distribution with $M-1$ degrees of freedom.
Denote $\overline{Y} := \frac{1}{M}\sum_{j=1}^{M}Y_j$,  we can construct the $1-\alpha$ prediction interval of $Y$ as
$$
\left[\overline{Y} - t_{\alpha/2}(M-1)S\sqrt{1 + \frac{1}{M}}, \overline{Y} + t_{\alpha/2}(M-1)S\sqrt{1 + \frac{1}{M}}\right].
$$
To approximate the prediction interval in practice, we replace $\overline{Y}$ with $\overline{\wh{Y}}_{\bx}$ and $S$ with $\wh{S}_{\bx}$.
In our analysis, we set $\alpha=5\%$.
The comprehensive numerical results corresponding to these three datasets are provided in the following Subsections \ref{sec:r1}-\ref{sec:r3}. 
\subsubsection{The Wine Quality Dataset}\label{sec:r1}
The wine quality dataset, available at \url{http://www3.dsi.uminho.pt/pcortez/wine/}, represents a comprehensive collection of data comprising two distinct subsets: red and white vinho verde wine samples originating from the northern region of Portugal. Specifically, the dataset consists of 1599 samples of red wine and 4898 samples of white wine. Each sample is meticulously characterized by 11 predictor variables, encompassing essential physicochemical attributes such as \textit{fixed acidity}, \textit{volatile acidity}, \textit{citric acid}, \textit{residual sugar}, \textit{chlorides}, \textit{free sulfur dioxide}, \textit{total sulfur dioxide}, \textit{density}, \textit{pH}, \textit{sulphates}, and \textit{alcohol content}. The response variable within this dataset is the quality score, denoting the subjective assessment of wine quality, which spans a range from 0 to 10. For a more comprehensive understanding of this dataset and its underlying features, interested readers may refer to \cite{cortez2009modeling}.

Within our modeling framework, wine quality serves as the response variable, denoted as $Y \in \mathbb{R}$, while the remaining measurements constitute the covariate vector, represented by $X \in \mathbb{R}^{11}$. To evaluate the effectiveness of our model, we employ a standard data partitioning approach, allocating 85\% of the dataset for training and reserving the remaining 15\% for testing purposes. The numerical results are depicted in Figure \ref{wine}.
Specifically, the left panel of Figure \ref{wine} presents a comparison between the estimated conditional density and the actual density distribution of the test set. Remarkably,  it reveals a close correspondence between the estimated and actual density distributions. This alignment underscores the fidelity of our modeling approach in accurately capturing the underlying distribution of the response variable within the test dataset.
Furthermore, to examine the prediction performance of our method, we construct the  prediction interval for the wine quality in the test set. 
The prediction intervals are shown in the right panel of Figure \ref{wine}, with the actual wine qualities plotted as solid dots. The actual coverage rate for the test set is 95.38\%, close to the nominal level of 95\%. This efficient interval serves as a reliable measure of the uncertainty associated with our model predictions. Such a comprehensive evaluation reinforces the robustness and reliability of our modeling framework in predicting wine quality based on the specified covariates.

\begin{figure}[H]
    \centering
    \subfloat{
\includegraphics[width=0.5\columnwidth, keepaspectratio]{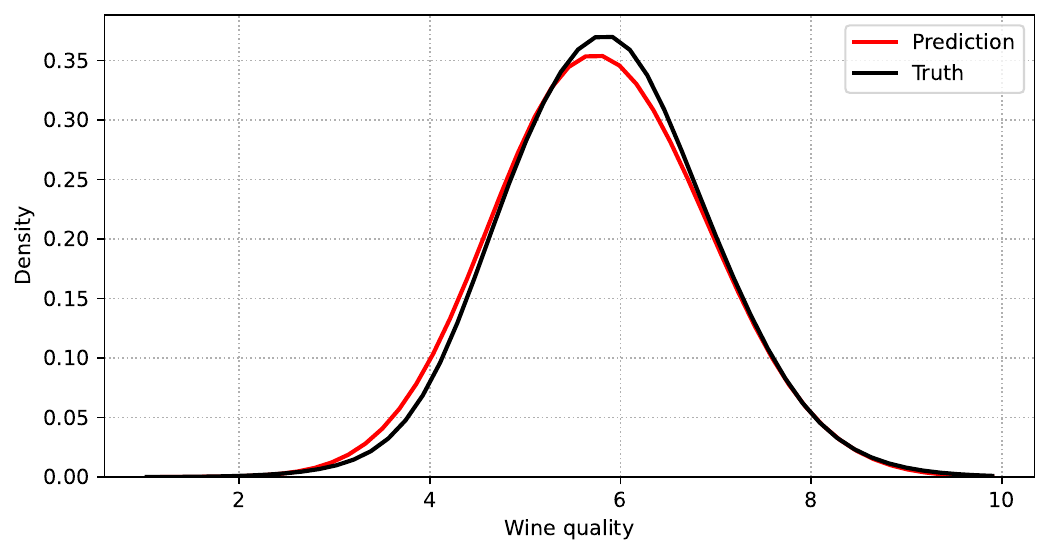}
    }
    \subfloat{
    \includegraphics[width=0.5\columnwidth, keepaspectratio]{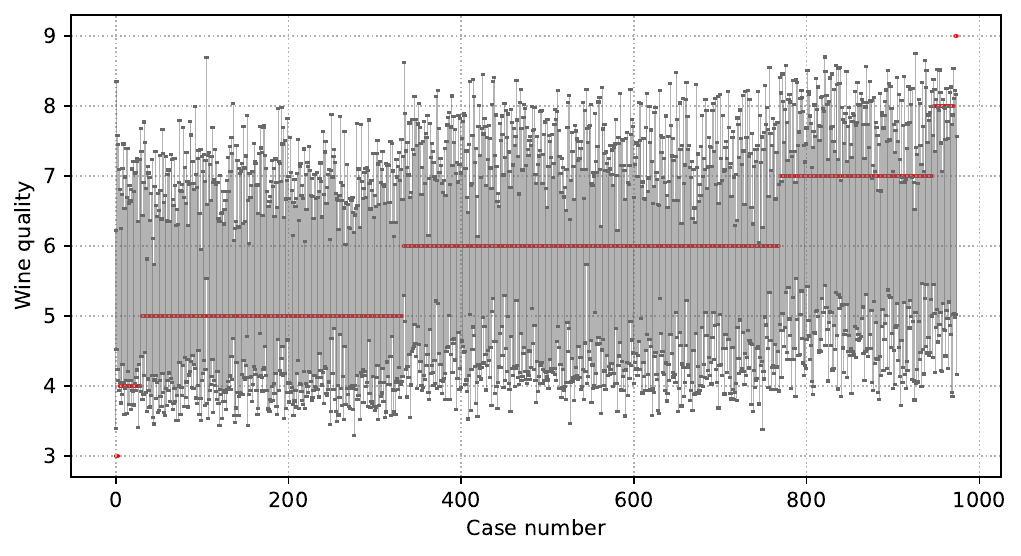}
    }
    \caption{(Left) Estimated density and actual density on the test set. (Right) The prediction intervals on the test set.}
    \label{wine}
\end{figure}

\subsubsection{The Abalone Dataset}\label{sec:r2}
The abalone dataset, accessible at \url{https://archive.ics.uci.edu/dataset/1/abalone}, presents a rich repository of data aimed at predicting the age of abalone based on physical measurements. Determining the age of the abalone traditionally involves a laborious process of cutting the shell, staining it, and manually counting the rings under a microscope, a tedious and time-consuming task. To circumvent this challenge, this dataset leverages other, more easily obtainable measurements to predict the age of abalone. However, solving this predictive task may require additional contextual information, such as weather patterns and location.
This dataset comprises 4177 observations across 8 input variables, including \textit{sex}, \textit{length}, \textit{diameter}, \textit{height}, \textit{whole weight}, \textit{shucked weight}, \textit{viscera weight}, and \textit{shell weight}. In particular, the response variable represents the number of rings found in the shell, which serves as a proxy for the age of the abalone. While all variables except for the categorical variable \textit{sex} are continuous, the \textit{sex} variable categorizes individuals into three groups: female, male, and infant.  
In our analysis, we designate the response variable as 
$Y\in\mathbb{R}$, representing rings, while the remaining measurements serve as the covariate vector  $X\in\mathbb{R}^8$.
 We partition 85\% of the dataset for training purposes, reserving the remaining 15\% for testing. Due to the discrete nature of the variable \textit{sex}, we employ one-hot encoding of this feature prior to training the model.

In Figure \ref{abalone}, our analysis of the abalone dataset reveals a significant alignment between the estimated density function and the true density function. 
Moreover, our calculated CP value of 94.44\% closely approximates the target value of 95\%.  This observation underscores the robustness of our method, affirming its reliability in accurately capturing the underlying data distribution and its efficiency in predictive modeling.
\begin{figure}[H]
    \centering
    \subfloat{
\includegraphics[width=1.0\columnwidth, keepaspectratio]{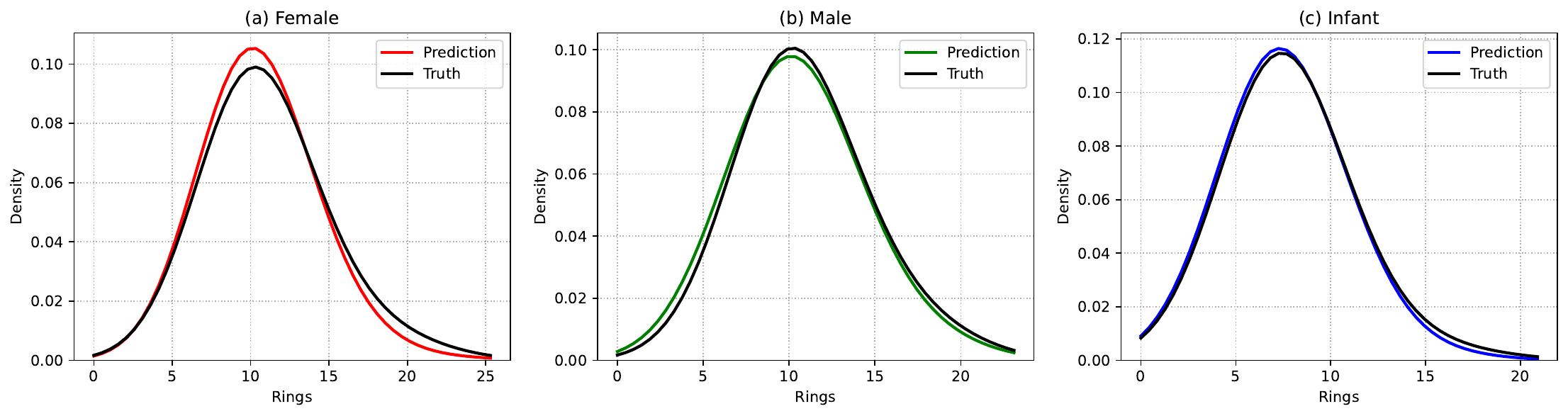}
    }\\
    \subfloat{
\includegraphics[width=1.0\columnwidth, keepaspectratio]{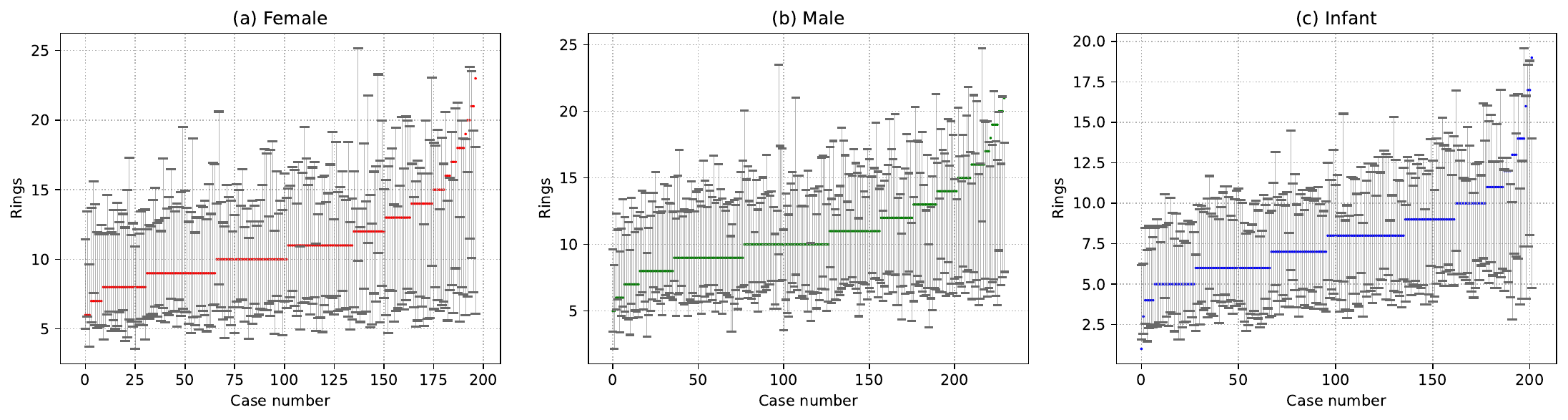}
    }
    \caption{(Top) Estimated density and actual density on the test set. (Bottom) The prediction intervals on the test set.}
    \label{abalone}
\end{figure}

\subsubsection{The Superconductivity Dataset}\label{sec:r3}
The superconducting material database, supported by the National Institute of Materials Science, a public institution based in Japan, is a pivotal resource in the field of superconductivity research. From this database comes the superconductivity dataset, which is accessible at \url{https://archive.ics.uci.edu/dataset/464/superconductivty+data}. 
This dataset serves as a valuable repository of information, encompassing data on 21263 superconductors and consisting of 81 features. 
Additionally, the dataset provides details on the chemical formula and critical temperature for all 21263 superconductors. This comprehensive dataset serves as a cornerstone resource, offering invaluable insights into the properties of superconductors. As such, it is instrumental in advancing research in materials science and condensed matter physics. For further detailed information on this dataset, one can refer to \cite{hamidieh2018data}.
In our numerical analysis, we partition this dataset into three subsets: 70\% for training, 20\% for validation, and 10\% for testing. 
In this setup, we designate the critical temperature as the response variable $Y\in\mathbb{R}$, while the other measurements form the covariate vector  $X\in\mathbb{R}^{81}$.
Figure \ref{superconductivity} presents the numerical results for the superconductivity dataset. 
Upon inspection of this figure, we observe that our method yields a precise estimation of the density function. Notably, the calculated CP value stands at 95.20\%,  illustrating the robustness of our approach in furnishing dependable prediction intervals for this examined dataset.
\begin{figure}[H]
    \centering
    \subfloat{
\includegraphics[width=0.5\columnwidth, keepaspectratio]{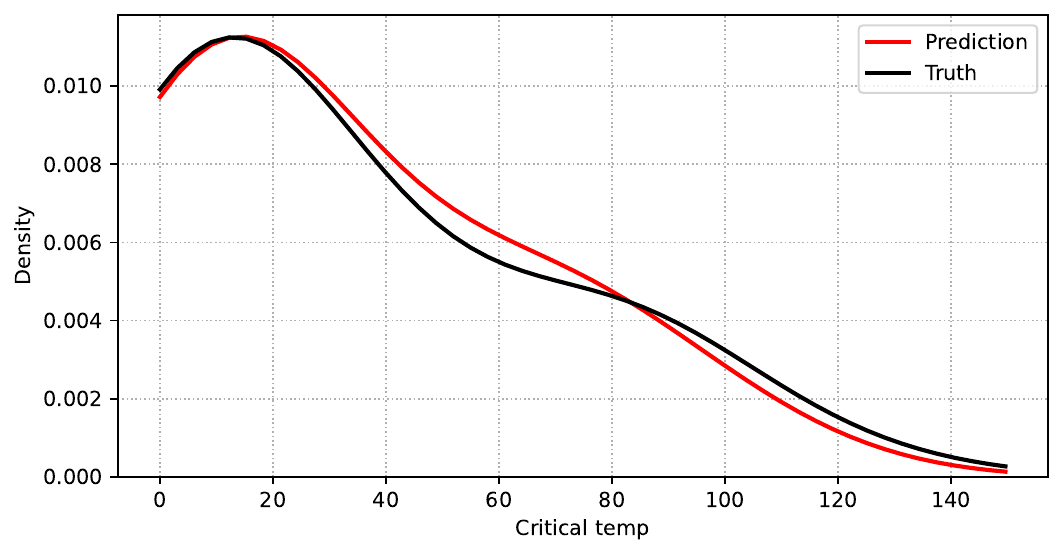}
    }
    \subfloat{
\includegraphics[width=0.5\columnwidth, keepaspectratio]{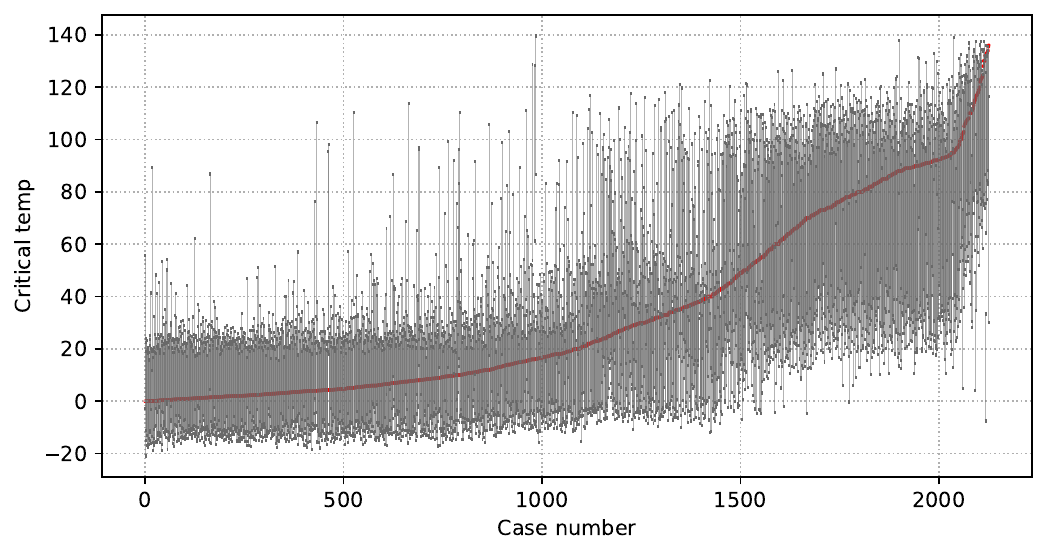}
    }
    \caption{(Left) Estimated density and actual density on the test set. (Right) The prediction intervals on the test set.}
    \label{superconductivity}
\end{figure}

\section{Conclusion}\label{sec:con}
This work   delves into statistical inference within the domain of deep nonparametric regression. This achievement is realized through the utilization of conditional diffusion models, which serve as the cornerstone of our investigative approach. Specifically, we introduce a novel conditional diffusion model and conduct a comprehensive theoretical analysis for the proposed model.
Subsequently, we establish the asymptotic normality, thereby laying a solid foundation for statistical inference in deep nonparametric regression.
Moreover, exploring alternative methods in conditional generative learning and improving convergence rates represent promising avenues for future research endeavors. 
Additionally, the study of statistical inference in deep quantile regression presents itself as another compelling direction for future investigation.

\section*{Appendix}
In this appendix, we provide detailed proofs for all lemmas and theorems presented in this paper.
In Section \ref{sec:appenda}, we begin by deriving the Lipschitz continuity of the drift term.
Section \ref{sec:appendb} contains the detailed proofs for bounding approximation and statistical errors.
In Sections \ref{sec:appendc}-\ref{sec:appende}, we
 bound $\Ebb_{\mcD,\mcT,\mcZ}[\mathrm{TV}(\ck{p}_{T_0},\wt{p}_{T_0})]$,
$\mathrm{TV}(p_{T_0},p_0)$,
and $\Ebb_{\mcD,\mcT,\mcZ}\left[\mathrm{TV}(\wt{p}_{T_0}, p_{T_0})\right]$, respectively.
Finally, in Section \ref{sec:appendf}, we present the proofs of Theorems \ref{th: end_to_end_convergence}-\ref{th: weakcon}. 

\appendix
\section{Lipschitz Continuity}\label{sec:appenda}
In this section, under Assumption \ref{ass: bounded_support}, we establish the Lipschitz continuity with respect to $\by$ on $[T_0,T]\times\Rbb^{d_\mcY}\times[0,B_{\mcX}]^{d_{\mcX}}$ and the Lipschitz continuity with respect to time $t$ on $[T_0,T]\times[-R,R]^{d_\mcY}\times[0,B_{\mcX}]^{d_{\mcX}}$ for any $R > 0$.

\begin{proof}[Proof of Lemma \ref{lem: smooth_lipschitz}.]
We rewrite $\bb(t,\by,\bx)$ as
$$
\begin{aligned}
\bb(t,\by,\bx) = \by + \frac{2\int\left(-\frac{\by-e^{-t}\by_0}{1-e^{-2t}}\right)\exp\left(-\frac{\Vert \by - e^{-t}\by_0\Vert^2}{2(1-e^{-2t})}\right)p_0(\by_0|\bx)\mrd\by_0}{\int\exp\left(-\frac{\Vert \by - e^{-t}\by_0\Vert^2}{2(1-e^{-2t})}\right)p_0(\by_0|\bx)\mrd\by_0}.
\end{aligned}
$$
Then,  we calculate $\nabla_{\by}\bb(t,\by,\bx)$ as follows:
$$
\begin{aligned}
\nabla_{\by}\bb(t,\by,\bx) = &\left(1-\frac{2}{1-e^{-2t}}\right)\mI_{d_\mcY} + \frac{
2\int\frac{(\by - e^{-t}\by_0)^{\otimes2}}{(1-e^{-2t})^2}\exp\Big(-\frac{-\Vert \by - e^{-t}\by_0\Vert^2}{2(1-e^{-2t})}\Big)p_{0}(\by_0|\bx)\mrd\by_0
}
{\int\exp\Big(-\frac{\Vert \by - e^{-t}\by_0\Vert^2}{2(1-e^{-2t})}\Big)p_{0}(\by_0|\bx)\mrd\by_0}\\
& - 2\left(\frac{\int\left(\frac{\by - e^{-t}\by_0}{1-e^{-2t}}\right)\exp\Big(-\frac{\Vert \by - e^{-t}\by_0\Vert^2}{2(1-e^{-2t})}\Big)p_{0}(\by_0|\bx)\mrd\by_0}{\int\exp\Big(-\frac{\Vert \by - e^{-t}\by_0\Vert^2}{2(1-e^{-2t})}\Big)p_0(\by_0|\bx)\mrd\by_0}
\right)^{\otimes2}.
\end{aligned}
$$
Given $\bx$ and $\ov\bY_t=\by$, the conditional density function becomes
$$
\begin{aligned}
    p(\by_0|\by,\bx):&= p(\by_0|\ov{\bY}_t=\by,\bx)\\& = \frac{p_t(\by|\by_0,\bx)p_0(\by_0|\bx)}{p_t(\by|\bx)}\\
    &=\frac{\exp\Big(-\frac{\Vert \by - e^{-t}\by_0\Vert^2}{2(1-e^{-2t})}\Big)p_{0}(\by_0|\bx)}{\int\exp\Big(-\frac{\Vert \by - e^{-t}\by_0\Vert^2}{2(1-e^{-2t})}\Big)p_{0}(\by_0|\bx)\mrd\by_0}.
\end{aligned}
$$
Then,  $\nabla_{\by}\bb(t,\by,\bx)$ can be rewritten as
\begin{equation*}
\begin{aligned}
\nabla_{\by}\bb(t,\by,\bx) &= \left(1-\frac{2}{1-e^{-2t}}\right)\mI_{d_\mcY} + \frac{2}{(1-e^{-2t})^2}\mathrm{Cov}[\by - e^{-t}\ov{\bY}_0|\by,\bx]\\
&=\left(1-\frac{2}{1-e^{-2t}}\right)\mI_{d_\mcY} + \frac{2e^{-2t}}{(1-e^{-2t})^2}\mathrm{Cov}[\ov{\bY}_0|\by,\bx].
\end{aligned}
\end{equation*}
Under Assumption \ref{ass: bounded_support},  we have
$$
\m0\preccurlyeq\mathrm{Cov}[\ov{\bY}_0|\by,\bx]\preccurlyeq d_\mcY\mI_{d_\mcY}.
$$
Thus, we have
\begin{equation*}
\begin{aligned}
\left(1-\frac{2}{1-e^{-2t}}\right)\mI_{d_\mcY}
\preccurlyeq\nabla_{\by}\bb(t,\by,\bx)&\preccurlyeq\left(\frac{2d_\mcY e^{-2t}}{(1-e^{-2t})^2} + 1-\frac{2}{1-e^{-2t}}\right)\mI_{d_\mcY}\\
&\preccurlyeq \frac{2d_\mcY}{(1-e^{-2t})^2}\mI_{d_\mcY},
\end{aligned}
\end{equation*}
which implies that
$$
\Vert \nabla_{\by}\bb(t,\by,\bx) \Vert\leq \frac{2d_\mcY}{(1-e^{-2T_0})^2}\leq \frac{2d_\mcY}{T_0^2},
$$
where we used $T_0\leq 1-e^{-2T_0}\leq 2T_0$. The proof is complete.

\end{proof}

\begin{lemma}\label{lem: lipschitz_t}
For any $R>0$, we have
$$
\mathop{\sup}_{t\in[T_0,T]}\mathop{\sup}_{\by\in[-R,R]^{d_\mcY}}\mathop{\sup}_{\bx\in[0,B_{\mathcal{X}}]^{d_{\mathcal{X}}}}
\Vert\partial_t\bb(t,\by,\bx)\Vert = \mathcal{O}\left(\frac{d_{\mcY}^{3/2}(R + 1)}{T_0^3}\right).
$$
\end{lemma}

\begin{proof}
For ease of calculation, we denote $\bs^*(t,\by,\bx) = \nabla_{\by}\log p_t(\by|\bx)$ and
$\phi_t(\by_0,\by,\bx) = \exp\left(-\frac{\Vert \by - e^{-t}\by_0
 \Vert^2}{2(1-e^{-2t})}\right)p_{0}(\by_0|\bx)$. Then
$$
\begin{aligned}
&\partial_{t}\bs^{*}(t,\by,\bx) \\
&= -\frac{\int\left(\frac{e^{-t}\by_0(1-e^{-2t})-2(\by-e^{-t}\by_0)e^{-2t}}{(1-e^{-2t})^2}\right)\phi_{t}(\by_0,\by,\bx)\mrd\by_0 + \int\left(\frac{\by - e^{-t}\by_0}{1-e^{-2t}}\right)\partial_t\phi_t(\by_0,\by,\bx)\mrd\by_0}{\int\phi_{t}(\by_0,\by,\bx)\mrd\by_0} \\
&~~~~+ \frac{\int\left(\frac{\by - e^{-t}\by_0}{1-e^{-2t}}\right)\phi_{t}(\by_0,\by,\bx)\mrd\by_0}{\int\phi_{t}(\by_0,\by,\bx)\mrd\by_0}\cdot\frac{\int\partial_t\phi_{t}(\by_0,\by,\bx)\mrd\by_0}{\int\phi_{t}(\by_0,\by,\bx)\mrd\by_0}.
\end{aligned}
$$
Now we calculate $\partial_t\phi_t(\by_0,\by,\bx)$. Since
$$
\begin{aligned}
    \partial_t\left(-\frac{\Vert \by - e^{-t}\by_0
 \Vert^2}{2(1-e^{-2t})}\right) &=- \frac{(\by-e^{-t}\by_0)^{\top}\by_0e^{-t}(1-e^{-2t}) - \Vert \by - e^{-t}\by_0\Vert^2e^{-2t}}{(1-e^{-2t})^2}\\
&=\frac{e^{-2t}(\Vert \by_0\Vert^2 + \Vert \by \Vert^2) - \by^{\top}\by_0e^{-t}(1 + e^{-2t})}{(1-e^{-2t})^2},
\end{aligned}
$$
we have
$$
\partial_t\phi_t(\by_0,\by,\bx) = \phi_t(\by_0,\by,\bx)\left[\frac{e^{-2t}(\Vert \by_0\Vert^2 + \Vert \by \Vert^2) - \by^{\top}\by_0e^{-t}(1 + e^{-2t})}{(1-e^{-2t})^2}\right].
$$
Thus, $\partial_t\bs^{*}(t,\by,\bx)$ can be expressed as follows:
$$
\begin{aligned}
&\partial_t\bs^{*}(t,\by,\bx)\\
&= -\mathbb{E}\left[\frac{e^{-t}(1 + e^{-2t})\ov{\bY}_0 - 2e^{-2t}\by}{(1-e^{-2t})^2}\Big|\by,\bx\right] \\
&~~~~+ \mathbb{E}\left(\left[\frac{\by - e^{-t}\ov{\bY}_0}{1-e^{-2t}}\right]\left[\frac{e^{-2t}(\Vert \ov{\bY}_0\Vert^2 + \Vert \by \Vert^2) - \by^{\top}\ov{\bY}_0e^{-t}(1 + e^{-2t})}{(1-e^{-2t})^2}\right]\Big|\by,\bx\right)\\
&~~~~-\mathbb{E}\left[\frac{\by - e^{-t}\ov{\bY}_0}{1-e^{-2t}}\Big|\by,\bx\right]\mathbb{E}\left[\frac{e^{-2t}(\Vert \ov{\bY}_0\Vert^2 + \Vert \by \Vert^2) - \by^{\top}\ov{\bY}_0e^{-t}(1 + e^{-2t})}{(1-e^{-2t})^2}\Big|\by,\bx\right]\\
&=-\mathbb{E}\left[\frac{e^{-t}(1 + e^{-2t})\ov{\bY}_0 - 2e^{-2t}\by}{(1-e^{-2t})^2}\Big|\by,\bx\right] + \frac{e^{-2t}(1 + e^{-2t})}{(1-e^{-2t})^3}\mathrm{Cov}[\ov{\bY}_0|\by,\bx]\by \\
&~~~~- \frac{e^{-3t}}{(1-e^{-2t})^3}\left(\mathbb{E}[\Vert \ov{\bY}_0\Vert^2\ov{\bY}_0|\by,\bx] - \mathbb{E}[\Vert \ov{\bY}_0\Vert^2|\by,\bx]\cdot\mathbb{E}[\ov{\bY}_0|\by,\bx]\right).
\end{aligned}
$$
By Assumption \ref{ass: bounded_support}, the operator norm of $\mathrm{Cov}[\ov{\bY}_0|\by,\bx]$ satisfies
$$\Vert\mathrm{Cov}[\ov{\bY}_0|\by,\bx]\Vert_{\mathrm{op}}\leq d_\mcY.$$
By Cauchy-Schwarz inequality, we have
$$
\left\Vert\mathbb{E}[\Vert \ov{\bY}_0\Vert^2\ov{\bY}_0|\by,\bx]\right\Vert \leq \Ebb[\Vert \ov{\bY}_0 \Vert^6|\by,\bx]^{1/2} \leq d_{\mcY}^{3/2},
$$
and
$$
\left\Vert
\mathbb{E}[\ov{\bY}_0|\by,\bx]
\right\Vert \leq \Ebb[\Vert\ov{\bY}_0\Vert^2|\by,\bx]^{1/2}\leq \sqrt{d_\mcY}.
$$
Therefore, for any $t\in[T_0,T]$, $\by\in[-R,R]^{d_\mcY}$ and $\bx\in[0,B_{\mathcal{X}}]^{d_{\mathcal{X}}}$,  we obtain
$$
\begin{aligned}
    \Vert\partial_t\bs^{*}(t,\by,\bx)\Vert
    &\leq \frac{2\sqrt{d_\mcY} + 2\sqrt{d_\mcY}R}{(1-e^{-2T_0})^2} + \frac{2d_{\mcY}^{3/2}R}{(1-e^{-2T_0})^3} + \frac{2d_{\mcY}^{3/2}}{(1-e^{-2T_0})^3}\\
    &\leq\frac{2\sqrt{d_\mcY}(R + 1)}{T_0^2} + \frac{2d_{\mcY}^{3/2}(R + 1)}{T_0^3},
\end{aligned}
$$
which implies 
$$
\mathop{\sup}_{t\in[T_0,T]}\mathop{\sup}_{\by\in[-R,R]^{d_\mcY}}\mathop{\sup}_{\bx\in[0,B_{\mathcal{X}}]^{d_{\mathcal{X}}}}
\Vert\partial_t\bs^{*}(t,\by,\bx)\Vert = \mathcal{O}\left(\frac{d_{\mcY}^{3/2}(R + 1)}{T_0^3}\right).
$$
The above inequality also implies 
$$
\mathop{\sup}_{t\in[T_0,T]}\mathop{\sup}_{\by\in[-R,R]^{d_\mcY}}\mathop{\sup}_{\bx\in[0,B_{\mathcal{X}}]^{d_{\mathcal{X}}}}
\Vert\partial_t\bb(t,\by,\bx)\Vert = \mathcal{O}\left(\frac{d_{\mcY}^{3/2}(R + 1)}{T_0^3}\right).
$$
The proof is complete.
\end{proof}

\section{Approximation and Statistical Errors}\label{sec:appendb}
In this section, we first give the upper bounds of approximation  and statistical errors. Then we prove Theorem \ref{th: drift_estimation}. In subsection \ref{subsec: approx_error}, we prove Lemma \ref{lem: approx_err}. In subsection \ref{subsec: stat_error}, we prove Lemma \ref{lem: stat_error}. In subsection \ref{subsec: error_bound_for_drift}, we prove Theorem \ref{th: drift_estimation}. We start by proving Lemma \ref{lem: connection}.

\begin{proof}[Proof of Lemma \ref{lem: connection}.] For any $\bs:[T_0, T]\times\mathbb{R}^{d_\mcY}\times[0,B_{\mathcal{X}}]^{d_{\mathcal{X}}}\rightarrow\mathbb{R}^{d_\mcY}$, we have
\begin{align*}
 &~\mathbb{E}_{\ov\bY_t,\ov\bY_0,\bx}\left\Vert \bs(t,\overline{\bY}_t,\bx) - \overline{\bY}_t - 2\nabla_{\by}\log{p}_t(\overline{\bY}_t|\overline{\bY}_0,\bx)\right\Vert^2
 \\
&~~~~~~~~~ - \mathbb{E}_{\ov\bY_t,\ov\bY_0,\bx}\left\Vert \bb(t,\overline{\bY}_t,\bx)- \overline{\bY}_t - 2\nabla_{\by}\log{p}_t(\overline{\bY}_t|\overline{\bY}_0,\bx)\right\Vert^2\\
=&~\mathbb{E}_{\ov\bY_t,\bx}\Vert \bs(t,\overline{\bY}_t,\bx) \Vert^2 - \mathbb{E}_{\ov\bY_t,\bx}\Vert \bb(t,\overline{\bY}_t,\bx) \Vert^2\\
&~~~~- 2\mathbb{E}_{\ov\bY_t,\ov\bY_0,\bx}\left\langle \bs(t,\overline{\bY}_t,\bx)-\bb(t,\overline{\bY}_t,\bx), \overline{\bY}_t + 2\nabla_{\by}\log p_t(\overline{\bY}_t|\overline{\bY}_0,\bx)\right\rangle.
\end{align*}
Taking the conditional expectation for $\overline{\bY}_t$, $\bx$, we have
$$
\begin{aligned}
&~\mathbb{E}_{\ov\bY_t,\ov\bY_0,\bx}\left\langle \bs(t,\overline{\bY}_t,\bx)-\bb(t,\overline{\bY}_t,\bx), \overline{\bY}_t + 2\nabla_{\by}\log p_t(\overline{\bY}_t|\overline{\bY}_0,\bx)\right\rangle\\
=&~\mathbb{E}_{\ov\bY_t,\ov\bY_0,\bx}\left\langle \bs(t,\overline{\bY}_t,\bx)-\bb(t,\overline{\bY}_t,\bx), \overline{\bY}_t-\frac{2(\overline{\bY}_t - e^{-t}\overline{\bY}_0)}{1-e^{-2t}}\right\rangle\\
=&~\mathbb{E}_{\ov\bY_t,\bx}\left\langle \bs(t,\overline{\bY}_t,\bx)-\bb(t,\overline{\bY}_t,\bx), \overline{\bY}_t + 2\mathbb{E}_{\ov\bY_0}\left[-\frac{\overline{\bY}_t - e^{-t}\overline{\bY}_0}{1-e^{-2t}}\Big|\overline{\bY}_t,\by\right]\right\rangle\\
=&~\mathbb{E}_{\ov\bY_t,\bx}\langle \bs(t,\overline{\bY}_t,\bx)-\bb(t,\overline{\bY}_t,\bx),\bb(t,\overline{\bY}_t,\bx)\rangle.
\end{aligned}
$$
Therefore, we obtain
$$
\begin{aligned}
&~\mathbb{E}_{\ov\bY_t,\ov\bY_0,\bx}\left\Vert \bs(t,\overline{\bY}_t,\bx) - \overline{\bY}_t - 2\nabla_{\by}\log p_t(\overline{\bY}_t|\overline{\bY}_0,\bx)\right\Vert^2\\
& ~~~~~~~~  - \mathbb{E}_{\ov\bY_t,\ov\bY_0,\bx}\left\Vert \bb(t,\overline{\bY}_t,\bx) - \overline{\bY}_t - 2\nabla_{\by}\log p_t(\overline{\bY}_t| \overline{\bY}_0,\bx)\right\Vert^2 \\
=&~\mathbb{E}_{\ov\bY_t,\bx}\Vert \bs(t,\overline{\bY}_t,\bx) - \bb(t,\overline{\bY}_t,\bx)\Vert^2.
\end{aligned}
$$
This implies that
$$
\mathcal{L}(\bs) - \mathcal{L}(\bb) = \frac{1}{T - T_0}\int_{T_0}^{T}\mathbb{E}_{\ov\bY_t,\bx}\Vert \bs(t,\overline{\bY}_t,\bx) - \bb(t,\overline{\bY}_t,\bx)\Vert^2\mrd t.
$$
The proof is complete.
\end{proof}

\subsection{Approximation Error}\label{subsec: approx_error}
\textbf{Approximation on Compact Set.} The goal is to find a neural network $\bs\in\mathrm{NN}$ to approximate $\bb$. A major difficulty in approximating $\bb$ is that the input space $[T_0, T]\times\mathbb{R}^{d_\mcY}\times[0,B_{\mathcal{X}}]^{d_{\mathcal{X}}}$ is unbounded. In order to address this difficulty, we 
partition $\mathbb{R}^{d_\mcY}$ in to $\mathcal{K}$ and $\mathcal{K}^c$.

First, we consider the approximation on $[T_0,T]\times\mathcal{K}\times[0,B_{\mathcal{X}}]^{d_{\mathcal{X}}}$. Let $\mathcal{K}=\{\by\in\Rbb^{d_\mcY}|\Vert{\by}\Vert_{\infty}\leq R\}$ to be a $d_\mcY$-dimensional hypercube with edge length $2R>0$, where $R$ will be determined later. On $[T_0, T]\times\mathcal{K}\times[0,B_{\mathcal{X}}]^{d_{\mathcal{X}}}$, we approximate $k$-coordinate maps $b_{k}(t,\by,\bx)$ separately, where $\bb(t,\by,\bx) = [b_{1}(t,\by,\bx),\cdots,b_{d_\mcY}(t,\by,\bx)]^{\top}$. Then, we can obtain an approximation of $\bb(t,\by,\bx)$ by concatenation. We rescale the input by $t^{\prime} = \frac{t-T_0}{T - T_0}$, $\by^{\prime} = \frac{\by + R\mathbf{1}}{2R}$ and $\bx^{\prime} = \frac{\bx}{B_{\mathcal{X}}}$, so that the transformed space is $[0,1]\times[0,1]^{d_\mcY}\times[0,1]^{d_{\mathcal{X}}}$. Such a transformation can be exactly implemented by a single ReLU layer.

By Lemma \ref{lem: smooth_lipschitz}, $\bb(t,\by,\bx)$ is $\xi$-Lipschitz in $\by$. We define the rescaled function on the transformed input space as $\bB(t^{\prime}, \by^{\prime}, \bx^{\prime}):=\bb((T - T_0)t^{\prime} + T_0, 2R\by^{\prime}-R\mathbf{1},B_{\mathcal{X}}\bx^{\prime})$, so that $\widetilde{\bB}$ is $2R\xi$-Lipschitz in $\by^{\prime}$. Furthermore, by Assumption \ref{ass: smoothness_y}, $\bB$ is $\beta B_{\mathcal{X}}$-Lipschitz in $\bx^{\prime}$. We denote
$$
\tau(R):=\mathop{\sup}_{t\in[T_0,T]}\mathop{\sup}_{\by\in[-R,R]^{d_\mcY}}\mathop{\sup}_{\bx\in[0,B_{\mathcal{X}}]^{d_{\mathcal{X}}}}\Vert\partial_{t}\bb(t,\by,\bx)\Vert.
$$
By Lemma \ref{lem: lipschitz_t}, $\tau(R) = \mathcal{O}\Big(\frac{d_{\mcY}^{3/2}(1 + R)}{T_0^3}\Big)$, then $\bB$ is $(T-T_0)\tau(R)$-Lipschitz in $t^{\prime}$. Now the goal becomes approximating $\bB$ on $[0,1]\times[0,1]^{d_\mcY}\times[0,1]^{d_{\mathcal{X}}}$.

Second, we partition the time interval $[0, 1]$ into non-overlapping sub-intervals of length $e_1$. We also partition $[0,1]^{d_\mcY}$ for $\by^{\prime}$ into non-overlapping hypercubes with equal edge length and $[0,1]^{d_{\mathcal{X}}}$ for $\bx^{\prime}$ into equal edge length $e_3$. $e_1$, $e_2$ and $e_3$ will be chosen depending on the desired approximation error. We denote $N_1=\lceil\frac{1}{e_1}\rceil$, $N_2=\lceil\frac{1}{e_2}\rceil$ and $N_3=\lceil\frac{1}{e_3}\rceil$.

Let $\boldsymbol{u}=[u_1,u_2,\cdots,u_{d_\mcY}]^{\top}\in[N_2]^{d_\mcY}$, $\boldsymbol{v}=[v_1,v_2,\cdots,v_{d_\mcX}]^{\top}\in[N_3]^{d_\mcX}$ be multi-indexes. We define $\overline{B}_i(t^{\prime}, \by^{\prime}, \bx^{\prime})$ as
$$
\overline{B}_{i}(t^{\prime}, \by^{\prime}, \bx^{\prime}):=\sum_{j\in[N_1],\boldsymbol{u}\in[N_2]^{d_\mcY},\boldsymbol{v}\in[N_3]^{d_\mcX}}B_{i}\left(\frac{j}{N_1},\frac{\boldsymbol{u}}{N_2},\frac{\boldsymbol{v}}{N_3}\right)\Psi_{j,\boldsymbol{u},\boldsymbol{v}}(t^{\prime}, \by^{\prime}, \bx^{\prime}), i=1,2,\cdots,d_\mcY,
$$
where $\Psi_{j,\boldsymbol{u},\boldsymbol{v}}(t^{\prime}, \by^{\prime}, \bx^{\prime})$ is a partition of unity function which satisfies that 
$$
\sum_{j\in[N_1],\boldsymbol{u}\in[N_2]^{d_\mcY},\boldsymbol{v}\in[N_3]^{d_\mcX}}\Psi_{j,\boldsymbol{u},\boldsymbol{v}}(t^{\prime}, \by^{\prime}, \bx^{\prime})\equiv{1}, \text{ on } [0,1]\times[0,1]^{d_\mcY}\times[0,1]^{d_{\mathcal{X}}}.
$$
We choose $\Psi_{j,\boldsymbol{u},\boldsymbol{v}}$ as a product of coordinate-wise trapezoid functions:
$$
\Psi_{j,\boldsymbol{u},\boldsymbol{v}}:=\psi\left(3N_1\left(t^{\prime} - \frac{j}{N_1}\right)\right)\prod_{i=1}^{d_\mcY}\psi\left(3N_2\left(y_{i}^{\prime} - \frac{u_i}{N_2}\right)\right)\prod_{i=1}^{d_{\mathcal{X}}}\psi\left(3N_3\left(x_{i}^{\prime} - \frac{v_i}{N_3}\right)\right),
$$
where $\psi$ is a trapezoid function,
$$
\psi(x):=
\begin{cases}
    1, &|x| < 1,\\
    2 - |x|, &|x|\in[1,2],\\
    0, & |x| > 2.
\end{cases}
$$
We claim that
\begin{itemize}
    \item $\overline{B}_{i}(t^{\prime}, \by^{\prime}, \bx^{\prime})$ is an approximation of $B_{i}(t^{\prime}, \by^{\prime}, \bx^{\prime})$;
    \item $\overline{B}_{i}(t^{\prime}, \by^{\prime}, \bx^{\prime})$ can be implemented by a ReLU neural network $S_i(t^{\prime}, \by^{\prime}, \bx^{\prime})$ with small error.
\end{itemize}
Both claims are verified in \cite[Lemma 10]{chen2022distribution}, where we only need to substitute the
Lipschitz constant $2R\xi$, $\beta B_{\mathcal{X}}$ and $(T - T_0)\tau(R)$ into the error analysis. By concatenating $S_{i}(t^{\prime}, \by^{\prime}, \bx^{\prime}), i=1,2,\cdots,d_\mcY$ together, we construct $\bS=[S_1,S_2,\cdots,S_{d_\mcY}]^{\top}$. Given $\epsilon>0$, we can achieve
$$
\mathop{\sup}_{t^{\prime},\by^{\prime},\bx^{\prime}\in[0,1]\times[0,1]^{d_\mcY}\times[0,1]^{d_{\mathcal{X}}}}\Vert \bS(t^{\prime}, \by^{\prime}, \bx^{\prime}) - \bB(t^{\prime}, \by^{\prime}, \bx^{\prime}) \Vert_{\infty}\leq\epsilon,
$$
where the neural network configuration is 
$$
L = \mathcal{O}\Big(\log\frac{1}{\epsilon} + d_\mcX + d_{\mathcal{Y}}\Big), M = \mathcal{O}\Big(d_\mcY(T-T_0)\tau(R)(R\xi)^{d_\mcY}(\beta B_{\mathcal{X}})^{d_{\mathcal{X}}}\epsilon^{-(d_\mcX+d_{\mathcal{Y}}+1)}\Big),
$$
$$
J = \mathcal{O}\left(d_\mcY(T-T_0)\tau(R)(R\xi)^{d_\mcY}(\beta B_{\mathcal{X}})^{d_{\mathcal{X}}}\epsilon^{-(d_\mcX+d_{\mathcal{Y}}+1)}\Big(\log\frac{1}{\epsilon} + d_\mcX+d_{\mathcal{Y}}\Big)\right),
$$
$$
K = \mathcal{O}\left(\frac{\sqrt{d_\mcY}R}{T_0}\right), \kappa = \mathcal{O}(\max\{1,R\xi,(T - T_0)\tau(R),\beta B_{\mathcal{X}}\}).
$$
Here, we set $e_1 = \mathcal{O}\Big(\frac{\epsilon}{(T - T_0)\tau(R)}\Big)$, $e_2 = \mathcal{O}\Big(\frac{\epsilon}{R\xi}\Big)$ and $e_3=\mathcal{O}\Big(\frac{\epsilon}{\beta B_{\mathcal{X}}}\Big)$.
The output range $K$ is computed by
$$
\begin{aligned}
K &= \mathop{\sup}_{t,\by,\bx\in[T_0,T]\times[-R,R]^{d_\mcY}\times[0,B_{\mathcal{X}}]^{d_{\mathcal{X}}}}\Vert 
\bb(t,\by,\bx)\Vert\\
&\leq\mathop{\sup}_{t,\by,\bx\in[T_0,T]\times[-R,R]^{d_\mcY}\times[0,B_{\mathcal{X}}]^{d_{\mathcal{X}}}}\left(\frac{\Vert \by \Vert}{1-e^{-t}} + \frac{2e^{-t}\sqrt{d_\mcY}}{1 - e^{-2t}}\right)\\
&\leq \frac{\sqrt{d_\mcY}(R + 1)}{1-e^{-T_0}} =\mathcal{O}\left(\frac{\sqrt{d_\mcY}R}{T_0}\right).
\end{aligned}
$$
Further, the neural network $\bS(t^{\prime},\by^{\prime},\bx^{\prime})$ is Lipschitz continuous in $\bx^{\prime}$, i.e.,
for any $\bx^{\prime}_1$, $\bx^{\prime}_2\in[0,1]^{d_{\mathcal{X}}}$, $t^{\prime}\in[0,1]$ and $\by^{\prime}\in[0,1]^{d_\mcY}$, it holds
$$
\Vert \bS(t^{\prime},\by^{\prime},\bx^{\prime}_1) - \bS(t^{\prime},\by^{\prime},\bx^{\prime}_2) \Vert_{\infty}\leq 10d\beta B_{\mathcal{X}}\Vert \bx^{\prime}_1 - \bx^{\prime}_2\Vert.
$$
$\bS(t^{\prime},\by^{\prime},\bx^{\prime})$ is also Lipschitz continuous in $\by^{\prime}$, i.e., for any $t^{\prime}\in[0,1]$, $\bx^{\prime}\in[0,1]^{d_{\mathcal{X}}}$, and $\by^{\prime}_1$, $\by^{\prime}_2\in[0,1]^{d_\mcY}$, it holds
$$
\Vert \bS(t^{\prime},\by^{\prime}_1, \bx^{\prime}) - \bS(t^{\prime},\by^{\prime}_2, \bx^{\prime}) \Vert_{\infty}\leq 20dR\xi\Vert \by^{\prime}_1 - \by^{\prime}_2\Vert.
$$
And for any $t^{\prime}_{1},t^{\prime}_{2}\in[0,1]$, $\by^{\prime}\in[0,1]^{d_\mcY}$, $\bx^{\prime}\in[0,1]^{d_{\mathcal{X}}}$, it holds
$$
\Vert \bS(t^{\prime}_{1},\by^{\prime},\bx^{\prime}) - \bS(t^{\prime}_{2},\by^{\prime},\bx^{\prime}) \Vert_{\infty}\leq 10(T - T_0)\tau(R)|t^{\prime}_{1} - t^{\prime}_{2}|.
$$
Combining with the single input transformation layer, we denote 
$$
\bs(t,\by,\bx) = \bS\left(\frac{t-T_0}{T - T_0}, \frac{\by + R\mathbf{1}}{2R}, \frac{\bx}{B_{\mathcal{X}}}\right),
$$
then for any 
$t\in[T_0,T]$, $\bx_1,\bx_2\in[0,B_{\mathcal{X}}]^{d_{\mathcal{X}}}$, $\by\in\mathcal{K}$, it holds
$$
\Vert \bs(t,\by,\bx_1) - \bs(t,\by,\bx_2) \Vert_{\infty}\leq 10d\beta\Vert \bx_1 - \bx_2 \Vert.
$$
Moreover, for any $t\in[T_0,T]$, $\bx\in[0,B_{\mathcal{X}}]^{d_{\mathcal{X}}}$, $\by_1, \by_2\in\mathcal{K}$, it holds
$$
\Vert \bs(t,\by_1,\bx) - \bs(t, \by_2, \bx) \Vert_{\infty}\leq 10d\xi\Vert \by_1- \by_2 \Vert.
$$
And for any $t_1,t_2\in[T_0,T]$, $\by\in\mathcal{K}$, $\bx\in[0,B_{\mathcal{X}}]^{d_{\mathcal{X}}}$, it holds
$$
\Vert \bs(t_1,\by,\bx) - \bs(t_2,\by,\bx)\Vert_{\infty}\leq 10\tau(R)|t_1-t_2|.
$$
\\
\noindent\textbf{Global Lipschitz.} In order to ensure global Lipschitz continuity of $\bs$ with respect to $\by$, we introduce the following lemma.

\begin{lemma}\label{lem: clipping}
For $\mathcal{K}=\{\by\in\mathbb{R}^{d_\mcY}|\Vert 
{\by}\Vert_{\infty}\leq R\}$, there exists an $1$-Lipshitz map $\mathcal{T}_{\mathcal{K}}:\mathbb{R}^{d_\mcY}\rightarrow\mathcal{K}$, which satisfies $\mathcal{T}_{\mathcal{K}}(\by)=\by$, for any $\by\in\mathcal{K}$. And $\mathcal{T}_{\mathcal{K}}$ can be expressed as an 2-layer ReLU network with width order $\mathcal{O}(d_\mcY)$.
\end{lemma}
\begin{proof}
We consider the following function
$$
f_{R}(y):=
\begin{cases}
    R, & y > R\\
    y, & y\in[-R, R]\\
    -R, & y < -R
\end{cases}
$$
and define the map $\mathcal{T}_{\mathcal{K}}(\by)=(f_{R}(y_1),f_{R}(y_2),\cdots,f_{R}(y_{d_\mcY}))^{\top}$. Then, for any $\by\in\mathcal{K}$, $\mathcal{T}_{\mathcal{K}}(\by) = \by$, and for any $\by\in\mathcal{K}^{c}$, $\mathcal{T}_{\mathcal{K}}(\by)\in\partial\mathcal{K}$. By a simple calculation, for any $\by,\boldsymbol{z}\in\Rbb^{d_\mcY}$, we have
$$
\begin{aligned}
\Vert\mathcal{T}_{\mathcal{K}}(\by) - \mathcal{T}_{\mathcal{K}}(\boldsymbol{z})\Vert_{\infty} &= \mathop{\max}_{i=1,2,\cdots,d_\mcY}|f_{R}(y_{i}) - f_{R}(z_{i})|\\
&\leq\mathop{\max}_{i=1,2,\cdots,d_\mcY}|y_i - z_i| = \Vert 
\by - \boldsymbol{z} \Vert_{\infty}\leq\Vert 
\by - \boldsymbol{z} \Vert.
\end{aligned}
$$
And it is easy to check
$$
f_{R}(y) = \mathrm{ReLU}(y) - \mathrm{ReLU}(-y) + \mathrm{ReLU}(-y - R) - \mathrm{ReLU}(y - R).
$$
The proof is complete.
\end{proof}

We now consider the neural network $\bs(t,\mathcal{T}_{\mathcal{K}}(\by),\bx)$. It preserves the approximation capability of $\bs(t,\by,\bx)$, i.e.,
$$
\begin{aligned}
&\sup_{t,\by,\bx\in[T_0,T]\times\mathcal{K}\times[0,B_{\mathcal{X}}]^{d_{\mathcal{X}}}}\Vert \bs(t,\mathcal{T}_{\mathcal{K}}(\by),\bx) - \bb(t,\by,\bx)\Vert_{\infty} \\
=&\sup_{t,\by,\bx\in[T_0,T]\times\mathcal{K}\times[0,B_{\mathcal{X}}]^{d_{\mathcal{X}}}}\Vert \bs(t,\by,\bx) - \bb(t,\by,\bx)\Vert_{\infty}\leq\epsilon. 
\end{aligned}
$$
Moreover, it holds the upper bound, i.e.,
$$
\sup_{t,\by,\bx\in[T_0,T]\times\mathbb{R}^{d_\mcY}\times[0,B_{\mathcal{X}}]^{d_{\mathcal{X}}}}\Vert \bs(t,\mathcal{T}_{\mathcal{K}}(\by),\bx)\Vert = \sup_{t,\by,\bx\in[T_0,T]\times\mathcal{K}\times[0,B_{\mathcal{X}}]^{d_{\mathcal{X}}}}\Vert \bs(t,\by,\bx)\Vert \leq K.
$$
For any $\bx\in[0,B_{\mathcal{X}}]^{d_\mcX}$, $\by_1,\by_2\in\mathbb{R}^{d_\mcY}$ and $t\in[T_0,T]$, it yields that
$$
\begin{aligned}
\Vert \bs(t,\mathcal{T}_{\mathcal{K}}(\by_1),\bx) - \bs(t,\mathcal{T}_{\mathcal{K}}(\by_2),\by)\Vert_{\infty}
&\leq 10d\xi\Vert \mathcal{T}_{\mathcal{K}}(\by_1) - \mathcal{T}_{\mathcal{K}}(\by_2)\Vert\\
&\leq 10d\xi\Vert \by_1 - \by_2\Vert.
\end{aligned}
$$
Therefore, $\bs(t,\mathcal{T}_{\mathcal{K}}(\by),\bx)$ is global Lipschitz continuous with respect to $\by$. For notational convenience, we still use $\bs(t,\by,\bx)$ to denote $\bs(t,\mathcal{T}_{\mathcal{K}}(\by),\bx)$.
\\\\
\noindent\textbf{$L^2$ Approximation.} For any $t\in[T_0,T]$, $\bx\in[0,B_{\mcX}]^{d_{\mcX}}$, the $L^2$ approximation error can be decomposed into two terms,
\begin{equation*}\label{eq: L2_error}
\begin{aligned}
\mathbb{E}_{\ov\bY_t}\Vert \bs(t,\overline{\bY}_t,\bx) - \bb(t,\overline{\bY}_t,\bx)\Vert^2 &= \mathbb{E}_{\ov\bY_t}\left(\Vert \bs(t,\overline{\bY}_t,\bx) - \bb(t,\overline{\bY}_t,\bx)\Vert^2\mI_{\{\Vert \overline{\bY}_t \Vert_{\infty}\leq R\}}\right)\\
&~~~~~+ \mathbb{E}_{\ov\bY_t}\left(\Vert \bs(t,\overline{\bY}_t,\bx) - \bb(t,\overline{\bY}_t,\bx)\Vert^2\mI_{\{\Vert \overline{\bY}_t \Vert_{\infty} > R\}}\right).
\end{aligned}
\end{equation*}
The first term satisfies
$$
\begin{aligned}
&~~~~\mathbb{E}_{\ov\bY_t}\left(\Vert \bs(t,\overline{\bY}_t,\bx) - \bb(t,\overline{\bY}_t,\bx)\Vert^2\mI_{\{\Vert \overline{\bY}_t \Vert_{\infty}\leq R\}}\right)\\
&\leq d_\mcY\sup_{t,\by,\bx\in[T_0,T]\times[-R,R]^{d_\mcY}\times[0,B_{\mathcal{X}}]^{d_{\mathcal{X}}}}\Vert \bs(t,\by,\bx) - \bb(t,\by,\bx)\Vert^2_{\infty}\\
&\leq d_\mcY\epsilon^2.
\end{aligned}
$$
The second term satisfies
$$
\begin{aligned}
&\mathbb{E}_{\ov\bY_t}\left(\Vert \bs(t,\overline{\bY}_t,\bx) - \bb(t,\overline{\bY}_t,\bx)\Vert^2\mI_{\{\Vert \overline{\bY}_t \Vert_{\infty} > R\}}\right)\\
&\leq 2\wh{K}^2\mathbb{P}_{\ov\bY_t}(\Vert 
 \overline{\bY}_t\Vert_{\infty} > R) + 2\mathbb{E}_{\ov\bY_t}\left[\Vert \bb(t,\overline{\bY}_t,\bx) \Vert^2\mI_{\{\Vert \overline{\bY}_t\Vert_{\infty} > R\}}\right]\\
&\leq 2\wh{K}^2\mathbb{P}_{\ov\bY_t}(\Vert 
 \overline{\bY}_t\Vert_{\infty} > R) + 2\left(\mathbb{E}_{\ov\bY_t}\Vert \bb(t,\overline{\bY}_t,\bx)\Vert^4\right)^{\frac{1}{2}}\mathbb{P}_{\ov\bY_t}(\Vert 
 \overline{\bY}_t\Vert_{\infty} > R)^{\frac{1}{2}},
\end{aligned}
$$
where $\wh{K} = \sup_{ \by\in\mathbb{R}^{d_\mcY}, \bx\in[0,B_{\mathcal{X}}]^{d_{\mathcal{X}}}}\Vert \bs(t,\by,\bx)\Vert\leq  \frac{\sqrt{d_\mcY}(R + 1)}{1-e^{-t}}$. 
\\\\
Now we bound $\mathbb{E}_{\ov\bY_t}\Vert \bb(t,\overline{\bY}_t,\bx)\Vert^4$ and $\mathbb{P}_{\ov\bY_t}(\Vert 
 \overline{\bY}_t\Vert_{\infty} > R)$ separately. By Jensen's inequality, 
$$
\begin{aligned}
\mathbb{E}_{\ov\bY_t}\Vert \bb(t,\overline{\bY}_t,\bx)\Vert^4
&\leq\mathbb{E}_{\ov\bY_t}\left(\frac{\Vert\overline{\bY}_t\Vert}{1-e^{-t}} + \frac{2e^{-t}\mathbb{E}_{\ov\bY_0}[\Vert\ov{\bY}_0\Vert|\ov\bY_t]}{1-e^{-2t}}\right)^4\\
&\leq\mathbb{E}_{\ov\bY_t}\left(\frac{\Vert\overline{\bY}_t\Vert + \mathbb{E}_{\ov\bY_0}[\Vert\ov\bY_0\Vert|\ov\bY_t]}{1-e^{-t}}\right)^4\\
&\leq\frac{8\mathbb{E}_{\ov\bY_t}\Vert \overline{\bY}_t\Vert^4 + 8\mathbb{E}_{\ov\bY_0}\Vert\ov\bY_0\Vert^4}{(1-e^{-t})^4}.
\end{aligned}
$$
The fourth moment $\mathbb{E}_{\ov\bY_t}\Vert \overline{\bY}_t\Vert^4$ can be bounded by
$$
\begin{aligned}
\mathbb{E}_{\ov\bY_t}\Vert \overline{\bY}_t\Vert^4
&\leq\mathbb{E}_{\ov\bY_0}\left[\mathbb{E}\left(e^{-t}\Vert \overline{\bY}_0 \Vert + \sqrt{1-e^{-2t}}\Vert \bZ \Vert\right)^4\Big|\overline{\bY}_0\right]\\
&\leq 8e^{-4t}\mathbb{E}_{\ov\bY_0}\Vert\ov{\bY}_0\Vert^4 + 8(1-e^{-2t})^2\mathbb{E}_{\bZ}\Vert \bZ \Vert^4\\
&\leq 8d_\mcY^2 + 8d_\mcY(d_\mcY + 2)\\
&=16d_\mcY(d_\mcY+ 1).
\end{aligned}
$$
Therefore,
$$
\mathbb{E}_{\ov\bY_t}\Vert \bb(t,\overline{\bY}_t,\bx)\Vert^4\leq\frac{136d_\mcY^2+128d_\mcY}{(1-e^{-t})^4}\lesssim \frac{d_\mcY^2}{(1-e^{-t})^4}.
$$
The tail probability can be bounded as follows:
$$
\begin{aligned}
\mathbb{P}_{\ov\bY_t}(\Vert \overline{\bY}_t\Vert_{\infty} > R)
&= \mathbb{E}_{\ov\bY_0}\left[\mathbb{P}_{\bZ}\left(\Vert e^{-t}\ov{\bY}_0 + \sqrt{1-e^{-2t}}\bZ\Vert_{\infty} > R, \Big|\ov{\bY}_0\right)\right]
\\
&\leq\sum_{i=1}^{d_\mcY}\mathbb{E}_{\ov\bY_0}\left[\mathbb{P}_{\bZ}\left(|e^{-t}\overline{Y}_{0}^{(i)} + \sqrt{1-e^{-2t}}Z^{(i)}| > R, \Big|\ov\bY_{0}\right)\right]\\
&\leq\sum_{i=1}^{d_\mcY}\mathbb{P}_{\bZ}\left(|Z^{(i)}| > \frac{R-1}{\sqrt{1-e^{-2t}}}\right)\\
&\leq 2 d_\mcY \exp\left(-\frac{(R-1)^2}{2(1-e^{-2t})}\right)\\
&\leq 2d_\mcY\exp\left(-\frac{(R-1)^2}{4}\right).
\end{aligned}
$$
Thus, the second term can be bounded by
$$
\begin{aligned}
&~\mathbb{E}_{\ov\bY_t}\left(\Vert \bs(t,\overline{\bY}_t,\bx) - \bb(t,\overline{\bY}_t,\bx)\Vert^2\mI_{\{\Vert \overline{\bY}_t \Vert_{\infty} > R\}}\right)\\
\lesssim &~\frac{d_\mcY^2R^2}{(1-e^{-t})^2}\exp\left(-\frac{(R-1)^2}{4}\right) + \frac{d_\mcY^{3/2}}{(1-e^{-t})^2}\exp\left(-\frac{(R-1)^2}{8}\right)\\
\lesssim &~\frac{d_\mcY^2(R-1)^2}{(1-e^{-t})^2}\exp\left(-\frac{(R-1)^2}{8}\right).
\end{aligned}
$$
Using $x\leq e^{x}$, we have
$$
\frac{(R-1)^2}{16} \leq \exp\left(\frac{(R-1)^2}{16}\right).
$$
Thus, we have
\begin{equation}\label{eq: L2_Bound}
\mathbb{E}_{\ov\bY_t}\left(\Vert \bs(t,\overline{\bY}_t,\bx) - \bb(t,\overline{\bY}_t,\bx)\Vert^2\mI_{\{\Vert \overline{\bY}_t \Vert_{\infty} > R\}}\right)
\lesssim \frac{d_\mcY^2}{(1-e^{-t})^2}\exp\left(-\frac{(R-1)^2}{16}\right).
\end{equation}
Let the right hand of \eqref{eq: L2_Bound} be less than $\epsilon^2$, we obtain
$$
R\gtrsim \sqrt{\log\frac{d_\mcY}{\epsilon(1-e^{-t})}}.
$$
Therefore, we can choose
\begin{equation*}
R = \mathcal{O}\left(\sqrt{\log\frac{d_\mcY}{\epsilon(1-e^{-T_0})}}\right)= \mathcal{O}\left(\sqrt{\log\frac{d_\mcY}{\epsilon T_0}}\right),
\end{equation*}  
such that
$$
\mathbb{E}_{\ov\bY_t}\Vert \bs(t,\overline{\bY}_t,\bx) - \bb(t,\overline{\bY}_t,\bx)\Vert^2\leq(1 + d_\mcY)\epsilon^2.
$$
Substituting $R$ into the network configuration, we obtain
$$
L = \mathcal{O}\left(\log{\frac{1}{\epsilon}} + d_\mcX+ d_{\mathcal{Y}}\right),
M = \mathcal{O}\Bigg(\frac{d_\mcY^{\frac{5}{2}}(T-T_0)}{T_0^3}\left(\log{\frac{d_\mcY}{\epsilon T_0}}\right)^{\frac{d_\mcY+1}{2}}\xi^{d_\mcY}(\beta B_{\mcX})^{d_{\mcX}}
\epsilon^{-(d_{\mcX}+d_{\mcY}+1)}\Bigg),
$$
$$
J = \mathcal{O}\Bigg(\frac{d_\mcY^{\frac{5}{2}}(T-T_0)}{T_0^3}\left(\log{\frac{d_\mcY}{\epsilon T_0}}\right)^{\frac{d_\mcY+1}{2}}\xi^d(\beta B_{\mathcal{X}})^{d_{\mcX}}\epsilon^{-(d_\mcX+d_{\mathcal{Y}}+1)}\left(\log{\frac{1}{\epsilon}} + d_\mcX + d_{\mathcal{Y}}\right)\Bigg), 
$$
$$
K = \mathcal{O}\Bigg(\frac{\sqrt{d_\mcY\log{\frac{d_\mcY}{\epsilon T_0}}}}{T_0}\Bigg),
\kappa=\mathcal{O}\Bigg(
\xi\sqrt{\log{\frac{d_\mcY}{\epsilon T_0}}}
\vee
\frac{(T-T_0)d_\mcY^{3/2}\sqrt{\log\frac{d_\mcY}{\epsilon T_0}}}{T_0^3}\vee\beta B_{\mathcal{X}}
\Bigg),
$$
$$
\gamma_1 = 10d\xi=\mathcal{O}\left(\frac{20 d_\mcY^2}{T_0^2}\right), \gamma_2 = 10\tau(R)=\mathcal{O}\Big(\frac{10 d_\mcY^{3/2}\sqrt{\log\frac{d_\mcY}{\epsilon T_0}}}{T_0^3}\Big).
$$

\subsection{Statistical Error}\label{subsec: stat_error}
\textbf{Upper Bound for $\ell_{\bs}(\overline{\bY}_0,\bx)$.} Now we derive the upper bound for $\ell_{\bs}(\overline{\bY}_0,\bx)$. For notational convenience, we define
$$
h_{\bs}(t,\overline{\bY}_0,\bx,\bZ) := \left\Vert \bs(t,e^{-t}\overline{\bY}_{0} + \sqrt{1-e^{-2t}}\bZ, \bx) -e^{-t}\overline{\bY}_{0} + \frac{(1 + e^{-2t})\bZ}{\sqrt{1-e^{-2t}}}\right\Vert^2,
$$
then
$$
\begin{aligned}
\mathbb{E}_{\bZ}[h_{\bs}(t,\overline{\bY}_0,\bx,\bZ)]
&\leq 3\left(K^2 + d_\mcY e^{-2t} + \frac{(1 + e^{-2t})^2}{1-e^{-2t}}\mathbb{E}_{\bZ}\Vert \bZ\Vert^2\right)\\
&=3\left(K^2 + \frac{d_\mcY(1 + 3e^{-2t})}{1-e^{-2t}}\right)\\
&\leq 3\left(K^2 + \frac{4d_\mcY}{1-e^{-{2T_0}}}\right)\\
&\leq 3\left(K^2 + \frac{4d_\mcY}{T_0}\right).
\end{aligned}
$$
Thus, we have
$$
\ell_{\bs}(\overline{\bY}_0,\bx)\leq A_{T_0}:= 3\left(K^2 + \frac{4d_\mcY}{T_0}\right).
$$
\textbf{Lipschitz Continuity for $\ell_{\bs}(\overline{\bY}_0,\bx)$.} Now we derive the Lipschitz continuity for $\ell_{\bs}(\overline{\bY}_0,\bx)$.
Some elementary computations show that
$$
\begin{aligned}
&~|\ell_{\bs_1}(\overline{\bY}_0,\bx) - \ell_{\bs_2}(\overline{\bY}_0,\bx)|\\
\leq &~ \frac{1}{T - T_0}\int_{T_0}^{T}\mathbb{E}_{\bZ}\Vert \bs_1 - \bs_2\Vert\left\Vert \bs_1 + \bs_2 - 2e^{-t}\overline{\bY}_0  + \frac{2(1 + e^{-2t})\bZ}{\sqrt{1-e^{-2t}}}\right\Vert \mrd t\\
\leq &~ \frac{1}{T - T_0}\int_{T_0}^{T}\left(\mathbb{E}_{\bZ}\Vert \bs_1 - \bs_2\Vert^2\right)^{\frac{1}{2}}\left(\mathbb{E}_{\bZ}\left\Vert 
\bs_1 + \bs_2 - 2e^{-t}\overline{\bY}_0  + \frac{2(1 + e^{-2t})\bZ}{\sqrt{1-e^{-2t}}}
\right\Vert^2\right)^{\frac{1}{2}}\mrd t\\
\leq &~ \frac{1}{T - T_0}\left(\int_{T_0}^{T}\mathbb{E}_{\bZ}\Vert \bs_1 - \bs_2\Vert^2 \mrd t\right)^{\frac{1}{2}}
\left(\int_{T_0}^{T}\mathbb{E}_{\bZ}\left\Vert 
\bs_1 + \bs_2 - 2e^{-t}\overline{\bY}_0  + \frac{2(1 + e^{-2t})\bZ}{\sqrt{1-e^{-2t}}}
\right\Vert^2 \mrd t\right)^{\frac{1}{2}}\\
\leq &~ \left(\frac{1}{T - T_0}\int_{T_0}^{T}\mathbb{E}_{\bZ}\Vert \bs_1 - \bs_2\Vert^2 \mrd t\right)^{\frac{1}{2}}\left(8K^2 + \frac{64 d_\mcY }{1-e^{-2T_0}}\right)^{\frac{1}{2}}\\
\leq &~ \left(8K^2 + \frac{64 d_\mcY}{T_0}\right)^{\frac{1}{2}}
\Vert \bs_1 - \bs_2\Vert_{L^{\infty}([T_0,T]\times\mathbb{R}^{d_\mcY}\times[0,B_{\mathcal{X}}]^{d_{\mathcal{X}}})}.
\end{aligned}
$$
For convenience, we let
$$
B_{T_0}:=\left(8K^2 + \frac{64 d_\mcY}{T_0}\right)^{\frac{1}{2}}.
$$
\\
\textbf{Covering Number Evaluation.} 
Denote $\mcN_{\delta}:=\mathcal{N}\left(\mathrm{NN},\delta,\Vert\cdot\Vert_{L^{\infty}([T_0,T]\times\mathbb{R}^{d_\mcY}\times[0,B_{\mathcal{X}}]^{d_{\mathcal{X}}})}\right)$ as the $\delta$-covering number of the neural network class $\mathrm{NN}$.  By \cite{chen2022distribution},
$\mcN_{\delta}$ is then  evaluated as follows:
$$
\begin{aligned}
\log\mcN_{\delta} &= \log\mathcal{N}\left(\mathrm{NN},\delta,\Vert\cdot\Vert_{L^{\infty}([T_0,T]\times\mathbb{R}^{d_\mcY}\times[0,B_{\mathcal{X}}]^{d_{\mathcal{X}}})}\right)\\
&=\log\mathcal{N}\left(\mathrm{NN},\delta,\Vert\cdot\Vert_{L^{\infty}([T_0,T]\times[-R,R]^{d_\mcY}\times[0,B_{\mathcal{X}}]^{d_{\mathcal{X}}})}\right)\\
&\lesssim JL\log\left(\frac{LM(T \vee R\vee B_{\mathcal{X}})\kappa}{\delta}\right)
\\
&=\widetilde{\mathcal{O}}\left(\frac{(T-T_0)\xi^{d_\mcY}(\beta B_{\mathcal{X}})^{d_{\mathcal{X}}}\epsilon^{-(d_\mcX+ d_{\mathcal{Y}} + 1)}}{T_0^3}\right)
.
\end{aligned}
$$
\\
\noindent\textbf{Statistical Error Bound.} Let $\widetilde\ell_{\bs}(\overline{\bY}_{0},\bx) = \ell_{\bs}(\overline{\bY}_{0},\bx) - \ell_{\bs^{*}}(\overline{\bY}_{0},\bx)$ and $\mathcal{D}^{\prime}=\{(\overline{\bY}_{0,i}^{\prime},\bx_{i}^{\prime})\}_{i=1}^{n}$ be an independent copy of $\mathcal{D}$. By using the Lipschitz continuity of $\ell_{\bs}$, we have
$$
\begin{aligned}
|\widetilde\ell_{\bs_1}(\overline{\bY}_{0},\bx) - \widetilde\ell_{\bs_2}(\overline{\bY}_{0},\bx)| &= |\ell_{\bs_1}(\overline{\bY}_{0},\bx) - \ell_{\bs_2}(\overline{\bY}_{0},\bx)|\\
&\leq B_{T_0}\Vert \bs_1 - \bs_2\Vert_{L^{\infty}([T_0,T]\times\mathbb{R}^{d_\mcY}\times[0,B_{\mathcal{X}}]^{d_{\mathcal{X}}})}.
\end{aligned}
$$
We first estimate $\mathbb{E}_{\mathcal{D},\mathcal{T},\mathcal{Z}}\left(\mathcal{L}(\wh{\bs}) - 2\overline{\mathcal{L}}_{\mathcal{D}}(\wh{\bs}) + \mathcal{L}(\bs^{*})\right)$. We define
$G_{\wh{\bs}}(\overline{\bY}_{0},\bx) = \mathbb{E}_{\mathcal{D}^{\prime}}[\widetilde\ell_{\wh{\bs}}(\overline{\bY}_0^{\prime},\bx^{\prime})] - 2\widetilde\ell_{\bs}(\overline{\bY}_{0},\bx)$, then
$$
\begin{aligned}
&~\mathbb{E}_{\mathcal{D},\mathcal{T},\mathcal{Z}}\left(\mathcal{L}(\wh{\bs}) - 2\overline{\mathcal{L}}_{\mathcal{D}}(\wh{\bs}) + \mathcal{L}(\bs^{*})\right) \\
=&~\mathbb{E}_{\mathcal{D},\mathcal{T},\mathcal{Z}}\left(
\mathbb{E}_{\mathcal{D}^{\prime}}\left[\frac{1}{n}\sum_{i=1}^{n}(\ell_{\wh{\bs}}(\overline{\bY}_{0,i}^{\prime}, \bx_{i}^{\prime}) - \ell_{\bs^{*}}(\overline{\bY}_{0,i}^{\prime},\bx_{i}^{\prime}))\right] - \frac{2}{n}\sum_{i=1}^{n}(\ell_{\wh{\bs}}(\overline{\bY}_{0,i},\bx_{i}) - \ell_{\bs^{*}}(\overline{\bY}_{0,i},\bx_{i}))\right)\\
=&~\mathbb{E}_{\mathcal{D},\mathcal{T},\mathcal{Z}}\left[\frac{1}{n}\sum_{i=1}^{n}G_{\wh{\bs}}(\overline{\bY}_{0,i},\bx_{i})\right].
\end{aligned}
$$
Given $\delta > 0$, we denote the $\delta$-covering with minimal cardinality $\mcN_{\delta}$ as $\mathcal{C}_{\delta}$. For any $\bs\in\mathrm{NN}$, there exist $\bs_{\delta}\in\mathcal{C}_{\delta}$ such that
$$
|\widetilde\ell_{\bs}(\overline{\bY}_0,\bx) - \widetilde\ell_{\bs_{\delta}}(\overline{\bY}_0,\bx)|\leq B_{T_0}\Vert \bs - \bs_{\delta}\Vert_{L^{\infty}([T_0,T]\times\mathbb{R}^{d_\mcY}\times[0,B_{\mathcal{X}}]^{d_{\mathcal{X}}})}
\leq B_{T_0}\delta.
$$
Therefore,
$$
G_{\bs}(\overline{\bY}_0,\bx)\leq G_{\bs_{\delta}}(\overline{\bY}_0,\bx) + 3B_{T_0}\delta.
$$
For any $(\overline{\bY}_{0,i},\bx_{i})$,  $|\widetilde{\ell}_{\bs_{\delta}}(\overline{\bY}_{0,i},\bx_{i})|\leq 2A_{T_0}$, we have $|\widetilde{\ell}_{\bs_{\delta}}(\overline{\bY}_{0,i},\bx_{i}) - \mathbb{E}_{\mathcal{D}}[\widetilde{\ell}_{\bs_{\delta}}(\overline{\bY}_{0,i},\bx_{i})]|\leq 4A_{T_0}$. Let $V^2 = \mathrm{Var}[\widetilde{\ell}_{\bs_{\delta}}(\overline{\bY}_{0,i},\bx_{i})]$, then
$$
\begin{aligned}
V^2\leq\mathbb{E}_{\mathcal{D}}[\widetilde{\ell}_{\bs_{\delta}}(\overline{\bY}_{0,i},\bx_{i})]^2
&\leq B^2_{T_0}\mathbb{E}_{\mathcal{D}}\left(\frac{1}{T - T_0}\int_{T_0}^{T}\mathbb{E}_{\bZ}\Vert\bs_{\delta} - \bs^{*}\Vert^2 \mrd t\right)\\
&=B^2_{T_0}[\mathcal{L}(\bs_{\delta}) - \mathcal{L}(\bs^{*})]\\
&=B^2_{T_0}\mathbb{E}_{\mathcal{D}}[\ell_{\bs_{\delta}}(\overline{\bY}_{0,i},\bx_{i}) - \ell_{\bs^{*}}(\overline{\bY}_{0,i},\bx_{i})]\\
&= B^2_{T_0}\mathbb{E}_{\mathcal{D}}[\widetilde{\ell}_{\bs_{\delta}}(\overline{\bY}_{0,i},\bx_{i})].
\end{aligned}
$$
We obtain
$$
\mathbb{E}_{\mathcal{D}}[\widetilde{\ell}_{\bs_{\delta}}(\overline{\bY}_{0,i},\bx_{i})] \geq \frac{V^2}{B^2_{T_0}}.
$$
By Bernstein's inequality, we have
$$
\begin{aligned}
&~\mathbb{P}_{\mathcal{D},\mathcal{T},\mathcal{Z}}\left[\frac{1}{n}\sum_{i=1}^{n}G_{\bs_{\delta}}(\overline{\bY}_{0,i},\bx_{i}) > t\right]\\
=&~\mathbb{P}_{\mathcal{D},\mathcal{T},\mathcal{Z}}\left(
\mathbb{E}_{\mathcal{D}^{\prime}}\left[\frac{1}{n}\sum_{i=1}^{n}\widetilde{\ell}_{\bs_{\delta}}(\overline{\bY}_{0,i}^{\prime},\bx_{i}^{\prime})\right] - \frac{1}{n}\sum_{i=1}^{n}\widetilde{\ell}_{\bs_{\delta}}(\overline{\bY}_{0,i},\bx_{i}) > \frac{t}{2} + \mathbb{E}_{\mathcal{D}^{\prime}}\left[\frac{1}{2n}\sum_{i=1}^{n}\widetilde{\ell}_{\bs_{\delta}}(\overline{\bY}_{0,i}^{\prime},\bx_{i}^{\prime})\right]
\right)\\
=&~\mathbb{P}_{\mathcal{D},\mathcal{T},\mathcal{Z}}\left(
\mathbb{E}_{\mathcal{D}}\left[\frac{1}{n}\sum_{i=1}^{n}\widetilde{\ell}_{\bs_{\delta}}(\overline{\bY}_{0,i},\bx_{i})\right] - \frac{1}{n}\sum_{i=1}^{n}\widetilde{\ell}_{\bs_{\delta}}(\overline{\bY}_{0,i},\bx_{i}) > \frac{t}{2} + \mathbb{E}_{\mathcal{D}}\left[\frac{1}{2n}\sum_{i=1}^{n}\widetilde{\ell}_{\bs_{\delta}}(\overline{\bY}_{0,i},\bx_{i})\right]
\right)\\
\leq&~\mathbb{P}_{\mathcal{D},\mathcal{T},\mathcal{Z}}\left(
\mathbb{E}_{\mathcal{D}}\left[\frac{1}{n}\sum_{i=1}^{n}\widetilde{\ell}_{\bs_{\delta}}(\overline{\bY}_{0,i},\bx_{i})\right] - \frac{1}{n}\sum_{i=1}^{n}\widetilde{\ell}_{\bs_{\delta}}(\overline{\bY}_{0,i},\bx_{i}) > \frac{t}{2} + \frac{V^2}{2B^2_{T_0}}
\right)\\
\leq&~\exp\left(-\frac{nu^2}{2V^2 + \frac{8uA_{T_0}}{3}}\right)\\
\leq&~\exp\left(-\frac{nt}{8B^2_{T_0} + \frac{16A_{T_0}}{3}}\right),
\end{aligned}
$$
where $u = \frac{t}{2} + \frac{V^2}{2B^2_{T_0}}$, and we use $u\geq\frac{t}{2}$ and $V^2\leq 2B^2_{T_0}u$. Hence, for any $t > 3B_{T_0}\delta$, we have
$$
\begin{aligned}
\mathbb{P}_{\mathcal{D},\mathcal{T},\mathcal{Z}}\left[\frac{1}{n}\sum_{i=1}^{n}G_{\wh{\bs}}(\overline{\bY}_{0,i},\bx_{i}) > t\right]
\leq&~\mathbb{P}_{\mathcal{D},\mathcal{T},\mathcal{Z}}\left[\mathop{\sup}_{\bs\in\mathrm{NN}}\frac{1}{n}\sum_{i=1}^{n}G_{\bs}(\overline{\bY}_{0,i},\bx_{i}) > t\right]\\
\leq &~ \mathbb{P}_{\mathcal{D},\mathcal{T},\mathcal{Z}}\left[\mathop{\max}_{\bs_{\delta}\in\mathcal{C}_{\delta}}\frac{1}{n}\sum_{i=1}^{n}G_{\bs_{\delta}}(\overline{\bY}_{0,i},\bx_{i}) > t - 3B_{T_0}\delta\right]\\
\leq &~ \mathcal{N}_{\delta}\cdot\mathop{\max}_{\bs_{\delta}\in\mathcal{C}_{\delta}}\mathbb{P}_{\mathcal{D},\mathcal{T},\mathcal{Z}}\left[\frac{1}{n}\sum_{i=1}^{n}G_{\bs_{\delta}}(\overline{\bY}_{0,i},\bx_{i}) > t - 3B_{T_0}\delta\right]\\
\leq &~ \mathcal{N}_{\delta}\exp\left(-\frac{n(t - 3B_{T_0}\delta)}{8B^2_{T_0} + \frac{16A_{T_0}}{3}}\right).
\end{aligned}
$$
By setting $a = 3B_{T_0}\delta + \frac{8B^2_{T_0} + \frac{16A_{T_0}}{3}}{n}$ and $\delta = \frac{1}{n}$, then we obtain
$$
\begin{aligned}
\mathbb{E}_{\mathcal{D},\mathcal{T},\mathcal{Z}}\left[\frac{1}{n}\sum_{i=1}^{n}G_{\wh{\bs}}(\overline{\bY}_{0,i},\bx_{i})\right]
&\leq a + \mathcal{N}_{\frac{1}{n}}\int_{a}^{\infty}\exp\left(-\frac{n(t - 3B_{T_0}\delta)}{8B^2_{T_0} + \frac{16A_{T_0}}{3}}\right)dt\\
&\leq\frac{\left(8B^2_{T_0} + \frac{16A_{T_0}}{3}\right)\left(1 + \log\mathcal{N}_{\frac{1}{n}}\right) + 3B_{T_0}}{n}\\
&=\widetilde{\mathcal{O}}\left(\frac{1}{n}\cdot\frac{(T-T_0)\xi^{d_\mcY}(\beta B_{\mathcal{X}})^{d_{\mathcal{X}}}\epsilon^{-(d_\mcX+ d_{\mathcal{Y}} + 1)}}{T_0^5}\right).
\end{aligned}
$$
Next, we estimate the term $\mathbb{E}_{\mathcal{D},\mathcal{T},\mathcal{Z}}\left(\overline{\mathcal{L}}_{\mathcal{D}}(\wh{\bs}) - \wh{\mathcal{L}}_{\mathcal{D},\mathcal{T},\mathcal{Z}}(\wh{\bs})\right)$. Recall that
$$
\overline{\mathcal{L}}_{\mathcal{D}}(\wh{\bs}) - \wh{\mathcal{L}}_{\mathcal{D},\mathcal{T},\mathcal{Z}}(\wh{\bs}) = \frac{1}{n}\sum_{i=1}^{n}\left(
\ell_{\wh{\bs}}(\overline{\bY}_{0,i},\bx_{i}) - \wh{\ell}_{\wh{\bs}}(\overline{\bY}_{0,i},\bx_{i})
\right).
$$
Let $h_{\bs}^{\mathrm{trunc}}(t,\overline{\bY}_{0},\bx,\bZ):=h_{\bs}(t,\overline{\bY}_{0},\bx,\bZ)\mI_{\{\Vert \bZ\Vert_{\infty}\leq r\}}$. We decompose 
$$
\frac{1}{n}\sum_{i=1}^{n}\left(\ell_{\wh{\bs}}(\overline{\bY}_{0,i},\bx_{i}) - \wh{\ell}_{\wh{\bs}}(\overline{\bY}_{0,i},\bx_{i})\right)
$$ into the following three terms:
$$
\begin{aligned}
\frac{1}{n}\sum_{i=1}^{n}\left(\ell_{\wh{\bs}}(\overline{\bY}_{0,i},\bx_{i}) - \wh{\ell}_{\wh{\bs}}(\overline{\bY}_{0,i},\bx_{i})\right)
& = \underbrace{\frac{1}{n}\sum_{i=1}^{n}\left(
\ell_{\wh{\bs}}(\overline{\bY}_{0,i},\bx_{i}) - \ell_{\wh{\bs}}^{\mathrm{trunc}}(\overline{\bY}_{0,i},\bx_{i})
\right)}_{(\mathrm{I})}\\
&+\underbrace{\frac{1}{n}\sum_{i=1}^{n}\left(
\ell_{\wh{\bs}}^{\mathrm{trunc}}(\overline{\bY}_{0,i},\bx_{i}) - \wh{\ell}_{\wh{\bs}}^{\mathrm{trunc}}(\overline{\bY}_{0,i},\bx_{i})
\right)}_{(\mathrm{II})}\\
&+\underbrace{\frac{1}{n}\sum_{i=1}^{n}\left(
\wh{\ell}_{\wh{\bs}}^{\mathrm{trunc}}(\overline{\bY}_{0,i},\bx_{i}) - \wh{\ell}_{\wh{\bs}}(\overline{\bY}_{0,i},\bx_{i})
\right)}_{(\mathrm{III})},
\end{aligned}
$$
where
$$
\ell_{\wh{\bs}}^{\mathrm{trunc}}(\overline{\bY}_{0,i},\bx_{i}):= \frac{1}{T - T_0}\int_{T_0}^{T}\mathbb{E}_{\bZ}\left[h_{\wh{\bs}}^{\mathrm{trunc}}(t,\overline{\bY}_{0,i},\bx_{i},\bZ)\right] \mrd t
$$
and
$$
\wh{\ell}_{\wh{\bs}}^{\mathrm{trunc}}(\overline{\bY}_{0,i},\bx_{i}):=\frac{1}{m}\sum_{j=1}^{m}h_{\wh{\bs}}^{\mathrm{trunc}}(t_j,\overline{\bY}_{0,i},\bx_{i},\bZ_{j}).
$$
Now, we estimate these three terms separately. The first term satisfies
$$
\begin{aligned}
(\mathrm{I}) &= \frac{1}{n}\sum_{i=1}^{n}\left(\frac{1}{T - T_0}\int_{T_0}^{T}\mathbb{E}_{\bZ}\left[h_{\bs}(t,\overline{\bY}_{0,i},\bx_{i},\bZ)\mI_{\{\Vert \bZ \Vert_{\infty} > r\}}\right] \mrd t\right)\\
&\leq 3(K^2 + d_\mcY e^{-2t})\mathbb{P}_{\bZ}(\Vert \bZ \Vert_{\infty} > r) + \frac{3(1 + e^{-2t})^2}{1-e^{-2t}}\mathbb{E}_{\bZ}[\Vert \bZ\Vert^2\mI_{\{\Vert \bZ \Vert_{\infty} > r\}}]\\
&\leq 3(K^2 + d_\mcY e^{-2t})\exp\left(-\frac{r^2}{2}\right) + \frac{3(1 + e^{-2t})^2}{1-e^{-2t}}\sqrt{d_\mcY(d_\mcY+ 2)}\exp\left(-\frac{r^2}{4}\right)\\
&\leq 3(K^2 + d_\mcY)\exp\left(-\frac{r^2}{2}\right) + \frac{12\sqrt{d_\mcY(d_\mcY+ 2)}}{T_0}\exp\left(-\frac{r^2}{4}\right).
\end{aligned}
$$
Therefore,
\begin{equation}\label{eq: trunc_1}
\begin{aligned}
&\mathbb{E}_{\mathcal{D},\mathcal{T},\mathcal{Z}}\left[\frac{1}{n}\sum_{i=1}^{n}\left(
\ell_{\wh{\bs}}(\overline{\bY}_{0,i},\bx_{i}) - \ell_{\wh{\bs}}^{\mathrm{trunc}}(\overline{\bY}_{0,i},\bx_{i})\right)\right]\\
\leq &\left(3K^2 + 3d_\mcY + \frac{12\sqrt{d_\mcY(d_\mcY+ 2)}}{T_0}\right)\exp\left(-\frac{r^2}{4}\right).
\end{aligned}
\end{equation}
Next, we estimate the second term. For any paired sample $(\overline{\bY}_{0,i},\bx_{i})$, we have
$$
0\leq h_{\bs}^{\mathrm{trunc}}(t,\overline{\bY}_{0,i},\bx_{i},\bZ)\leq H_{T_0}(r):= 3\left(K^2 + d_\mcY + \frac{4d_\mcY r^2}{T_0}\right).
$$
For any $\bs\in\mathrm{NN}$, there exists $\bs_{\delta}\in\mathcal{C}_{\delta}$ such that
$$
\begin{aligned}
&~|h_{\bs}^{\mathrm{trunc}}(t,\overline{\bY}_{0,i},\bx_{i},\bZ) - h_{\bs_{\delta}}^{\mathrm{trunc}}(t,\overline{\bY}_{0,i},\bx_{i},\bZ)|\\
\leq &~ \delta\left\Vert \bs + \bs_{\delta} -2e^{-t}\overline{\bY}_0 + \frac{2(1+e^{-2t})\bZ}{\sqrt{1-e^{-2t}}}\right\Vert\mI_{\{\Vert \bZ\Vert_{\infty}\leq r\}}\\
\leq &~ \delta\left(2K +2\sqrt{d_\mcY} + \frac{4\sqrt{d_\mcY}r}{\sqrt{T_0}}\right).
\end{aligned}
$$
Then, for any paired sample $(\overline{\bY}_{0,i},\bx_{i})$, we have
$$
\begin{aligned}
&\frac{1}{T - T_0}\int_{T_0}^{T}\mathbb{E}_{\bZ}\left[
h_{\wh{\bs}}^{\mathrm{trunc}}(t,\overline{\bY}_{0,i}\bx_{i},\bZ)
\right]\mrd t - \frac{1}{m}\sum_{j=1}^{m}h_{\wh{\bs}}^{\mathrm{trunc}}(t_j,\overline{\bY}_{0,i},\bx_{i},\bZ_j)\\
\leq &\mathop{\sup}_{\bs\in\mathrm{NN}}\left(
\frac{1}{T - T_0}\int_{T_0}^{T}\mathbb{E}_{\bZ}\left[
h_{\bs}^{\mathrm{trunc}}(t,\overline{\bY}_{0,i},\bx_{i},\bZ)
\right]\mrd t - \frac{1}{m}\sum_{j=1}^{m}h_{\bs}^{\mathrm{trunc}}(t_j,\overline{\bY}_{0,i},\bx_{i},\bZ_j)\right)\\
\leq &\mathop{\max}_{\bs_{\delta}\in\mathcal{C}_{\delta}}\left(
\frac{1}{T - T_0}\int_{T_0}^{T}\mathbb{E}_{\bZ}\left[
h_{\bs_{\delta}}^{\mathrm{trunc}}(t,\overline{\bY}_{0,i},\bx_{i},\bZ)
\right]\mrd t - \frac{1}{m}\sum_{j=1}^{m}h_{\bs_{\delta}}^{\mathrm{trunc}}(t_j,\overline{\bY}_{0,i},\bx_{i},\bZ_j)\right)\\
&~~~ +4\delta\left(K + \sqrt{d_\mcY} +  \frac{2\sqrt{d_\mcY}r}{\sqrt{T_0}}\right).
\end{aligned}
$$
Let $b:=4\delta\left(K + \sqrt{d_\mcY} +  \frac{2\sqrt{d_\mcY}r}{\sqrt{T_0}}\right)$. For $t > b$, using Hoeffding's inequality implies
$$
\begin{aligned}
&~\mathbb{P}_{\mathcal{T},\mathcal{Z}}
\left(\frac{1}{T - T_0}\int_{T_0}^{T}\mathbb{E}_{\bZ}\left[
h_{\hat{\bs}}^{\mathrm{trunc}}(t,\overline{\bY}_{0,i},\bx_{i},\bZ)
\right] \mrd t - \frac{1}{m}\sum_{j=1}^{m}h_{\wh{\bs}}^{\mathrm{trunc}}(t_j,\overline{\bY}_{0,i},\bx_{i},\bZ_j) > t\right)\\
\leq  &~\mathbb{P}_{\mathcal{T},\mathcal{Z}}
\left(\mathop{\max}_{\bs_{\delta}\in\mathcal{C}_{\delta}}\left(
\frac{1}{T - T_0}\int_{T_0}^{T}\mathbb{E}_{\bZ}\left[
h_{\bs_{\delta}}^{\mathrm{trunc}}(t,\overline{\bY}_{0,i},\bx_{i},\bZ)
\right]\mrd t - \frac{1}{m}\sum_{j=1}^{m}h_{\bs_{\delta}}^{\mathrm{trunc}}(t_j,\overline{\bY}_{0,i},\bx_{i},\bZ_j)\right) > t - b\right)\\
\leq &~ \mathcal{N}_{\delta}\cdot\exp\left(-\frac{2m(t-b)^2}{H_{T_0}^2(r)}\right).
\end{aligned}
$$
By taking expectation over $\mathcal{T},\mathcal{Z}$, the error $(\mathrm{II})$ satisfies
$$
\begin{aligned}
&\mathbb{E}_{\mathcal{T},\mathcal{Z}}\left[
\ell_{\wh{\bs}}^{\mathrm{trunc}}(\overline{\bY}_{0,i},\bx_{i}) - \wh{\ell}_{\wh{\bs}}^{\mathrm{trunc}}(\overline{\bY}_{0,i},\bx_{i})\right]\\ = &\int_{0}^{\infty}\mathbb{P}_{\mathcal{T},\mathcal{Z}}\left(\ell_{\wh{\bs}}^{\mathrm{trunc}}(\overline{\bY}_{0,i},\bx_{i}) - \wh{\ell}_{\wh{\bs}}^{\mathrm{trunc}}(\overline{\bY}_{0,i},\bx_{i}) > t\right) \mrd t\\
\leq &~ b + c + \mathcal{N}_{\delta}\cdot\int_{c}^{\infty}\exp\left(-\frac{2mt^2}{H_{T_0}^2(r)}\right)dt\\
\leq &~ b + c + \frac{\sqrt{\pi}}{2}\mathcal{N}_{\delta}\cdot\exp\left(-\frac{2mc^2}{H_{T_0}^2(r)}\right)\frac{H_{T_0}(r)}{\sqrt{2m}}.
\end{aligned}
$$
Thus, we have
\begin{equation}\label{eq: trunc_2}
\begin{aligned}
&~\mathbb{E}_{\mathcal{D},\mathcal{T},\mathcal{Z}}\left(\frac{1}{n}\sum_{i=1}^{n}\left[
\ell_{\wh{\bs}}^{\mathrm{trunc}}(\overline{\bY}_{0,i},\bx_{i}) - \wh{\ell}_{\wh{\bs}}^{\mathrm{trunc}}(\overline{\bY}_{0,i},\bx_{i})\right]\right) \\
= &~ \frac{1}{n}\sum_{i=1}^{n}\mathbb{E}_{\mathcal{D}}\left(\mathbb{E}_{\mathcal{T},\mathcal{Z}}\left[
\ell_{\wh{\bs}}^{\mathrm{trunc}}(\overline{\bY}_{0,i},\bx_{i}) - \wh{\ell}_{\wh{\bs}}^{\mathrm{trunc}}(\overline{\bY}_{0,i},\bx_{i})\right]\right)\\
\leq &~ b + c + \frac{\sqrt{\pi}}{2}\mathcal{N}_{\delta}\cdot\exp\left(-\frac{2mc^2}{H_{T_0}^2(r)}\right)\frac{H_{T_0}(r)}{\sqrt{2m}}.
\end{aligned}
\end{equation}
The last term can be expressed as
$$
(\mathrm{III}) = -\frac{1}{mn}\sum_{i=1}^{n}\sum_{j=1}^{m}h_{\wh{s}}(t_j,\overline{\bY}_{0,i},\bx_i,\bZ_{j})\mI_{\{\Vert \bZ_j \Vert_{\infty} > r\}} \leq 0.
$$
Thus, we have
\begin{equation}\label{eq: trunc_3}
\mathbb{E}_{\mathcal{D},\mathcal{T},\mathcal{Z}}\left[\frac{1}{n}\sum_{i=1}^{n}\left(
\wh{\ell}_{\wh{\bs}}^{\mathrm{trunc}}(\overline{\bY}_{0,i},\bx_{i}) - \wh{\ell}_{\wh{\bs}}(\overline{\bY}_{0,i},\bx_{i})
\right)\right]\leq 0.
\end{equation}
Combining \eqref{eq: trunc_1}, \eqref{eq: trunc_2} and \eqref{eq: trunc_3}, we obtain
$$
\begin{aligned}
&~\mathbb{E}_{\mathcal{D},\mathcal{T},\mathcal{Z}}\left[\frac{1}{n}\sum_{i=1}^{n}\left(
\ell_{\wh{\bs}}(\overline{\bY}_{0,i},\bx_{i}) - \wh{\ell}_{\wh{\bs}}(\overline{\bY}_{0,i},\bx_{i})
\right)\right]\\
\leq &~ b_0 + c + \frac{\sqrt{\pi}}{2}\mathcal{N}_{\delta}\cdot\exp\left(-\frac{2mc^2}{H_{T_0}^2(r)}\right)\frac{H_{T_0}(r)}{\sqrt{2m}},
\end{aligned}
$$
where 
$$
b_0 = \left(3K^2 + 3d_\mcY + \frac{12\sqrt{d_\mcY(d_\mcY+ 2)}}{T_0}\right)\exp\left(-\frac{r^2}{4}\right) + 4\delta\left(K + \sqrt{d_\mcY} +  \frac{2\sqrt{d_\mcY}r}{\sqrt{T_0}}\right).
$$
By setting $r = 2\sqrt{\log{m}}$, $\delta=\frac{1}{m}$, and $c = H_{T_0}(r)\sqrt{\frac{\log\mathcal{N}_{\frac{1}{m}}}{2m}}$, we obtain
$$
\begin{aligned}
&~\mathbb{E}_{\mathcal{D},\mathcal{T},\mathcal{Z}}\left[\frac{1}{n}\sum_{i=1}^{n}\left(
\ell_{\wh{\bs}}(\overline{\bY}_{0,i},\bx_{i}) - \wh{\ell}_{\wh{\bs}}(\overline{\bY}_{0,i},\bx_{i})
\right)\right]\\
\leq &~ b_0 + H_{T_0}(r)\frac{\sqrt{\log\mathcal{N}_{\frac{1}{m}}} + 1}{\sqrt{2m}}\\
\leq &~
\frac{1}{m}\left(3K^2 + 3d_\mcY + \frac{12\sqrt{d_\mcY(d_\mcY+ 2)}}{T_0}\right) + \frac{4}{m}\left(K + \sqrt{d_\mcY} +  \frac{4\sqrt{d_\mcY\log m}}{\sqrt{T_0}}\right)
\\
 &~~~~~~ +3\left(K^2 + d_\mcY + \frac{16d_\mcY\log m}{T_0}\right)\cdot\frac{\sqrt{\log\mathcal{N}_{\frac{1}{m}}} + 1}{\sqrt{2m}}\\
= &~ \widetilde{\mathcal{O}}\left(\frac{K^2}{m} + \frac{K^2}{\sqrt{m}}\cdot\frac{(T-T_0)^\frac{1}{2}\xi^{\frac{d_\mcY}{2}}(\beta B_{\mathcal{X}})^{\frac{d_{\mathcal{Y}}}{2}}\epsilon^{-\frac{d_\mcX+ d_{\mathcal{Y}} + 1}{2}}}{T_0^{\frac{3}{2}}}\right)\\
=&\widetilde{\mathcal{O}}\left(\frac{1}{\sqrt{m}}\cdot\frac{(T-T_0)^{\frac{1}{2}}\xi^{\frac{d_\mcY}{2}}(\beta B_{\mathcal{X}})^{\frac{d_{\mathcal{X}}}{2}}\epsilon^{-\frac{d_\mcX + d_{\mathcal{Y}} + 1}{2}}}{T_0^{\frac{7}{2}}}\right).
\end{aligned}
$$

\subsection{Error Bound for Drift Estimation}\label{subsec: error_bound_for_drift}
\begin{proof}[Proof of Theorem \ref{th: drift_estimation}.] Combining the approximation error and statistical error, we obtain
$$
\begin{aligned}
&~\mathbb{E}_{\mathcal{D},\mathcal{T},\mathcal{Z}}\left(\frac{1}{T - T_0}\int_{T_0}^{T}\mathbb{E}_{\ov\bY_t,\bx}\Vert \bs(t,\overline{\bY}_t,\bx) - \bb(t,\overline{\bY}_t,\bx)\Vert^2 \mrd t\right)\\
=&~\widetilde{\mathcal{O}}\left(\frac{1}{n}\cdot\frac{(T-T_0)\xi^{d_\mcY}(\beta B_{\mathcal{X}})^{d_{\mathcal{X}}}\epsilon^{-(d_\mcX+ d_{\mathcal{Y}} + 1)}}{T_0^5} + \frac{1}{\sqrt{m}}\cdot\frac{(T-T_0)^{\frac{1}{2}}\xi^{\frac{d_\mcY}{2}}(\beta B_{\mathcal{X}})^{\frac{d_{\mathcal{X}}}{2}}\epsilon^{-\frac{d_\mcX+ d_{\mathcal{Y}} + 1}{2}}}{T_0^{\frac{7}{2}}} + (1 + d_\mcY)\epsilon^2\right)\\
=&~\widetilde{\mathcal{O}}\left(
\frac{(T-T_0)\xi^d(\beta B_{\mathcal{X}})^{d_{\mathcal{X}}}}{T_0^5} 
\left(\frac{\epsilon^{-(d_\mcX+ d_{\mathcal{Y}}+ 1)}}{n} + \frac{\epsilon^{-\frac{d_\mcX+ d_{\mathcal{Y}} + 1}{2}}}{\sqrt{m}} + \epsilon^2
\right)\right).
\end{aligned}
$$
By setting $\epsilon = n^{-\frac{1}{d_\mcX + d_{\mathcal{Y}} + 3}}$, it holds that
$$
\begin{aligned}
&~\mathbb{E}_{\mathcal{D},\mathcal{T},\mathcal{Z}}\left(\frac{1}{T - T_0}\int_{T_0}^{T}\mathbb{E}_{\ov\bY_t,\bx}\Vert \bs(t,\overline{\bY}_t,\bx) - \bb(t,\overline{\bY}_t,\bx)\Vert^2 \mrd t\right)\\
=&~\widetilde{\mathcal{O}}\left(
\frac{(T-T_0)\xi^{d_\mcY}(\beta B_{\mathcal{X}})^{d_{\mathcal{X}}}}{T_0^5} 
\left(n^{-\frac{2}{d_\mcX+ d_{\mathcal{Y}} + 3}} + n^{\frac{d_\mcX+ d_{\mathcal{Y}} + 1}{2(d_\mcX+ d_{\mathcal{Y}} + 3)}}m^{-\frac{1}{2}}\right)
\right).
\end{aligned}
$$
The proof is complete.
\end{proof}

\section{Bound $\Ebb_{\mcD,\mcT,\mcZ}[\mathrm{TV}(\ck{p}_{T_0},\wt{p}_{T_0})]$}\label{sec:appendc}
In this section, we give the upper bound for $\Ebb_{\mcD,\mcT,\mcZ}[\mathrm{TV}(\ck{p}_{T_0},\wt{p}_{T_0})]$.
We first prove Lemma \ref{lem: KL_2_normal}. Then we prove Theorem \ref{th:TV_initial_distribution}.
\begin{proof}[Proof of Lemma \ref{lem: KL_2_normal}.]
By the definition of $\mathrm{KL}$ divergence, we have
    $$
    \begin{aligned}
    \mathrm{KL}(p|q) &= \Ebb_{p(\by)}\left[\log\frac{p(\by)}{q(\by)}\right]\\
    &=\Ebb_{p(\by)}[\log p(\by)] - \Ebb_{p(\by)}[\log q(\by)].
    \end{aligned}
    $$
We first calculate $ \Ebb_{p(\by)}[-\log q(\by)]$ as follows:
    $$
    \Ebb_{p(\by)}[-\log q(\by)] = \frac{d_\mcY}{2}\log 2\pi + \frac{1}{2}\log|\boldsymbol{\Sigma}_2| + \frac{1}{2}\Ebb_{p(\by)}\left[(\by-\boldsymbol{\mu}_2)^{\top}\boldsymbol{\Sigma}_2^{-1}(\by-\boldsymbol{\mu}_2)\right].
    $$
    Since $(\by-\boldsymbol{\mu}_2)\boldsymbol{\Sigma}_2^{-1}(\by-\boldsymbol{\mu}_2)\in\Rbb$, we have
    $$
    \begin{aligned}
        \Ebb_{p(\by)}\left[(\by-\boldsymbol{\mu}_2)^{\top}\boldsymbol{\Sigma}_2^{-1}(\by-\boldsymbol{\mu}_2)\right] &= \Ebb_{p(\by)}\left[\mathrm{Tr}\left((\by-\boldsymbol{\mu}_2)^{\top}\boldsymbol{\Sigma}_2^{-1}(\by-\boldsymbol{\mu}_2)\right)\right]\\
        &=\Ebb_{p(\by)}\left[\mathrm{Tr}\left(\boldsymbol{\Sigma}_2^{-1}(\by-\boldsymbol{\mu}_2)(\by-\boldsymbol{\mu}_2)^{\top}\right)\right]\\
        &=\mathrm{Tr}\left(\boldsymbol{\Sigma}_2^{-1}\Ebb_{p(\by)}\left[(\by-\boldsymbol{\mu}_2)(\by-\boldsymbol{\mu}_2)^{\top}\right]\right)\\
        &=\mathrm{Tr}\left(
        \boldsymbol{\Sigma}_2^{-1}(\boldsymbol{\Sigma}_1 + \boldsymbol{\mu}_1\boldsymbol{\mu}_1^{\top} - \boldsymbol{\mu}_2\boldsymbol{\mu}_1^{\top} - \boldsymbol{\mu}_1\boldsymbol{\mu}_2^{\top} + \boldsymbol{\mu}_2\boldsymbol{\mu}_2^{\top})
        \right)\\
        &=\mathrm{Tr}\left(\boldsymbol{\Sigma}_2^{-1}\boldsymbol{\Sigma}_1 + \boldsymbol{\Sigma}_2^{-1}(\boldsymbol{\mu}_1 - \boldsymbol{\mu}_2)(\boldsymbol{\mu}_1 - \boldsymbol{\mu}_2)^{\top}\right)\\
        &=\mathrm{Tr}\left(\boldsymbol{\Sigma}_2^{-1}\boldsymbol{\Sigma}_1\right) + (\boldsymbol{\mu}_1 - \boldsymbol{\mu}_2)^{\top}\boldsymbol{\Sigma}_2^{-1}(\boldsymbol{\mu}_1 - \boldsymbol{\mu}_2)^{\top}.
    \end{aligned}
    $$
    Therefore, we have
    $$
    \Ebb_{p(\by)}[-\log q(\by)] = \frac{d_\mcY}{2}\log 2\pi + \frac{1}{2}\log|\boldsymbol{\Sigma}_2| + \frac{1}{2}\mathrm{Tr}\left(\boldsymbol{\Sigma}_2^{-1}\boldsymbol{\Sigma}_1\right) + (\boldsymbol{\mu}_1 - \boldsymbol{\mu}_2)^{\top}\boldsymbol{\Sigma}_2^{-1}(\boldsymbol{\mu}_1 - \boldsymbol{\mu}_2)^{\top}.
    $$
    In the above equality, we set $q(\by)=p(\by)$, then we obtain
    $$
    \Ebb_{p(\by)}[\log p(\by)] = -\frac{d_\mcY}{2}\log 2\pi - \frac{1}{2}\log|\boldsymbol{\Sigma}_1| - \frac{d_\mcY}{2}.
    $$
    Combining these two equations, we finally obtain
    $$
    \mathrm{KL}(p|q) = \frac{1}{2}\left[
    (\boldsymbol{\mu}_1-\boldsymbol{\mu}_2)^{\top}\boldsymbol{\Sigma}_2^{-1}(\boldsymbol{\mu}_1 - \boldsymbol{\mu}_2) - \log\left|\boldsymbol{\Sigma}_2^{-1}\boldsymbol{\Sigma}_1\right| + \mathrm{Tr}(\boldsymbol{\Sigma}_2^{-1}\boldsymbol{\Sigma}_1) - d_\mcY
    \right].
    $$
    In particular, when $q(\by) = \mcN(\m0,\mI_{d_\mcY})$, the above result reduces to
    $$
    \mathrm{KL}(p|q) = \frac{1}{2}\left[\Vert\boldsymbol{\mu}_1\Vert^2-\log|\boldsymbol{\Sigma}_1| + \mathrm{Tr}(\boldsymbol{\Sigma}_1) - d_\mcY\right].
    $$
    The proof is complete.
\end{proof}

\begin{proof}[Proof of Theorem \ref{th:TV_initial_distribution}.]
    Since \eqref{sde: EM_scheme} and \eqref{sde: distri_replace} differ only in their initial distributions, by data processing inequality, we have
    $$
    \Ebb_{\mcD,\mcT,\mcZ}[\mathrm{TV}(\ck{p}_{T_0},\wt{p}_{T_0})]\leq\mathrm{TV}(p_{T}, \mcN(\m0,\mI_{d_\mcY})).
    $$
    According to Lemma \ref{lem: KL_2_normal}, we have
    $$
    \mathrm{KL}(p_{T}(\by|\ov\by_0,\bx),\mcN(\m0,\mI_{d_\mcY})) = \frac{1}{2}\left[e^{-2T}\Vert\ov\by_0\Vert^2-d\log(1-e^{-2T}) -de^{-2T}\right].
    $$
    By the convexity of the $\mathrm{KL}$ divergence, we have 
    $$
    \begin{aligned}
        \mathrm{KL}
        \left(p_{T}(\by|\bx)\big|\mcN(\m0,\mI_{d_\mcY})\right) &= \mathrm{KL}\left(\int_{\Rbb^{d_\mcY}}p_{T}(\by|\ov\by_0,\bx)p_{0}(\ov\by_0|\bx)\mrd\ov\by_0\big|\mcN(\m0,\mI_{d_\mcY})\right)\\
        &\leq\int_{\Rbb^{d_\mcY}}\mathrm{KL}\left(p_{T}(\by|\ov\by_0,\bx)\big|\mcN(\m0,\mI_{d_\mcY})\right)p_{0}(\ov\by_0|\bx)\mrd\ov\by_0\\
        &=\frac{1}{2}\left[e^{-2T}\Ebb_{\ov\bY_0}\Vert\ov\bY_0\Vert^2-d\log(1-e^{-2T}) -de^{-2T}\right]\\
        &\lesssim e^{-2T}.
    \end{aligned}
    $$
    Therefore, by the Pinsker's inequality  $\mathrm{TV}^2(p,q)\leq\frac{1}{2}\mathrm{KL}(p|q)$, we finally obtain
    $$
    \Ebb_{\mcD,\mcT,\mcZ}[\mathrm{TV}(\ck{p}_{T_0},\wt{p}_{T_0})]\lesssim e^{-T}.
    $$
    The proof is complete.
\end{proof}

\section{Bound $\mathrm{TV}(p_{T_0},p_0)$}\label{sec:appendd}
In this section, we  prove Theorem \ref{th: TV_early_stopping}. 

\begin{proof}[Proof of Theorem \ref{th: TV_early_stopping}.] 
Let $\mu_t:= e^{-t}$, $\sigma_t^2:=1-e^{-2t}$. We first show that there exists a constant $C>2$ such that for any $0 < \epsilon < 1$,
\begin{equation}\label{eq: tail_prob}
\int_{\Rbb^{d_\mcY}\backslash A_{\by}}|p_0(\by|\bx)-p_t(\by|\bx)|\mrd\by \lesssim\epsilon,
\end{equation}
where $A_{\by} = \left[-C, C\right]^{d_\mcY}$.

Since $p_0(\by|\bx)$ is supported on $[0,1]^{d_\mcY}$, we have
$$
\begin{aligned}
\int_{\Rbb^{d_\mcY}\backslash A_{\by}}|p_0(\by|\bx)-p_t(\by|\bx)|\mrd\by &= \int_{\Rbb^{d_\mcY}\backslash A_{\by}}p_t(\by|\bx)\mrd\by\\
&=\mathbb{P}_{\ov\bY_t}\left(\Vert \ov{\bY}_t \Vert_{\infty} > C|\bx\right)\\
&=\Ebb_{\ov\bY_0}\left[\mathbb{P}_{\bZ}\left(\Vert \mu_t\ov{\bY}_0 + \sigma_t\bZ\Vert_{\infty} > C|\ov\bY_0, \bx\right)\right]\\
&\leq \mathbb{P}_{\bZ}\left(\Vert\bZ\Vert_{\infty} > \frac{C-\mu_t}{\sigma_t}\Big|\bx\right)\\
&\leq \mathbb{P}_{\bZ}\left(\Vert\bZ\Vert_{\infty} > \frac{C}{2}\right)\\
&\leq 2d_{\mcY}\exp\left(-\frac{C^2}{8}\right).
\end{aligned}
$$
Let $\exp\left(-\frac{C^2}{8}\right)\leq \epsilon$, we obtain
$C \geq \sqrt{8\log\frac{1}{\epsilon}}$. Therefore, we can choose $C = \mathcal{O}\left(\sqrt{\log\frac{1}{\epsilon}}\right)$ such that \eqref{eq: tail_prob} holds.

Now, we prove our result. 
We decompose $p_0(\by|\bx) - p_t(\by|\bx)$ as follows:
$$
\begin{aligned}
p_0(\by|\bx) - p_t(\by|\bx)
&= \int_{\Rbb^{d_\mcY}}(p_0(\by|\bx) - p_0(\by_0|\bx))\frac{\mu_t^{d_\mcY}}{\sigma_t^{d_\mcY}(2\pi)^{d_\mcY/2}}\exp\left(-\frac{\Vert\by-\mu_t\by_0\Vert}{2\sigma_t^2}\right)\mrd\by_0\\
&~~~~+ (\mu_t^{d_\mcY} - 1)p_t(\bx|\by).
\end{aligned}
$$
Then, by the Lipschitz continuity of $p_0(\by|\bx)$, we have
\begin{equation*}
\begin{aligned}
&\left|\int_{\Rbb^{d_\mcY}}(p_0(\by|\bx) - p_0(\by_0|\bx))\frac{\mu_t^{d_\mcY}}{\sigma_t^{d_\mcY}(2\pi)^{d_\mcY/2}}\exp\left(-\frac{\Vert \by - \mu_t\by_0\Vert^2}{2\sigma_t^2}\right)\mrd\by_0\right|\\
\leq & ~ L\int_{\Rbb^{d_\mcY}}\Vert\by-\by_0\Vert\frac{\mu_t^{d_\mcY}}{\sigma_t^{d_\mcY}(2\pi)^{d_\mcY/2}}\exp\left(-\frac{\Vert \by - \mu_t\by_0\Vert^2}{2\sigma_t^2}\right)\mrd\by_0\\
\leq & ~ L\int_{\Rbb^{d_\mcY}}\left(\frac{1-\mu_t}{\mu_t}\Vert\by\Vert + \frac{\sigma_t}{\mu_t}\Vert\bu\Vert\right)\exp\left(-\frac{\Vert\bu\Vert^2}{2}\right)\mrd\bu\\
\lesssim & ~ \frac{1-\mu_t}{\mu_t}\Vert\by\Vert + \frac{\sigma_t}{\mu_t}.
\end{aligned}
\end{equation*}
Therefore, for any $t > 0$, we have
$$
\begin{aligned}
\mathrm{TV}(p_t, p_0) &\lesssim 
\int_{A_{\by}}|p_0(\by|\bx) - p_t(\by|\bx)|\mrd\by + \int_{\Rbb^{d_\mcY}\backslash A_{\by}}|p_0(\by|\bx) - p_t(\by|\bx)|\mrd\by\\
&\lesssim\int_{A_{\by}}\left|\int_{\Rbb^{d_\mcY}}(p_0(\by|\bx) - p_0(\by_0|\bx))\frac{\mu_t^{d_\mcY}}{\sigma_t^{d_\mcY}(2\pi)^{d_\mcY/2}}\exp\left(-\frac{\Vert \by - \mu_t\by_0\Vert^2}{2\sigma_t^2}\right)\mrd\by_0\right|\mrd\by \\
& ~~~~~ + |\mu_t^{d_\mcY} - 1|\int_{A_{\by}}p_t(\by|\bx)\mrd\by + \epsilon\\
\lesssim &~ \int_{A_{\by}}\left(\frac{1-\mu_t}{\mu_t}\Vert\by\Vert + \frac{\sigma_t}{\mu_t}\right)\mrd\by + (1-\mu_t^{d_\mcY}) + \epsilon\\
\lesssim &(2C)^{d_\mcY}\left(\frac{1-\mu_t}{\mu_t}\sqrt{d_\mcY}C + \frac{\sigma_t}{\mu_t}\right) + (1-\mu_t^{d_\mcY}) + \epsilon.
\end{aligned}
$$
In our framework, $T_0$ will be chosen very small such that $e^{T_0} - 1 = \mathcal{O}(T_0)$. Let $t = T_0$, 
we choose $\epsilon = T_0$, then we have
$$
\mathrm{TV}(p_{T_0}, p_0)=\mathcal{O}\left(\sqrt{T_0}\log^{(d_\mcY + 1)/2}\frac{1}{T_0}\right).
$$
The proof is complete.
\end{proof}

\section{Bound $\Ebb_{\mcD,\mcT,\mcZ}\left[\mathrm{TV}(\wt{p}_{T_0}, p_{T_0})\right]$}\label{sec:appende}
In this section, we  prove Theorem \ref{th: TV_sampling_error}. To this end, we first introduce Girsanov theorem.
\begin{theorem}\label{th: Girsanov}
    Suppose that $(\Omega, \mathcal{F},(\mathcal{F}_t)_{t\geq 0}, Q)$ is a filtered probability space and $(\ba_t)_{t\in[0,T]}$ is an adapted process on this space. For $t\in[0,T]$, let $\mcL_t:=\int_0^t \ba_s\mrd\bB_s$ where $(\bB_t)_{t\in[0,T]}$ is a $Q$-Brownian motion. Assume that $\Ebb_{Q}\left[\int_0^T\Vert\ba_t\Vert^2\mrd t\right] < \infty$. Then, $(\mcL_t)_{t\in[0,T]}$ is a square-integrable $Q$-martingal. Moreover, if 
    \begin{equation*}
        \Ebb_Q\left[\mathcal{E}(\mcL)_T\right] = 1,
    \end{equation*}
    where 
    $$
    \mathcal{E}(\mcL)_t:=\exp\left(\int_0^t\ba_s\mrd\bB_s - \frac{1}{2}\int_0^t\Vert\ba_s\Vert^2\mrd s\right),
    $$
    then $(\mathcal{E}(\mcL)_t)_{t\in[0,T]}$ is also a $Q$-martingale and the process
    $$
    \boldsymbol{\beta}_t:=\bB_t - \int_0^t\ba_s\mrd s
    $$
    is a Brownian motion under the new measure $P:=\mathcal{E}(\mcL)_T Q$.
\end{theorem}
Denote the path measures corresponding to \eqref{sde: forward_cond_OU} and \eqref{sde: EM_scheme} are $Q$ and $\wt{Q}$. Let $\ba_t=\sqrt{2}\left[\wh\bs(T-t_k,\bY_{t_k},\bx) - \bb(T-t,\bY_t,x)\right]$, $t\in[t_k, t_{k+1}]$. We apply Theorem \ref{th: Girsanov} to $Q$ and $\ba_t$, and assume the assumptions of Theorem \ref{th: Girsanov} are satisfied, then by data processing inequality, we have
    \begin{equation}\label{eq: KL}
    \mathrm{KL}(p_{T_0}|\wt{p}_{T_0})\leq\mathrm{KL}(Q|\wt{Q})\leq\sum_{k=0}^{N-1}\Ebb_{Q}\int_{t_k}^{t_{k+1}}\Vert\bb(T-t,\bY_t,\bx)-\wh\bs(T-t_k,\bY_{t_k},\bx)\Vert^2\mrd t.
    \end{equation}
Taking expectations at the both sides of \eqref{eq: KL} with respect to $\mcD,\mcT,\mcZ$, we obtain
    \begin{equation}\label{eq: KL_Expectation}
        \Ebb_{\mcD,\mcT,\mcZ}\left[\mathrm{KL}(p_{T_0}|\wt{p}_{T_0})\right]\leq\Ebb_{\mcD,\mcT,\mcZ}\left(\sum_{k=0}^{N-1}\Ebb_{Q}\int_{t_k}^{t_{k+1}}\Vert\bb(T-t,\bY_t,\bx)-\wh\bs(T-t_k,\bY_{t_k},\bx)\Vert^2\mrd t\right).
    \end{equation}

\begin{proof}[Proof of Theorem \ref{th: TV_sampling_error}.]
    We first show that 
    $$
    \sum_{k=0}^{N-1}\Ebb_{Q}\int_{t_k}^{t_{k+1}}\Vert\bb(T-t,\bY_t,\bx)-\wh\bs(T-t_k,\bY_{t_k},\bx)\Vert^2\mrd t < \infty.
    $$
    For any $t\in[t_k, t_{k+1}]$, by \cite[Lemma 10]{Chen2022SamplingIA}, we have 
    $$
    \begin{aligned}
    &~~~~\Ebb_{Q}\Vert\bb(T-t,\bY_t,\bx)-\wh\bs(T-t_k,\bY_{t_k},\bx)\Vert^2\\
    &\lesssim\Ebb_{Q}\Vert\bY_t\Vert^2 + K^2 + \Ebb_{Q}\Vert\nabla_{\by}\log p_{T-t}(\bY_t|\bx)\Vert^2\\
    &\lesssim d_\mcY + K^2 + d_\mcY\xi
    \end{aligned}
    $$
    Therefore, we have
    $$
    \begin{aligned}
    &~~~~\sum_{k=0}^{N-1}\Ebb_{Q}\int_{t_k}^{t_{k+1}}\Vert\bb(T-t,\bY_t,\bx)-\wh\bs(T-t_k,\bY_{t_k},\bx)\Vert^2\mrd t\\ &\lesssim (d_\mcY + K^2 + d_\mcY\xi)(T-T_0)<\infty.
    \end{aligned}
    $$
    
    Secondly, we give an upper bound for the right side of \eqref{eq: KL_Expectation}. For any $t\in[t_k, t_{k+1}]$, by the Lipschitz continuity of $\wh{\bs}$, we have
    $$
    \begin{aligned}
    &~~~~\Vert \bb(T-t,\bY_t,\bx)-\wh\bs(T-t_k,\bY_{t_k},\bx)\Vert^2\\
    &\lesssim \Vert \bb(T-t,\bY_t,\bx)-\wh\bs(T-t,\bY_{t},\bx)\Vert^2
    + \Vert\wh\bs(T-t, \bY_t,\bx) - \wh\bs(T-t,\bY_{t_k},\bx)\Vert^2 \\
    &~~~~+ \Vert\wh\bs(T-t, \bY_{t_k},\bx) - \wh\bs(T-t_k,\bY_{t_k},\bx)\Vert^2\\
    &\lesssim \Vert \bb(T-t,\bY_t,\bx)-\wh\bs(T-t,\bY_{t},\bx)\Vert^2 + d_\mcY\gamma_1^2\Vert \bY_t - \bY_{t_k}\Vert^2 + \gamma_2^2|t-t_k|^2.
    \end{aligned}
    $$
    Then, we have
    $$
    \begin{aligned}
    &~~~~\sum_{k=0}^{N-1}\Ebb_{Q}\int_{t_k}^{t_{k+1}}\Vert\bb(T-t,\bY_t,\bx)-\wh\bs(T-t_k,\bY_{t_k},\bx)\Vert^2\mrd t\\
    &\leq \sum_{k=0}^{N-1}\Ebb_{Q}\int_{t_k}^{t_{k+1}}\Vert \bb(T-t,\bY_t,\bx) - \wh\bs(T-t, \bY_t,\bx)\Vert^2\mrd t \\
    &~~~~~~~~~~+ d_\mcY\gamma_1^2\sum_{k=0}^{N-1}\int_{t_k}^{t_{k+1}}\Ebb_{Q}\Vert\bY_t - \bY_{t_k}\Vert^2\mrd t + \frac{1}{3}\gamma_2^2\sum_{k=0}^{N-1}(t_{k+1}-t_k)^3\\
    &=\int_{0}^{T-T_0}\Ebb_{Q}\Vert\bb(T-t,\bY_t,\bx)-\wh\bs(T-t,\bY_{t},\bx)\Vert^2\mrd t\\
    &~~~~~~~~~~+ d_\mcY\gamma_1^2\sum_{k=0}^{N-1}\int_{t_k}^{t_{k+1}}\Ebb_{Q}\Vert\bY_t - \bY_{t_k}\Vert^2\mrd t + \frac{1}{3}\gamma_2^2\sum_{k=0}^{N-1}(t_{k+1}-t_k)^3.
    \end{aligned}
    $$
    The term $\Ebb_{Q}\Vert\bY_t - \bY_{t_k}\Vert^2$ can be bounded as follows:
    $$
    \begin{aligned}
    \Ebb_{Q}\Vert\bY_t - \bY_{t_k}\Vert^2 &= \Ebb_{Q}\Vert\ov\bY_{T-t} - \ov\bY_{T-t_k}\Vert^2\\
    &=\Ebb\left[\left\Vert-\int_{T-t}^{T-t_k}\ov\bY_{u}\mrd u + \sqrt{2}(\bB_{T-t_k} - \bB_{T-t})\right\Vert^2\right]\\
    &\lesssim (t-t_k)\int_{T-t}^{T-t_k}\Ebb_Q\Vert\ov\bY_u\Vert^2\mrd u + d_\mcY(t-t_k)\\
    &\lesssim d_\mcY(t-t_k)^2 + d_\mcY(t-t_k).
    \end{aligned}
    $$
    Therefore, with Theorem \ref{th: drift_estimation}, we have
    $$
\begin{aligned}
&~~~~\Ebb_{\mcD,\mcT,\mcZ}\left(\sum_{k=0}^{N-1}\Ebb_{Q}\int_{t_k}^{t_{k+1}}\Vert\bb(T-t,\bY_t,\bx)-\wh\bs(T-t_k,\bY_{t_k},\bx)\Vert^2\mrd t\right)\\
&\lesssim 
\widetilde{\mathcal{O}}\left(
\frac{(T-T_0)^2\xi^{d_\mcY}(\beta B_{\mathcal{X}})^{d_{\mathcal{X}}}}{T_0^5}  
\left(n^{-\frac{2}{d_\mcX + d_{\mathcal{Y}} + 3}} + n^{\frac{d_\mcX + d_{\mathcal{Y}} + 1}{2(d_\mcX + d_{\mathcal{Y}} + 3)}}m^{-\frac{1}{2}}\right)
    \right)\\
    &~~~~+ \frac{1}{3}\left(d_{\mcY}^2\gamma_1^2 + \gamma_2^2\right)\sum_{k=0}^{N-1}(t_{k+1}-t_k)^3 + \frac{1}{2}d_{\mcY}\sum_{k=0}^{N-1}(t_{k+1}-t_k)^2.
    \end{aligned}
    $$
By the definition of $\mcL_t$, we know $(\mathcal{E}(\mcL)_t)_{t\in[0,T-T_0]}$ is a continuous local martingale. So we can find an increasing sequence of stopping times $(T_r)_{r\geq 1}$ such that $T_r\uparrow T-T_0$ such that $(\mathcal{E}(\mcL)_{t\wedge T_r})_{t\in[0,T-T_0]}$ is a continuous martingale. We define $\mcL_t^r:=\int_0^t\ba_s\mI_{[0,T_r]}\mrd\bB_s$ for all $t\in[0,T-T_0]$ and $r\geq 1$. Then we have $\mathcal{E}(\mcL)_{t\wedge T_r} = \mathcal{E}(\mcL^r)_t$, so $\mathcal{E}(\mcL^r)$ is a continuous martingale, and it follows that $\Ebb_{Q}\left[\mathcal{E}(\mcL^n)_{T-T_0}\right]=1$. Therefore, we can apply Theorem \ref{th: Girsanov} to $\mcL^r$ on $[0,T-T_0]$. Since $\Ebb_{Q}\left[\int_0^{T-T_0}\Vert\ba_t\mI_{[0,T_r]}\Vert^2\mrd t\right]\leq\Ebb_{Q}\left[\int_0^{T-T_0}\Vert\ba_t\Vert^2\mrd t\right] < \infty$, we obtain that under the new measure $Q^r:=\mathcal{E}(\mcL^r)_{T-T_0}Q$, there exists a $Q^r$-Brownian motion $(\boldsymbol{\beta}_t^r)_{t\in[0,T-T_0]}$ such that
    $$
    \mrd\boldsymbol{\beta}_t^r = \mrd\bB_t - \ba_t\mI_{[0,T_r]}(t)\mrd t.
    $$
    Recall that under $Q$ we have
    $$
    \mrd\bY_t = \left[\bY_t + 2\nabla_{\by}\log p_{T-t}(\bY_t|\bx)\right]\mrd t + \sqrt{2}\mrd \bB_t, ~\bY_0\sim p_{T}(\by|\bx).
    $$
    The equation \eqref{sde: forward_cond_OU} still holds almost surely under $Q^r$, since $Q^r\ll Q$. Therefore, we obtain that under $Q^r$, it almost surely holds
    $$
    \mrd\bY_t = \wh\bs(T-t_k,\bY_{t_k}, 
    \bx)\mI_{[0,T_r]}(t)\mrd t + \bb(T-t,\bY_t,\bx)\mI_{[T_r,T-T_0]}(t)\mrd t + \sqrt{2}\mrd\boldsymbol{\beta}_t^r.
    $$
    In addition, we have the following bound:  
    $$
    \begin{aligned}
    \mathrm{KL}(Q|Q^r) &= \Ebb_{Q}\left[\log\frac{\mrd Q}{\mrd Q^r}\right] = -\Ebb_{Q}\left[\mathcal{E}(\mcL^r)_{T-T_0}\right]\\ 
    &=\Ebb_{Q}\left[-\mcL_{T_r} + \frac{1}{2}\int_0^{T_r}\Vert\ba_t\Vert^2\mrd t\right]\\
    &\leq\frac{1}{2}\Ebb_Q\left[\int_0^{T-T_0}\Vert\ba_t\Vert^2\mrd t\right]\\
    &=\sum_{k=0}^{N-1}\Ebb_{Q}\int_{t_k}^{t_{k+1}}\Vert\bb(T-t,\bY_t,\bx)-\wh\bs(T-t_k,\bY_{t_k},\bx)\Vert^2\mrd t,
    \end{aligned}
    $$
    where we used that $\Ebb_{Q}[\mcL_{T_r}] = 0$ since $(\mcL_t)_{t\in[0,T-T_0]}$ is a $Q$-martingale and $T_r$ is a bounded stopping time.

    Now, we consider coupling $(Q^r)_{r\in\mathbb{N}^{+}}$ and $\wt{Q}$ by defining a sequence of stochastic processes $(\bY_t^r)_{t\in[0,T-T_0]}$, a stochastic process $(\bY_t)_{t\in[0,T-T_0]}$ and a single Brownian motion $(\boldsymbol{W}_t)_{t\in[0,T-T_0]}$ on the same probability space such that
    $$
    \mrd\bY_t^r = \wh\bs(T-t_k, \bY_{t_k}^r,\bx)\mI_{[0,T_r]}(t)\mrd t + \bb(T-t,\bY_t,\bx)\mI_{[T_r,T-T_0]}(t)\mrd t + \sqrt{2}\mrd\boldsymbol{W}_t,
    $$
    and
    $$
    \mrd \bY_t = \wh\bs(T-t_k,\bY_{t_k},\bx)\mrd t + \sqrt{2}\mrd\boldsymbol{W}_t
    $$
    with ${\bY}^r_0 = \bY_0$ and $\bY_0\sim p_{T}(\by|\bx)$. The path measures corresponding to $\bY_t^r$, $\bY_t$ are $Q^r$ and $\wt{Q}$.

    For fixed $\epsilon > 0$, we define $\pi_{\epsilon}:C([0,T-T_0];\Rbb^{d_\mcY})\rightarrow C([0,T-T_0];\Rbb^{d_\mcY})$ by 
    $$
    \pi_{\epsilon}(\omega)(t):=\omega(t\wedge (T-T_0-\epsilon)).
    $$
    Then, $\pi_{\epsilon}(\bY_t^r)\rightarrow\pi_{\epsilon}(\bY_t)$ uniformly over $[0,T-T_0]$ almost surely. Therefore, $(\pi_{\epsilon})_{\#}Q^r\rightarrow (\pi_{\epsilon})_{\#}\wt{Q}$ weakly. Using the lower semicontinuity of the $\mathrm{KL}$ divergence and the data processing inequality, we have
    $$
    \begin{aligned}\mathrm{KL}\left((\pi_{\epsilon})_{\#}Q\big|(\pi_{\epsilon})_{\#}\wt{Q}\right)&\leq \mathop{\lim\inf}_{r\rightarrow\infty}\mathrm{KL}\left((\pi_{\epsilon})_{\#}Q\big|(\pi_{\epsilon})_{\#}Q^r\right)\\
    &\leq\mathop{\lim\inf}_{r\rightarrow\infty}\mathrm{KL}\left(Q|Q^r\right)\\
    &\leq \sum_{k=0}^{N-1}\Ebb_{Q}\int_{t_k}^{t_{k+1}}\Vert\bb(T-t,\bY_t,\bx)-\wh\bs(T-t_k,\bY_{t_k},\bx)\Vert^2\mrd t.
    \end{aligned}
    $$
    Finally, letting $\epsilon\rightarrow 0$, we obtain that $\pi_{\epsilon}(\omega)\rightarrow\omega$ uniformly on $[0,T-T_0]$ and hence 
    $$
    \mathrm{KL}\left((\pi_{\epsilon})_{\#}Q\big|(\pi_{\epsilon})_{\#}\wt{Q}\right)\rightarrow\mathrm{KL}(Q|\wt{Q}),
    $$
    which implies that
    $$
    \mathrm{KL}(p_{T_0}|\wt{p}_{T_0})\leq\mathrm{KL}(Q|\wt{Q})\leq\sum_{k=0}^{N-1}\Ebb_{Q}\int_{t_k}^{t_{k+1}}\Vert\bb(T-t,\bY_t,\bx)-\wh\bs(T-t_k,\bY_{t_k},\bx)\Vert^2\mrd t,
    $$
    and 
    $$
    \begin{aligned}
&~~~~\Ebb_{\mcD,\mcT,\mcZ}\left[\mathrm{KL}(p_{T_0}|\wt{p}_{T_0})\right]\\
&\leq\Ebb_{\mcD,\mcT,\mcZ}\left(\sum_{k=0}^{N-1}\Ebb_{Q}\int_{t_k}^{t_{k+1}}\Vert\bb(T-t,\bY_t,\bx)-\wh\bs(T-t_k,\bY_{t_k},\bx)\Vert^2\mrd t\right)\\
    &\lesssim 
    \widetilde{\mathcal{O}}\left(
    \frac{(T-T_0)^2\xi^{d_\mcY}(\beta B_{\mathcal{X}})^{d_{\mathcal{X}}}}{T_0^5}  
    \left(n^{-\frac{2}{d_\mcX + d_{\mathcal{Y}} + 3}} + n^{\frac{d_\mcX + d_{\mathcal{Y}} + 1}{2(d_\mcX + d_{\mathcal{Y}} + 3)}}m^{-\frac{1}{2}}\right)
    \right)\\
    &~~~~+ \frac{1}{3}\left(d_\mcY^2\gamma_1^2 + \gamma_2^2\right)\sum_{k=0}^{N-1}(t_{k+1}-t_k)^3 + \frac{1}{2}d_\mcY \sum_{k=0}^{N-1}(t_{k+1}-t_k)^2.
    \end{aligned}
    $$
    Therefore, with inequality 
    $\Ebb_{\mcD,\mcT,\mcZ}\left[\mathrm{TV}(p_{T_0},\wt{p}_{T_0})\right]\lesssim\left(\Ebb_{\mcD,\mcT,\mcZ}\left[\mathrm{KL}(p_{T_0}|\wt{p}_{T_0})\right]\right)^{1/2}$, 
    we have
    $$
    \begin{aligned}
    \Ebb_{\mcD,\mcT,\mcZ}\left[\mathrm{TV}(p_{T_0},\wt{p}_{T_0})\right] 
    &= 
\widetilde{\mathcal{O}}\Bigg(
    \frac{(T-T_0)\xi^{\frac{d_\mcY}{2}}(\beta B_{\mathcal{X}})^{\frac{d_\mcX}{2}}}{T_0^{\frac{5}{2}}}  
    \left(n^{-\frac{1}{d_\mcX+ d_{\mathcal{Y}} + 3}} + n^{\frac{d_\mcX + d_{\mathcal{Y}} + 1}{4(d_\mcX+ d_{\mathcal{Y}} + 3)}}m^{-\frac{1}{4}} \right) \\
    &~~~~~~~ + \left(d_\mcY \gamma_1 + \gamma_2\right)\sqrt{\sum_{k=0}^{N-1}(t_{k+1}-t_k)^3} + d_\mcY^{\frac{1}{2}
    }\sqrt{\sum_{k=0}^{N-1}(t_{k+1}-t_k)^2}\Bigg).
    \end{aligned}
    $$
The proof is complete.
\end{proof}

\section{Proof of Main Results.}\label{sec:appendf}
In this section, we give the proofs of Theorem \ref{th: end_to_end_convergence} and Theorem \ref{th: weakcon}. 

\begin{proof}[Proof of Theorem \ref{th: end_to_end_convergence}]
    Combining Theorem \ref{th:TV_initial_distribution}, Theorem \ref{th: TV_early_stopping}, and Theorem \ref{th: TV_sampling_error}, we have
$$
\begin{aligned}
&~~~~\Ebb_{\mcD,\mcT,\mcZ}\left[\mathrm{TV}(\ck{p}_{T_0},p_0)\right]\\ 
&= 
\widetilde{\mathcal{O}}\Bigg(e^{-T} + \sqrt{T_0}\log^{(d_\mcY+ 1)/2}\frac{1}{T_0} + 
    \frac{(T-T_0)\xi^{\frac{d_\mcY}{2}}(\beta B_{\mathcal{X}})^{\frac{d_\mcX}{2}}}{T_0^{\frac{5}{2}}}  
    \left(n^{-\frac{1}{d_\mcX+ d_{\mathcal{Y}} + 3}} + n^{\frac{d_\mcX + d_{\mathcal{Y}} + 1}{4(d_\mcX+ d_{\mathcal{Y}} + 3)}}m^{-\frac{1}{4}} \right) \\
    &~~~~~~~~~+ \left(d_\mcY \gamma_1 + \gamma_2\right)\sqrt{\sum_{k=0}^{N-1}(t_{k+1}-t_k)^3} + d_\mcY^{\frac{1}{2}
    }\sqrt{\sum_{k=0}^{N-1}(t_{k+1}-t_k)^2}\Bigg).
    \end{aligned}
    $$
By the chosen neural networks, we have $\gamma_1 = \mathcal{O}\left(\frac{20 d_\mcY^2}{T_0^2}\right)$ and $\gamma_2 = \mathcal{O}\Big(\frac{10 d_\mcY^{3/2}\sqrt{\log\frac{d_\mcY}{\epsilon T_0}}}{T_0^3}\Big)$. 
Let $m > n^{\frac{d_\mcX+ d_{\mcY} + 5}{d_\mcX+ d_{\mcY} + 3}}$, by choosing
    $T = \mathcal{O}(\log n)$, $T_0 = \mathcal{O}\left(n^{-\frac{1}{(d_\mcY+ 3)(d_\mcX+ d_{\mcY} + 3)}}\right)$ and $\max_{k=0,\cdots,N-1}|t_{k+1}-t_k| = \mathcal{O}\left(n^{-\frac{6}{(d_\mcY+ 3)(d_\mcX+ d_{\mcY} + 3)}}\right)$, we obtain
    $$
    \Ebb_{\mcD,\mcT,\mcZ}[\mathrm{TV}(\ck{p}_{T_0},p_0)] = \wt{O}\left(n^{-\frac{1}{2(d_\mcY+ 3)(d_\mcX+ d_{\mcY} + 3)}}\right).
    $$
    This proof is complete.
\end{proof}

\begin{proof}[Proof of Theorem \ref{th: weakcon}]
Denote by $\wh{f}(\bx):= \mathbb{E}_{\wh{Y}_{\bx} \sim \wh{P}_{Y|X=\bx}} \wh{Y}_{\bx}$.
Then, we have
\begin{align*}
\sqrt{M} \wh{S}_{\bx}^{-1} \left(\overline{\wh{Y}}_{\bx}-f_0(\bx)\right)
=\sqrt{M}\wh{S}_{\bx}^{-1}\left(\overline{\wh{Y}}_{\bx}-\wh{f}(\bx)\right)
+\sqrt{M}\wh{S}_{\bx}^{-1}\left(\wh{f}(\bx)-f_0(\bx)\right).
\end{align*}
From the definition of the TV distance and using Theorem \ref{th: end_to_end_convergence}, we can conclude that  
$\wh{S}^2_{\bx}$ converges to the variance of $Y_{\bx}$ as $M$ and $n$ tends to infinity.
Then, by invoking the central limit theorem,  
we can infer that 
$\sqrt{M} \wh{S}_{\bx}^{-1} \left(\overline{\wh{Y}}_{\bx}-\wh{f}(\bx)\right)
$ weakly converges to the standard Gaussian distribution.
Therefore, it remains to prove that 
$\sqrt{M}\wh{S}_{\bx}^{-1}\left(\wh{f}(\bx)-f_0(\bx)\right)
=o_p(1)$ as $M$ tends to infinity.
It follows that
\begin{align*}
\Ebb_{\mcD,\mcT,\mcZ}
\left[|\wh{f}(\bx)-f_0(\bx)|\right]
&=\Ebb_{\mcD,\mcT,\mcZ} \left[\left|\mathbb{E}_{\wh{Y}_{\bx} \sim \wh{P}_{Y|X=\bx}} \wh{Y}_{\bx}
-\mathbb{E}_{Y_{\bx} \sim P_{Y|X=\bx}} Y_{\bx}
\right|\right]\\
&\lesssim 
\Ebb_{\mcD,\mcT,\mcZ}\left[
\mathrm{TV}(\wh{P}_{Y|X=\bx},P_{Y|X=\bx})
\right]\\
&\leq 
\widetilde{\mathcal{O}}\Big(n^{-\frac{1}{2(d_\mcY+ 3)(d_\mcX+ d_{\mcY} + 3)}}\Big),
\end{align*}
where the first inequality follows from the definition of  TV distance,
the second inequality holds by Theorem \ref{th: end_to_end_convergence}.
When $n\geq M^{(d_\mcY+ 3)(d_\mcX+ d_{\mcY} + 4)}$, we have $\sqrt{M}\wh{S}_{\bx}^{-1}\left(\wh{f}(\bx)-f_0(\bx)\right)=o_p(1)$.
Therefore, we obtain the desired result.
\end{proof}

\bibliographystyle{alpha}
\bibliography{biblio}
\end{document}